\documentclass[preprint, 12pt]{elsarticle}

\usepackage{amssymb}
\usepackage{amsmath}
\usepackage{amsthm}
\usepackage[utf8]{inputenc}
\usepackage{amscd, mathtools, bm}
\usepackage{faktor}
\usepackage{subcaption}
\usepackage{graphicx}
\usepackage{epstopdf}
\usepackage{makecell}
\usepackage[labelfont=bf]{caption}
\usepackage{multicol}
\usepackage{cancel}
\usepackage{booktabs}
\usepackage{color}
\usepackage{appendix}
\usepackage{tikz}
\usepackage[normalem]{ulem}
\usepackage{multirow}
\usepackage{amsfonts}
\usepackage{algorithm}
\usepackage{algpseudocode}
\usepackage{float}
\usepackage{times}
\usepackage{xcolor}
\usepackage{sidecap}
\usepackage[hmargin=1in,vmargin=1in]{geometry}
\newtheorem{theorem}{\textsc{Theorem}}

\newtheorem{remark}{\textsc{Remark}}
\newtheorem{proposition}{\textsc{Proposition}}
\newcommand{\R}{\mathbb{R}}

\newcommand{\f}{\mathbf{f}}
\newcommand{\F}{\mathbf{F}}

\newcommand{\br}{\mathbf{r}}
\newcommand{\bs}{\mathbf{s}}
\newcommand{\bu}{\mathbf{u}}

\newcommand{\bv}{\mathbf{v}}
\newcommand{\bV}{\mathbf{V}}

\newcommand{\bx}{\mathbf{x}}
\newcommand{\bX}{\mathbf{X}}

\newcommand{\bY}{\mathbf{Y}}
\newcommand{\balpha}{\bm{\alpha}}

\newcommand{\zero}{\mathbf{0}}
\newcommand{\Hspace}{\mathcal{H}}
\newcommand{\phivec}{\bm\phi}
\newcommand{\varphivec}{\bm\varphi}
\newcommand{\estphi}{\hat\phi}
\newcommand{\estphivec}{\bm{\hat\phi}}
\newcommand{\Cmat}{\bm{C}}
\newcommand{\mE}{\mathcal{E}}
\newcommand{\numcl}{K}
\newcommand{\idxcl}{k}
\newcommand{\cl}{C}
\newcommand{\clof}{\mathrm{k}}
\newcommand{\norm}[1]{\left\| #1 \right\|}
\newcommand{\inpro}[1]{\langle #1 \rangle}
\newcommand{\argmin}{\mathop{\mathrm{argmin}}}

\newcommand{\E}{\mathop{\mathbb{E}}}
\newcommand{\supp}[1]{\text{supp}(#1)}

\DeclareMathOperator{\dif}{d \!}
\newcommand{\dt}{\dif t}

\journal{Physica D: Nonlinear Phenomena}

\begin{document}

\begin{frontmatter}



\title{Learning Interaction Kernels from Collective Steady States} 

\author{Baoli Hao} 

\affiliation{organization={Department of Applied Mathematics, Illinois Institute of Technology},
            city={Chicago},
            postcode={60616}, 
            state={IL},
            country={USA}}
\author{Mauro Maggioni} 

\affiliation{organization={Department of Mathematics, Department of Applied Mathematics and Statistics, Johns Hopkins University},
            city={Baltimore},
            postcode={21218}, 
            state={MD},
            country={USA}}
\author{Ming Zhong} 

\affiliation{organization={Department of Mathematics, University of Houston},
            city={Houston},
            postcode={77204}, 
            state={TX},
            country={USA}}            
\begin{abstract}
We propose a learning procedure for system identification in interacting particle systems from single-snapshot observations of collective behaviors, unlike existing approaches that rely on observations of trajectories. This setting leads to a fundamentally ill-posed inverse problem, which we solve by using a regularization strategy based on the empirical distribution of observed configurations, drawn from different, unobserved initial conditions. We test our learning procedure on a variety of representative models with steady-state and quasi-stationary patterns, where collective behaviors encode implicit information about the interaction mechanisms, demonstrating that our approach enables stable and accurate recovery of the underlying interaction laws, leading to faithful reproduction of the collective behavior, and in many cases even of the dynamics leading up to it.
\end{abstract}




\end{frontmatter}



\section{Introduction}\label{sec:intro}
Interacting particle systems provide a fundamental framework for modeling emergent phenomena in nonlinear science and statistical physics. Complex systems of many interacting particles often exhibit coherent macroscopic patterns such as clustering~\cite{OD2002, MT2014}, flocking~\cite{PhysRevLett.75.1226, CS2007, AD2021}, synchronization~\cite{STROGATZ20001}, milling and swarming~\cite{PhysRevLett.96.104302, PhysRevE.93.043112, CHUANG200733}. Variants based on these models can be used to describe even more complicated collective behaviors, where  multiple emergent behaviors are mixed together, such as swarmalation~\cite{swarmalator2017, hao2023mixed}, line formation~\cite{GTZ2023}, hot-spot formation in crime~\cite{crime2008, hao2025crime, hao2026crime}, and multi-species interactions~\cite{10.1098/rsif.2013.1208}. These behaviors can arise from microscopic interaction laws that are relatively simple and yet lead to highly nontrivial collective dynamics. Mathematically, such systems are commonly described by nonlinear systems of ordinary or stochastic differential equations, for example in the form
\[
\dot\bx_i = \frac{1}{N}\sum_{j = 1, j \neq i}^N\Phi(\bx_i, \bx_j)(\bx_j - \bx_i), \quad i = 1, \cdots, N.
\]
Here $\bx_i \in \R^d$ describes the state variable of agent $i$, $\Phi: \R^d\times\R^d \rightarrow \R$, a pairwise interaction kernel, dictates how agent $j$ is affecting the change of state of agent $i$. 
In many cases of interest, such interaction structure can be further simplified, for instance through radial kernels depending on pairwise distances between the states.  The averaging by $1/N$ is mainly used for the mean-field limit (as $N \rightarrow \infty$), which consists of kinetic and continuum equations, such as Vlasov-Fokker–Planck type of equations~\cite{Jeans1915, VFP2025}, providing a bridge between microscopic dynamics and macroscopic observables.

While the forward analysis of such systems, such as deriving qualitative behavior and emergent structures from known interaction laws, has been extensively developed~\cite{Tadmor2015, Tadmor2021, Tadmor2025}, corresponding inverse problems remain less understood. For example, given observational data, can one recover the underlying interaction law $\Phi$? This question is central to data-driven discovery of governing equations and has attracted significant attention in recent years~\cite{mauro2017, LZTM2019, ZMM2020, MMQZ2021, FMMZ2022, MTZM2022, LANG2026101867, LANG2023}. Existing approaches typically rely on trajectory data (in continuous time or at a discrete set of time points), which provide direct/indirect information about time-derivatives and enable likelihood-based reconstruction of the dynamics. Other approaches, designed for homogeneous systems with a large number of agents, are based on the mean field limit approximation (a PDE describing the evolution of the probability distribution of the agents), see for example \cite{messenger2022learning,lang2023identifiabilityinteractionkernelsmeanfield}, and only very recent works that do without such an approximation have emerged \cite{wei2026learninginteractingparticlesystems}.

In contrast, many systems arising in physics and related fields are observed only at equilibrium or at a collective steady state. Examples include steady configurations in self-organization, stationary distributions in kinetic systems, behaviors such as flocking and synchronization, and metastable patterns in nonlinear media. In such settings, one observes samples drawn from an invariant measure $\mu_T$ associated with the dynamics, rather than trajectory data, $\{\bx_i(t)\}_{t \in [0, T]}$.  
We consider here the inverse problem of inferring $\Phi$ from such equilibrium and collective steady states: the absence of temporal information removes direct access to the vector field driving the dynamics, and the mapping from interaction laws to invariant measures typically cannot be expected to be one-to-one, and it often is not so in the mean field regime, where many different interaction kernels can lead to the same stationary measure. The problem is ill-posed, with distinct interaction mechanisms inducing (almost) indistinguishable collective steady state patterns.

This work shows that, in perhaps more cases than one may expect, the problem of learning interaction laws from only collective steady state observations can be solved; there is enough structure in the state of the $N$ particles, even when they reach a collective steady state, to estimate the interaction kernel; of course a simple yet consequential observation is that such a state of $N$ particles is very different from sampling $N$ particles independently from a steady state distribution/collective state. 
We provide extensive numerical experiments, ranging from first-order systems with steady states, to semi-steady states in multi-species systems, and quasi-steady states in fish-milling systems.  These numerical experiments demonstrate that the proposed learning procedure can recover the underlying interaction kernels, and it can do so in a robust way, as shown when we increase the level of observation noise.

The remaining sections of our paper are organized as follows: in Section~\ref{sec:model}, we introduce the models from which the observation data is obtained; we develop the learning framework and discuss possible ill-conditioning and identifiability in Section~\ref{sec:learn}; in Section~\ref{sec:example}, we demonstrate our learning tested on nine different types of dynamics, ranging from complete steady-state in first-order systems, semi-steady-state (flocking) in first-order multi-species systems and in second-order systems, and quasi-steady-state (milling); we conclude our paper and note a few possible future directions with this line of research in Section~\ref{sec:conclude}. 
\subsection{Comparison to Other Methods}
Differential-equation-based learning, i.e. learning from the data generated by $\dot\bX = \f(\bX)$, for high-dimensional states $\bX$, has gained increased interest in recent years given the advances in high-dimensional statistics and machine learning.  Representative approaches include SINDy~\cite{sindy2016} along with its weak-formulation variant Weak SINDy~\cite{weak_sindy}, Neural ODE~\cite{NODE2018}, PINN~\cite{pinn2019}, and recurrent neural-network formulations for learning PDE dynamics~\cite{RNN_PDE2018}.

Another line of work~\cite{mauro2017, LZTM2019, ZMM2020, MMQZ2021, FMMZ2022, MTZM2022} focuses on interacting particles, and uses a variational inverse problem approach to infer interaction kernel from time-series-like trajectory data, maximally using the symmetries in the equations. These methods use locally-constructed, partially data-driven bases to estimate the interaction kernels, and demonstrate that in many cases this can be achieved with no curse of dimensionality in the dimension of $\bX$. Similar learning methods using Gaussian Processes and Random Features are reviewed in \cite{feng2023learning}; concurrent learning of the collective forces and non-collective forces is developed in~\cite{James2026}; learning of interaction kernels in the mean-field limit of the particle models was studied in~\cite{LANG2023}. Other related works on learning from such limited amount of data can be found in~\cite{oneshotPINN, LfSS20}. 

Our method is a combination of the variational method and eigenvalue solvers.  We re-cast the learning into an eigenvalue problem, and we are able to recover the unknown interaction kernel with limited time information.
\section{Model Equations}\label{sec:model}
In this paper, we focus on two families of interacting particle systems. The first one is modeled by a first-order dynamical system in gradient flow form with an external force.  It is defined for a system of $N$ agents, where each agent is assigned a time-dependent state vector $\bx_i(t) \in \R^d$ (describing quantities of interest such as position, velocity, opinions, emotion, etc.). The rate of change of such state is governed by the system of ODEs
\begin{equation}\label{eq:first_order}
\dot\bx_i = \f^E(\bx_i) + \sum_{j = 1, j \neq i}^N\frac{1}{N_{\clof_j}}\phi^E_{\clof_i, \clof_j}(\norm{\bx_j - \bx_i})(\bx_j - \bx_i), \quad i = 1, \cdots, N.
\end{equation}
These systems contain heterogeneous agents: the agents are partitioned into $\numcl$ different types, with $\cl_\idxcl$ containing the indices of the agents of type $\idxcl$, for $\idxcl=1, \dots, \numcl$,  forming a partition of $[N]:=\{1,\dots,N\}$:
\[
[N] = \{1, \cdots, N\} = \cl_1\cup\cdots\cup\cl_\numcl \quad \text{and} \quad \cl_{\idxcl_1} \cap \cl_{\idxcl_2} = \emptyset\quad \text{if $\idxcl_1 \neq \idxcl_2$}\,;
\]
$\clof_i\in\{1,\dots,K\}$ returns the type of the $i$-th agent (i.e., $i\in C_{\clof_i}$), and $N_\idxcl$ is the number of agents in $\cl_\idxcl$.  We assume throughout that the type $\clof_i$ is known for all $i$'s (for recent work where the type is unknown, see \cite{lang2026learningmultitypeheterogeneousinteracting}).  Without loss of generality, and only to ease some of the notation later on, we assume that the agents are arranged in increasing type order, i.e.
\[
\{1, \cdots, N_{1}\} = \cl_1, \quad \{N_{1} + 1, \cdots, N_1 + N_2\} = \cl_2, \quad \ldots \quad, \{\sum_{\idxcl = 1}^{\numcl - 1}N_{\idxcl}+1, \cdots, N\} = \cl_\numcl.
\]
The vector field $\f^E:\R^d \rightarrow \R^d$ provides a non-collective force allowing the agent to interact with its surrounding environment, and $\phi_{\idxcl_1, \idxcl_2}^E: \R^+ \rightarrow \R$ is an interaction kernel which defines how an agent of type $\idxcl_2$ influences an agent of type $\idxcl_1$. There are therefore a total of $K^2$ interaction kernels, one for each choice of a pair of types.

The second type of systems we consider are modeled by second-order dynamics, usually derived from  Newton's second law, in the form
\begin{equation}\label{eq:second_order}
\begin{cases}
    \dot\bx_i &= \bv_i, \\
    \dot\bv_i &= \f^{EA}(\bx_i, \bv_i) + \sum_{j = 1, j \neq i}^N\frac{1}{N_{\clof_j}}\big(\phi^E_{\clof_i, \clof_j}(\norm{\bx_j - \bx_i})(\bx_j - \bx_i) \\
    &\qquad\qquad\qquad\qquad\qquad\qquad + \phi^A_{\clof_i, \clof_j}(\norm{\bx_j - \bx_i})(\bv_j - \bv_i)\big), \\
\end{cases}, \quad i = 1, \cdots, N.
\end{equation}
Here, $\f^{EA}: \R^{2d} \rightarrow \R^d$ defines a non-collective and environment-dependent force, $\phi^E_{\idxcl_1, \idxcl_2}: \R^+ \rightarrow \R$ is a potential-energy-based interaction kernel that typically models long-range attraction and short-range repulsion, and $\phi^A_{\idxcl_1, \idxcl_2}: \R^+ \rightarrow \R$ is an interaction kernel modeling alignment. In both cases we assume that $\f^E$, $\f^{EA}$, $\phi^E_{\idxcl_1, \idxcl_2}$, and $\phi^A_{\idxcl_1, \idxcl_2}$ satisfy standard regularity conditions which guarantee the existence and uniqueness of solutions for all time.
\begin{remark}
When $\f^E = \zero$ in \eqref{eq:first_order}, the system, in the case of a single species (we will take $\phi^E = \phi^E_{1, 1}$ to simplify the notation), is a gradient flow system, i.e.
\[
 \dot\bx_i = -\nabla_{\bx_i} E(\bx_1, \bx_2, \cdots, \bx_N), \quad \text{for some system energy $E$}.
\]
In the case of collective dynamical systems, $E$ often has translation, rotation and permutation invariance, leading to models with only radial dependence, i.e.
\[
 E(\bx_1, \bx_2, \cdots, \bx_N) = \frac{1}{2N}\sum_{i \neq j}U(\norm{\bx_j - \bx_i}), \quad U: \R^+ \rightarrow \R^+, \quad \phi^E(r) = \frac{U'(r)}{r}.
\]
Hence, the dynamical system evolves towards a minimum energy state.  Similar remarks apply to multi-species systems.
\end{remark}
\begin{remark}
Note that we have considered interacting agent systems in which the interaction kernels depend on low-dimensional features of the states:
\[
\Phi(\bx_i, \bx_j) = \phi(\xi(\bx_i, \bx_j)), \quad \text{where $\xi:\R^{2d} \rightarrow \R^{d'}$ and $\phi: \R^{d'} \rightarrow \R$},
\]
where $\phi$ is called the reduced interaction kernel and $\xi$ is a known feature map~\cite{FMMZ2022}. Here we assume that $d' \ll 2d$, in fact, the most common assumption for $\xi$ is $\xi(\bx_i, \bx_j) = \norm{\bx_j - \bx_i}$.
While not a strict requirement for the applicability of our ideas, this is motivated both by invariances (e.g., to translations and rotations), and it also facilitates the estimation problem, in the sense that the number of parameters to estimate, which in general scales exponentially in the dimension of the domain of the interaction kernel, is reduced as the dimension drops from $2d$ to $d'$.
\end{remark}
The initial condition of the dynamics above is drawn from some probability distribution $\mu^E_0$ (or $\mu_0^{EA}$) on the state (or, respectively, phase) space, i.e. $\bx_i(0) \sim \mu_0^E$ for first-order and $(\bx_i(0), \bv_i(0)) \sim \mu_0^{EA}$ for second-order systems.  We assume that the dynamics of the systems above converges to a stationary distribution $\mu_\infty^E$ (resp., $\mu^{EA}_\infty$) in state (resp., phase) space as $t \rightarrow \infty$, and that at a sufficiently large time $T$, irrespective of their initial condition, the system will have reached a distribution $\mu_T^E$ (resp., $\mu^{EA}_T$) sufficiently close to $\mu_\infty^E$ (resp., $\mu^{EA}_\infty$). We also consider more general cases where the dynamics converges to an attractive collective steady state, for example a flocking or milling dynamics.
\section{Learning Framework}\label{sec:learn}
The input data consist of a set of single-time snapshots of the dynamical system at some unknown large time $T$, i.e. from the stationary or collective steady state distribution $\mu_T$, in the form
\begin{equation*}
\begin{aligned}
&\{(\bx_{i, T}^m, \dot\bx_{i, T}^m)\}_{i, m = 1}^{N, M} & \text{ with } & \{\bx_{i, T}^m\}_{i = 1}^N \sim \mu_T^E &\text{for first-order }, \\
&\{(\bx_{i, T}^m, \bv_{i, T}^m, \dot\bv_{i, T}^m)\}_{i, m = 1}^{N, M} & \text{ with } &\{(\bx^m_{i, T}, \bv^m_{i, T})\}_{i = 1}^N \sim \mu_T^{EA} &\text{for second-order.}
\end{aligned}
\end{equation*}
These snapshots are assumed to be independent, identically distributed samples from the stationary distribution $\mu_T$; for example, they may have been obtained from trajectories started at independent, identically distributed initial conditions. Here $\bx_{i, T}^m = \bx_i(T)$ and $\bv_{i, T}^m = \bv_i(T)$.  In order to simplify the notation, from here on we suppress the dependence on $T$ from the notation  $\bx_{i, T}^m, \bv_{i, T}^m$, since $T$ is not needed in obtaining the estimators and it is assumed to be unknown unless otherwise specified.  Our goal is to estimate the interaction kernel(s) $\{\hat\phi^E_{\idxcl_1, \idxcl_2}\}_{\idxcl_1, \idxcl_2 = 1}^{\numcl}$ or $\{\hat\phi^E_{\idxcl_1, \idxcl_2}, \hat\phi^A_{\idxcl_1, \idxcl_2}\}_{\idxcl_1, \idxcl_2 = 1}^{\numcl}$ that determine the dynamics.

For first-order systems, we follow \cite{LZTM2019, ZMM2020, MMQZ2021, FMMZ2022, MTZM2022} and use the empirical loss functional
\begin{equation}\label{eq:loss_1st}
 \mE^E(\varphivec^E) = \frac1M\sum_{m, i = 1}^{M, N}\frac{1}{N_{\clof_i}}\norm{\dot\bx_{i}^m - \f_i^{E, m} - \sum_{j = 1, j \neq i}^N\frac{1}{N_{\clof_j}}\varphi^E_{\clof_i, \clof_j}(r_{i, j}^m)\br_{i, j}^m}^2,
\end{equation}
where $\f_{i}^{E, m} = \f^E(\bx_{i}^m)$, $\br_{i, j}^m = \bx_{j}^m - \bx_{i}^m$, $r_{i, j}^m = \norm{\br_{i, j}^m}$ ($\norm{\cdot}$ usual Euclidean norm), $\varphivec^E = (\varphi^E_{\idxcl_1, \idxcl_2})_{\idxcl_1, \idxcl_2 = 1}^{\numcl} \in \Hspace^E = \oplus_{\idxcl_1, \idxcl_2 = 1}^{\numcl}\Hspace^E_{\idxcl_1, \idxcl_2}$ where each $\Hspace^E_{\idxcl_1, \idxcl_2}$ is a finite-dimensional
linear hypothesis space of functions.  For second-order data, we use the following empirical loss functional
\begin{equation}\label{eq:loss_2nd}
\begin{aligned}
\mE^{EA}(\varphivec^{EA}) &= \sum_{m, i = 1}^{M, N}\frac{1}{MN_{\clof_i}}\norm{\dot\bv_i^m - \f_i^{EA, m} - \sum_{j = 1, j \neq i}^N\frac{1}{N_{\clof_j}}\big(\varphi^E_{\clof_i, \clof_j}(r_{i, j}^m)\br_{i, j}^m + \varphi^A_{\clof_i, \clof_j}(r_{i, j}^m)\bu_{i, j}^m\big)}^2\,,
 \end{aligned}
\end{equation}
where $\f_i^{EA, m} = \f^{EA}(\bx_i^m, \bv_i^m)$, $\bu_{i, j}^m = \bv_j^m - \bv_i^m$, $\varphivec^{EA} = (\varphivec^E, \varphivec^A)$ with $\varphivec^E$ defined as before and $\varphivec^A = (\varphi_{\idxcl_1, \idxcl_2}^A)_{\idxcl_1, \idxcl_2 = 1}^{\numcl} \in \Hspace^A = \oplus_{\idxcl_1, \idxcl_2 = 1}^{\numcl}\Hspace_{\idxcl_1, \idxcl_2}^A$ (each $\Hspace_{\idxcl_1, \idxcl_2}^A$ is still a finite-dimensional
linear hypothesis space) and $\Hspace^{EA} = \Hspace^E \oplus \Hspace^A$.  

The estimators for the interaction kernels 
\[
\phivec^E = \{\phi^E_{\idxcl_1, \idxcl_2}\}_{\idxcl_1, \idxcl_2 = 1}^{\numcl}\,, \quad \text{or} \quad \phivec^{EA} = (\phivec^E, \phivec^A) \quad \text{with $\phivec^A = \{\phi_{\idxcl_1, \idxcl_2}^A\}_{\idxcl_1, \idxcl_2 = 1}^{\numcl}$}\,,
\]
are constructed as minimizers of the empirical loss functionals above, over the hypothesis spaces $\Hspace^E$ and $\Hspace^A$:
\[
\estphivec^E = \argmin_{\varphivec^E \in \Hspace^E}\mE^E(\varphivec^E)  \quad \text{or} \quad \estphivec^{EA} = \argmin_{\varphivec^{EA} \in \Hspace^{EA}}\mE^{EA}(\varphivec^{EA})\,.
\]
\subsection{Performance Measures}
We measure the performance in suitably weighted $L_2$ spaces. The weights are given by the probability distribution of pairwise distances, for each type of interaction:
\begin{equation}\label{eq:rhoT}
\rho_{\idxcl_1, \idxcl_2}^T(r) = \E\Big[\frac{1}{N_{\idxcl_1,\idxcl_2}}\sum_{i\in C_{\idxcl_1}}\sum_{j\in C_{\idxcl_2},\,j\neq i}\delta_{r_{i, j}}(r)\Big], \quad \idxcl_1, \idxcl_2 = 1, \cdots, \numcl,
\end{equation}
where  each $N_{\idxcl_1, \idxcl_2}$ is the number of ordered interacting pairs,
\[
N_{\idxcl_1, \idxcl_2} \coloneqq
\begin{cases}
N_{\idxcl_1}N_{\idxcl_2}, & \idxcl_1 \neq \idxcl_2,\\[2pt]
{N_{\idxcl_1}(N_{\idxcl_1} - 1)}, & \idxcl_1 = \idxcl_2,
\end{cases}
\]
and the expectations are over initial conditions distributed according to $\mu_T^{E}$.  We do not have direct access to these distributions, but we can estimate them from training data by constructing the empirical approximations:
\begin{equation}\label{eq:rhoMT}
\rho_{\idxcl_1, \idxcl_2}^{M, T}(r) = \frac{1}{M}\sum_{m = 1}^M \frac{1}{N_{\idxcl_1,\idxcl_2}}\sum_{i\in C_{\idxcl_1}}\sum_{j\in C_{\idxcl_2},\,j\neq i}\delta_{r_{i, j}^m}(r), \quad \idxcl_1, \idxcl_2 = 1, \cdots, \numcl.
\end{equation}
By the law of large numbers, $\rho_{\idxcl_1, \idxcl_2}^{M, T} \rightarrow \rho_{\idxcl_1, \idxcl_2}^T$ as $M \rightarrow \infty$.  The space $L_2(\rho_{\idxcl_1, \idxcl_2}^T)$ is the natural space where the accuracy of the estimators of $\phi^E_{\idxcl_1, \idxcl_2}$ can be measured. The accuracy of the estimators is computed with respect to the weighted inner product and norm defined as
\begin{equation}\label{eq:weighted_norm}
\inpro{\varphi, \psi}_{\rho_{T,\idxcl_1, \idxcl_2}} \coloneqq \int_{r \in \text{supp}(\rho_{T,\idxcl_1, \idxcl_2})} \varphi(r)\,\psi(r)\, r^2 \, d\rho_{T,\idxcl_1, \idxcl_2}(r), \qquad \norm{\varphi}_{\rho_{T,\idxcl_1, \idxcl_2}}^2 \coloneqq \inpro{\varphi, \varphi}_{\rho_{T,\idxcl_1, \idxcl_2}},
\end{equation}
for each $\varphi, \psi \in \Hspace_{\idxcl_1, \idxcl_2}^E$.  The weight $r^2$ is intrinsic to the problem, since each kernel enters the dynamics through the product $\varphi(r_{i,j})\,\br_{i,j}$ with $\norm{\br_{i,j}} = r_{i,j}$.  Here we abuse the notation by letting $\norm{\varphi}_{\rho_{T,\idxcl_1, \idxcl_2}} = \norm{\varphi(\cdot)\,\cdot}_{L^2(\rho_{T,\idxcl_1, \idxcl_2})}$.  Then for $\varphivec^E \in \Hspace^E$, we define the following norm
\[
\norm{\varphivec^E}_{\rho_T}^2 = \sum_{\idxcl_1, \idxcl_2 = 1}^{\numcl}\norm{\varphi^E_{\idxcl_1, \idxcl_2}}_{\rho_{T,\idxcl_1, \idxcl_2}}^2, \quad \rho_T = \oplus_{\idxcl_1, \idxcl_2 = 1}^{\numcl}\rho_{T,\idxcl_1, \idxcl_2}.
\]
As for $\phi^A_{\clof_i, \clof_j}$, we let
\begin{equation}\label{eq:rhoTEA}
\rho_{\idxcl_1, \idxcl_2}^{EA, T}(r, s) = \E\Big[\frac{1}{N_{\idxcl_1,\idxcl_2}}\sum_{i\in C_{\idxcl_1}}\sum_{j\in C_{\idxcl_2},\,j\neq i}\delta_{r_{i, j}, s_{i, j}}(r, s)\Big], \quad s_{i, j} = \norm{\bv_j - \bv_i}\,,
\end{equation}
where the expectation is taken with respect to $\mu_T^{EA}$; a similar definition applies to the empirical version of $\rho_{\idxcl_1, \idxcl_2}^{EA, T}$.  The corresponding inner product and norm are
\[
\inpro{\varphi, \psi}_{\rho^{EA}_{T,k_1,k_2}} \coloneqq \int_{(r, s) \in \text{supp}(\rho^{EA}_{T,k_1,k_2})} \varphi(r)\,\psi(r)\, s^2 \, d\rho^{EA}_{T,k_1,k_2}(r, s), \qquad \norm{\varphi}_{\rho^{EA}_{T,k_1,k_2}}^2 \coloneqq \inpro{\varphi, \varphi}_{\rho^{EA}_{T,k_1,k_2}},
\]
for each $\varphi, \psi \in \Hspace_{\idxcl_1, \idxcl_2}^A$. A similar definition of norm can be used for $\varphivec^A \in \Hspace^A$.

We assess the performance of the estimated kernels from two different points of view: how well $\estphivec$ matches the true kernel $\phivec$, and how well the dynamics driven by $\estphivec$ reproduce the emergent pattern.  Using the weighted inner product $\inpro{\cdot,\cdot}_{\rho_{T,\idxcl_1, \idxcl_2}}$ and the corresponding norm $\norm{\cdot}_{\rho_{T,\idxcl_1, \idxcl_2}}$ in~\eqref{eq:weighted_norm}, we report either the relative $\norm{\cdot}_{\rho_{T,\idxcl_1, \idxcl_2}}$ error of $\estphi^E_{\idxcl_1, \idxcl_2}$ (in well-conditioned regimes) or, when $\estphi^E_{\idxcl_1, \idxcl_2}$ is identified only up to scaling, the angle
\[
\theta_{\idxcl_1, \idxcl_2} \coloneqq \arccos\!\left(\frac{\inpro{\phi_{\idxcl_1, \idxcl_2}, \estphi_{\idxcl_1, \idxcl_2}}_{\rho_{T,\idxcl_1, \idxcl_2}}}{\norm{\phi_{\idxcl_1, \idxcl_2}}_{\rho_{T,\idxcl_1, \idxcl_2}}\,\norm{\estphi_{\idxcl_1, \idxcl_2}}_{\rho_{T,\idxcl_1, \idxcl_2}}}\right), \quad \idxcl_1, \idxcl_2 = 1, \cdots, \numcl\,,
\]
which is invariant to the residual scaling constant.  Since we observe only the final steady-state pattern, we compare patterns through the following trajectory error
\[
\text{Err}_{\text{Traj}} := \frac{1}{M}\sum_{m = 1}^M
\norm{\bX_T^m - \hat\bX_{2T}^m}_{N, K}, \qquad \bX = \begin{bmatrix} \bx_1 \\ \vdots \\ \bx_N\end{bmatrix},
\]
where $\norm{\bX}_{N,\numcl}^2 \coloneqq \sum_{i=1}^N \frac{1}{N_{\clof_i}}\norm{\bx_i}^2$, and the estimated trajectory $\hat\bX_{2T}^m$ is evolved from the observed steady state via
\[
\hat\bX_{2T}^m = \bX_T^m + \int_T^{2T}\F_{\estphivec^E}^E(\hat \bX_t) \, \dt.
\]
For semi-static and quasi-static regimes, where the emergent behavior is characterized by collective velocity alignment or coherent rotation rather than a fixed configuration, $\text{Err}_{\text{Traj}}$ is replaced by regime-specific scores ($I_{\mathrm{flock}}$ and $I_{\mathrm{mill}}$), introduced in Section~\ref{sec:example} where the corresponding patterns first appear.
\subsection{Ill-conditioned estimation of the interaction kernels}\label{sec:ill_conditioned}
The construction of the estimators above can be ill-conditioned, or even ill-posed, both in general (see the discussion in \cite{LZTM2019} on the coercivity condition, and subsequent papers) and, even more so, in our setting.  We consider two regimes, depending on which time derivative vanishes on the observed data from a collective steady state:
\begin{description}
  \item[Static:] The steady state of any first-order system, with $\dot\bx_{i}^m=\zero$ for all $i,m$.
  \item[Semi-static:] A flocking state when all agents share a common constant velocity $\bv_i^m=\bv_*^m\neq\zero$.  Due to the flocking state, the pairwise relative positions are frozen leading to $\dot\bv_{i}^m=\zero$. This occurs not only in second-order systems, but also in multi-species first-order systems with $\dot\bx_{i}^m=\bv_*^m\neq\zero$. However, only the second-order flocking gives rise to ill-conditioned training data.
\end{description}
When the steady state data is given for first-order systems with no non-collective forcing, we have $\dot\bx_i = \zero$ and $\f^E(\bx_i) = \zero$ for $i = 1, \cdots, N$ which leads to
\begin{equation}\label{eq:ill_posed}
\sum_{j = 1, j \neq i}^N\frac{1}{N_{\clof_j}}\phi^E_{\clof_i, \clof_j}(r_{i, j})\br_{i, j} = \zero, \quad i = 1, \cdots, N.
\end{equation}
When flocking data is given for second-order systems with no non-collective forcing, i.e., $\dot\bv_i = \zero$, $\f^{EA}(\bx_i, \bv_i) = \zero$ and $\bv_i = \bv_*$ for $i = 1, \cdots, N$, it leads again to equation \eqref{eq:ill_posed}.  Note that for the flocking data given from second-order systems, the information about $\phi^A_{\idxcl_1, \idxcl_2}$ is lost due to the weighting by $\bv_j - \bv_i$ which is identically $\zero$.  This renders $\phi^A_{\idxcl_1, \idxcl_2}$ unlearnable from the observed data, unless $\phi^A_{\idxcl_1, \idxcl_2}$ is related to $\phi^E_{\idxcl_1, \idxcl_2}$ in some way; we will discuss an example of such a special case in the numerical section.  In order to understand the ill-conditioning better, we introduce the following vectorized notation.  The right-hand side of \eqref{eq:first_order} can be written as
\[
\F_{\varphivec^E}^E(\bX) := \begin{bmatrix}\sum_{j = 2}^N\frac{\varphi^E_{\clof_1, \clof_j}(r_{1, j})}{N_{\clof_j}}\br_{1, j} \\ \vdots \\ \sum_{j = 1}^{N - 1}\frac{\varphi^E_{\clof_N, \clof_j}(r_{N, j})}{N_{\clof_j}}\br_{N, j}\end{bmatrix}\,.
\]
Note that for a fixed and arbitrary $\bX \sim \mu_T^E$, the operator $\F^E_{\bm\cdot}(\bX):\Hspace^E \rightarrow \R^{Nd}$ is linear in the argument in its subscript, i.e. $\F_{c_1\varphivec^E_1+c_2\varphivec^E_2}^E(\bX) = c_1\F_{\varphivec^E_1}^E(\bX)+c_2\F_{\varphivec^E_2}^E(\bX)$ for any $c_1,c_2 \in \R$ and $\varphivec^E_1,\varphivec^E_2\in\Hspace^E$.  Moreover letting
\[
\Cmat := \begin{bmatrix}c_1\bm{I}_{N_1 \times N_1} & \zero_{N_1 \times N_2} & \cdots  & \zero_{N_1\times N_\numcl} \\ \vdots & \vdots & \ddots & \vdots \\ \zero_{N_\numcl\times N_1} & \zero_{N_\numcl \times N_2} & \cdots & c_\numcl\bm{I}_{N_\numcl \times N_\numcl} \end{bmatrix}, \quad \text{for $c_1,\dots,c_\numcl \in \R$},
\]
we still have $\F_{\Cmat\varphivec^E}^E(\bX) = \tilde\Cmat\F_{\varphivec^E}^E(\bX) = \tilde\Cmat\zero = \zero$, where $\tilde\Cmat = \Cmat\otimes\bm{I}_{d\times d}$ ($\otimes$: the Kronecker product).  Furthermore, given the collective steady state training data, both losses defined by \eqref{eq:loss_1st} and \eqref{eq:loss_2nd} simplify to the collective steady state loss
\[
\mE^{\text{CS}}(\varphivec^E) := \sum_{m, i = 1}^{M, N}\frac{1}{MN_{\clof_i}}\norm{\sum_{j = 1, j \neq i}^N\frac{1}{N_{\clof_j}}\varphi^E_{\clof_i, \clof_j}(r^m_{i, j})\br^m_{i, j}}^2.
\]
With the following notation
\[
\bX^{[1:M]} := \begin{bmatrix} \bX^1 \\ \vdots \\ \bX^M\end{bmatrix}, \quad \text{for $\bX^m \in \R^{Nd}$}, \quad \F_{\varphivec^E}(\bX^{[1:M]}) := \begin{bmatrix} \F^E_{\varphivec^E}(\bX^1) \\ \vdots \\ \F^E_{\varphivec^E}(\bX^M)\end{bmatrix},
\]
the loss $\mE^{\text{CS}}$ can be re-written as
\[
\mE^{\text{CS}}(\varphivec^E) = \norm{\F_{\varphivec^E}(\bX^{[1:M]})}_M^2, \quad \text{where $\norm{\bX^{[1:M]}}^2_M := \inpro{\bX^{[1:M]}, \bX^{[1:M]}}_M$},
\]
with the weighted inner product given by
\[
\inpro{\bX^{[1:M]}, \bY^{[1:M]}}_M := \frac{1}{M}\sum_{m = 1}^M\inpro{\bX^m, \bY^m}_{N, \numcl}, \quad \bY^{[1:M]} := \begin{bmatrix} \bY^1 \\ \vdots \\ \bY^M\end{bmatrix}, \quad \text{for $\bY^m \in \R^{Nd}$}.
\]
Here the $\inpro{\cdot, \cdot}_{N, \numcl}$ inner product is defined in Section~\ref{sec:learn}.  Clearly, it has a minimizer on $\Hspace^E$ with $\{\hat\phi^E_{\idxcl, \idxcl'} \equiv \zero\}_{\idxcl, \idxcl' = 1}^{\numcl}$.  The uniqueness of such a minimizer depends on whether the linear operator $\F_{\bm\cdot}(\bX^{[1:M]}):\Hspace^E \rightarrow \R^{MNd}$ has a non-trivial null space.  In order to avoid trivial solutions, we augment the simplified loss $\mE^{\text{CS}}$ with a regularity constraint as follows
\begin{equation}\label{eq:reg_CS_loss}
\mE^{\text{CS}}(\varphivec^E)
=
\norm{\F_{\varphivec^E}(\bX^{[1:M]})}_M^2
=
\sum_{\idxcl=1}^{\numcl}
\mE_{\idxcl}^{\text{CS}}(\varphivec^E_{\idxcl}),
\end{equation}
where
\(
\mE_{\idxcl}^{\text{CS}}(\varphivec^E_{\idxcl})
:=\frac{1}{MN_{\idxcl}}
\norm{
\mathcal A_{\idxcl}^{[1:M]}
(\varphivec^E_{\idxcl})
}^2.
\)
For each receiving type $\idxcl$, we solve
\[
\estphivec^E_{\idxcl}
\in
\argmin_{\varphivec^E_{\idxcl}\in\Hspace^E_{\idxcl}}
\mE_{\idxcl}^{\text{CS}}(\varphivec^E_{\idxcl}),
\qquad
\text{subject to }
\norm{\varphivec^E_{\idxcl}}_{\rho_{T,\idxcl}}=1.
\]
After restriction to a finite-dimensional hypothesis space, \eqref{eq:reg_CS_loss} becomes a generalized Rayleigh-quotient problem, with the equilibrium residual defining the numerator and the $\rho_T$-weighted kernel norm defining the normalization. The corresponding generalized eigenvalue formulation is given in Section~\ref{subsec:spectral-characterization}.
\subsection{Identifiability}\label{sec:identify}
We present the following theorem characterizing when the interaction kernels can be identified from the data.
\begin{theorem}[Identifiability for collective steady state training data]\label{thm:identifiability}
Assume that the data are generated from steady states of first-order systems or from flocking states of second-order systems, so that only \(\{\bX^m\}_{m=1}^M\) is observed.  For each receiving type \(\idxcl\), write \(\varphivec^E_{\idxcl}=\{\varphi^E_{\idxcl,\idxcl'}\}_{\idxcl'=1}^{\numcl}\in\Hspace^E_{\idxcl}\coloneqq\oplus_{\idxcl'=1}^{\numcl}\Hspace^E_{\idxcl,\idxcl'}\) and \(\rho_{T,\idxcl}=\oplus_{\idxcl'=1}^{\numcl}\rho_{T,\idxcl,\idxcl'}\), and define the linear operator \(\mathcal{A}_{\idxcl}^{[1:M]}:\Hspace^E_{\idxcl}\rightarrow\R^{MN_{\idxcl}d}\) by
\[
\mathcal{A}_{\idxcl}^{[1:M]}(\varphivec^E_{\idxcl}) = \begin{bmatrix} \F_{\varphivec^E_{\idxcl}}^{\idxcl, E}(\bX^1) \\ \vdots \\ \F_{\varphivec^E_{\idxcl}}^{\idxcl, E}(\bX^M)\end{bmatrix}, \quad \F^{\idxcl, E}_{\varphivec^E_{\idxcl}}(\bX) = \begin{bmatrix} \vdots \\ \sum_{j = 1, j \neq i}^N\frac{\varphi^E_{\idxcl, \clof_j}(r_{i, j})}{N_{\clof_j}}\br_{i, j} \\ \vdots\end{bmatrix} \in \R^{N_\idxcl d}, \quad i \in \cl_\idxcl.
\]
Let \(\mathcal{N}_{\idxcl}^{[1:M]}=\{\varphivec^E_{\idxcl}\in\Hspace^E_{\idxcl}:\mathcal{A}_{\idxcl}^{[1:M]}(\varphivec^E_{\idxcl})=\zero\}\) be its null space, and let \(\estphivec^E_{\idxcl}\) be a minimizer of the constrained collective-state loss~\eqref{eq:reg_CS_loss} on \(\Hspace^E_{\idxcl}\).  The same formula for \(\mathcal{A}_{\idxcl}^{[1:M]}\) extends it off \(\Hspace^E_{\idxcl}\).

If \(\dim\mathcal N_{\idxcl}^{[1:M]}=1\) and the true block \(\phivec^E_{\idxcl}\) belongs to \(\Hspace^E_{\idxcl}\), then \(\mathcal N_{\idxcl}^{[1:M]}=\operatorname{span}\{\phivec^E_{\idxcl}\}\) and the constrained minimizers are
\[
\estphivec^E_{\idxcl}
=\pm\frac{\phivec^E_{\idxcl}}{\norm{\phivec^E_{\idxcl}}_{\rho_{T,\idxcl}}}\,.
\]
In particular, the direction of \(\phivec^E_{\idxcl}\) is identifiable,
while the scale and the sign cannot be determined from the collective-state
snapshots alone. Equivalently,
\[
\phivec^E_{\idxcl}
=
c_{\idxcl,*}\estphivec^E_{\idxcl},
\qquad
c_{\idxcl,*}
=
\pm\norm{\phivec^E_{\idxcl}}_{\rho_{T,\idxcl}}.
\]
Assembling over types,
\[
\phivec^E=\Cmat_*\estphivec^E,
\qquad 
\Cmat_* =
\begin{bmatrix}
c_{1,*}\bm{I}_{N_1\times N_1}
& \cdots
& \zero_{N_1\times N_\numcl}
\\
\vdots
& \ddots
& \vdots
\\
\zero_{N_\numcl\times N_1}
& \cdots
& c_{\numcl,*}\bm{I}_{N_\numcl\times N_\numcl}
\end{bmatrix}.
\]
\end{theorem}
\begin{proof}
The observed snapshots satisfy the equilibrium identities~\eqref{eq:ill_posed}: $\F_{\phivec^E}^E(\bX^m)=\zero$ for each $m=1,\dots,M$.  The same identities hold for second-order flocking, since $\bv_i^m\equiv\bv_*^m$ forces $\bu_{i,j}^m=\zero$ and the alignment kernels drop out of~\eqref{eq:loss_2nd}.  By the linearity of $\F^E_{\bm\cdot}(\bX^m)$ recorded in Section~\ref{sec:ill_conditioned}, the identities persist along rays, $\F_{c\,\phivec^E}^E(\bX^m)=\zero$ for every $c\in\R$.  Moreover the residual of an agent of type $\idxcl$ depends only on the receiving-type block $\varphivec^E_{\idxcl}=\{\varphi^E_{\idxcl,\idxcl'}\}_{\idxcl'=1}^{\numcl}$.  Consequently the collective-state loss decouples over receiving types,
\begin{equation}\label{eq:CS-decouple}
\mE^{\text{CS}}(\varphivec^E)
=\sum_{\idxcl=1}^{\numcl}
\frac1M\sum_{m=1}^M
\frac1{N_{\idxcl}}
\Bigl\|\F^{\idxcl,E}_{\varphivec^E_{\idxcl}}(\bX^m)\Bigr\|^2
=\sum_{\idxcl=1}^{\numcl}\frac{1}{MN_{\idxcl}}
\bigl\|\mathcal A_{\idxcl}^{[1:M]}(\varphivec^E_{\idxcl})\bigr\|^2\,,
\end{equation}
and it is enough to identify each block separately.

Fix a receiving type $\idxcl$.  If the true block $\phivec^E_{\idxcl}$ lies in $\Hspace^E_{\idxcl}$, the equilibrium identities give $\mathcal A_{\idxcl}^{[1:M]}(\phivec^E_{\idxcl})=\zero$, hence
\[
\operatorname{span}\{\phivec^E_{\idxcl}\}\;\subset\;\mathcal N_{\idxcl}^{[1:M]}\,.
\]
In particular \(\dim\mathcal N_{\idxcl}^{[1:M]}\ge 1\).  The hypothesis \(\dim\mathcal N_{\idxcl}^{[1:M]}=1\) therefore forces \(\mathcal N_{\idxcl}^{[1:M]}=\operatorname{span}\{\phivec^E_{\idxcl}\}\).

The constrained estimator~\eqref{eq:reg_CS_loss} on type \(\idxcl\) is the unit-norm problem
\[
\min_{\varphivec^E_{\idxcl}\in\Hspace^E_{\idxcl}}
\bigl\|\mathcal A_{\idxcl}^{[1:M]}(\varphivec^E_{\idxcl})\bigr\|^2
\quad\text{subject to}\quad
\norm{\varphivec^E_{\idxcl}}_{\rho_{T,\idxcl}}=1\,.
\]
By~\eqref{eq:CS-decouple} the unconstrained minimum is zero and is attained precisely on \(\mathcal N_{\idxcl}^{[1:M]}\).  Intersecting with the unit sphere in \(L^2(\rho_{T,\idxcl})\) yields the two minimizers asserted in the statement.  The scale and sign of \(\phivec^E_{\idxcl}\) are invisible to~\eqref{eq:reg_CS_loss}: both correspond to the same vanishing residual.  The sign is selected by the energy comparison of Section~\ref{sec:min_energy}; the scale is recovered from stopping times in Algorithm~\ref{alg:timescale}.  Assembling over types gives the block-diagonal \(\Cmat_*\) of the statement.
\end{proof}
Note that if \(\phivec^E_{\idxcl}\notin\Hspace^E_{\idxcl}\), the same argument identifies \(\mathcal N_{\idxcl}^{[1:M]}=\ker\bigl(\mathcal A_{\idxcl}^{[1:M]}\big|_{\Hspace^E_{\idxcl}}\bigr)\).  Linearity of the extended operator gives
\[
\mathcal A_{\idxcl}^{[1:M]}(P_{\Hspace^E_{\idxcl}}{\phivec^E_{\idxcl}})
=-\mathcal A_{\idxcl}^{[1:M]}(\phivec^E_{\idxcl}-P_{\Hspace^E_{\idxcl}}{\phivec^E_{\idxcl}})\,,
\]
so \(P_{\Hspace^E_{\idxcl}}{\phivec^E_{\idxcl}}\in\mathcal N_{\idxcl}^{[1:M]}\) if and only if the approximation error produces no net force on the observed configurations.  When that identity holds and \(\dim\mathcal N_{\idxcl}^{[1:M]}=1\), the unit-norm minimizers are those of the statement.  If the residual is merely small, the same conclusion holds approximately, in the sense of the spectral gap of Proposition~\ref{prop:spectral-gap} which we now discuss.

\begin{remark}
Note that the theorem provides guarantees in the weighted $L^2(\rho_T)$ norm. This is meaningful if $\rho_T$ is spread out on the support of the interaction kernel, but it becomes less useful if $\rho_T$ concentrates at one or a few points, or has significant mass outside of the support of the interaction kernel. Both these situations can occur, for example for a locally, or globally, attractive interaction kernel particles may collapse at one or a few points when they reach a collective state (e.g., consensus in opinion dynamics). Clearly, in these situations the problem is fully ill-posed, and the conclusion of the Theorem, even if its hypotheses were satisfied, is not useful, and the interaction kernel is not learnable in any stronger, useful norm.
\end{remark}

\subsection{Finite-dimensional spectral characterization}\label{subsec:spectral-characterization}
Fix a receiving type \(k\). For
\(
    \balpha=(\alpha_1,\ldots,\alpha_{n_k})^\top
\),
write
\(
    \varphivec^E_{k,\balpha}
    =
    \sum_{\ell=1}^{n_k}\alpha_\ell\psi_{k,\ell},
\)
where the basis functions are understood componentwise over
\(
    \Hspace_k^E
    =
    \bigoplus_{k'=1}^{\numcl}\Hspace^E_{k,k'}.
\)
Define the matrices
\[
    (H_k)_{\ell q}
    :=
\frac{1}{MN_k}
\left\langle
\mathcal A_k^{[1:M]}\psi_{k,\ell},
\mathcal A_k^{[1:M]}\psi_{k,q}
\right\rangle,
    \qquad
    (G_k)_{\ell q}
    :=
    \left\langle
        \psi_{k,\ell},
        \psi_{k,q}
    \right\rangle_{\rho_k^T},
\]
where
\(
    \rho_k^T=\bigoplus_{k'=1}^{\numcl}\rho_{k,k'}^T
\).
Assuming that \(G_k\) is positive definite, the constrained problem
\eqref{eq:reg_CS_loss}, restricted to \(\Hspace_k^E\), is equivalent to
the generalized eigenvalue problem
\[
    H_k\balpha=\lambda G_k\balpha,
    \qquad
    \balpha^\top G_k\balpha=1.
\]
Let
\(
    0\leq\lambda_{k,1}
    \leq\lambda_{k,2}
    \leq\cdots\leq\lambda_{k,n_k}
\)
be its generalized eigenvalues, with \(G_k\)-orthonormal eigenvectors \(\{\balpha_{k,j}\}_{j=1}^{n_k}\).
\begin{proposition}[Spectral gap and directional stability]\label{prop:spectral-gap}
Assume that the smallest generalized eigenvalue is simple, namely
\(
    \gamma_k
    :=
    \lambda_{k,2}-\lambda_{k,1}
    >0.
\)
Then every \(\balpha\in\mathbb R^{n_k}\) satisfying
\(
    \balpha^\top G_k\balpha=1
\)
obeys
\[
    \sin^2
    \angle_{G_k}
    \bigl(
        \balpha,\balpha_{k,1}
    \bigr)
    \leq
    \frac{
        \balpha^\top H_k\balpha-\lambda_{k,1}
    }{
        \lambda_{k,2}-\lambda_{k,1}
    }.
\]
In particular, if the data are exact,
\(
    \dim(\mathcal N_k^{[1:M]})=1
\),
and the corresponding kernel belongs to \(\Hspace_k^E\), then
\(
    \lambda_{k,1}=0<\lambda_{k,2}
\)
and
\[
    \sin
    \angle_{G_k}
    \bigl(
        \balpha,\balpha_{k,1}
    \bigr)
    \leq
    \frac{
        \bigl(
            \balpha^\top H_k\balpha
        \bigr)^{1/2}
    }{
        \lambda_{k,2}^{1/2}
    }.
\]
\end{proposition}
\begin{proof}
Expand
\(
    \balpha=\sum_{j=1}^{n_k}c_j\balpha_{k,j}
\)
in the \(G_k\)-orthonormal generalized eigenbasis. Since
\(
    \balpha^\top G_k\balpha=1
\),
we have \(\sum_j c_j^2=1\), and therefore
\[
\begin{aligned}
    \balpha^\top H_k\balpha-\lambda_{k,1}
    =
    \sum_{j=2}^{n_k}
    (\lambda_{k,j}-\lambda_{k,1})c_j^2 \geq
    \gamma_k\sum_{j=2}^{n_k}c_j^2
    =
    \gamma_k
    \sin^2\angle_{G_k}
    (\balpha,\balpha_{k,1}).
\end{aligned}
\]
Rearranging proves the claim.
\end{proof}
The quantity \(\gamma_k\) measures the separation of the learned kernel direction from the remaining directions in the hypothesis space. Hence, one-dimensional nullity gives exact directional identifiability, whereas a positive and sufficiently large spectral gap gives quantitative stability. Small gaps indicate that additional kernel directions produce nearly the same equilibrium residual and are therefore difficult to distinguish under sampling error, observation noise, or model mismatch.
\begin{remark}[Connection with equilibrium stresses in rigidity theory]\label{rem:rigidity-self-stress}
Consider the homogeneous static setting, with \(\numcl=1\) and \(\f^E=\zero\). Let \(G=(V,E)\) be the interaction graph of an observed configuration \(\bX=(\bx_1,\ldots,\bx_N)\), and let \(R(\bX)\) denote its rigidity matrix. An edge-weight vector
\(
    \boldsymbol{\omega}
    =(\omega_{ij})_{\{i,j\}\in E}
\)
is a self-stress of the framework \((G,\bX)\) if $ R(\bX)^\top\boldsymbol{\omega}=\zero$, or, equivalently,
\[
    \sum_{j:\{i,j\}\in E}
    \omega_{ij}(\bx_j-\bx_i)=\zero,
    \qquad i=1,\ldots,N;
\]
see, e.g.,~\cite{connelly1982rigidity}.

In the present problem, the edge weights are constrained to be generated by a common radial kernel. Define the evaluation operator
\[
    \mathcal E_{\bX}:\Hspace^E\longrightarrow\R^{|E|},
    \qquad
    [\mathcal E_{\bX}\varphi^E]_{ij}
    =
    \varphi^E(r_{ij}).
\]
The static observation operator for one configuration then factors as $\mathcal A_{\bX} = R(\bX)^\top\mathcal E_{\bX}$ and the operator \(\mathcal A^{[1:M]}\) is obtained by stacking these factorizations over the \(M\) observed configurations. Consequently,
\[
    \mathcal N^{[1:M]}
    =
    \bigcap_{m=1}^M
    \mathcal E_{\bX^m}^{-1}
    \left(
        \ker R(\bX^m)^\top
    \right).
\]
Thus, learning a kernel from static configurations amounts to identifying radial self-stresses shared by the observed data. Directional identifiability is precisely the condition $\dim(\mathcal N^{[1:M]})=1$. When the stacked evaluation operator is injective on \(\Hspace^E\), this is equivalent to requiring a one-dimensional intersection between the admissible radial stress space and the self-stress constraints. Its dimension and conditioning depend on both the observed configurations and the chosen hypothesis space; in particular, geometric symmetries may reduce the number of independent equilibrium constraints.

For heterogeneous or nonreciprocal interactions, the same factorization has a directed block structure and should be viewed as an extension of the classical self-stress formulation.
\end{remark}
The condition, $\text{dim}(\mathcal{N}^{[1:M]}_{\idxcl}) \le 1$, can be relaxed; however in our experiments, $\text{dim}(\mathcal{N}^{[1:M]}_{\idxcl})$ is never bigger than $1$.  When $\text{dim}(\mathcal{N}^{[1:M]}_{\idxcl}) > 1$, (again in the homogeneous agents case), 
\[
\phi^E \approx c_{1, *}\hat\phi_1 + \cdots + c_{\beta, *}\hat\phi_{\beta}, \quad \{\hat\phi_\eta\}_{\eta = 1}^{\beta} \text{a basis}(\mathcal{N}^{[1:M]}_{\idxcl}).
\]
Via the minimal-energy direction, we can figure out the signs of $c_{\eta, *}$'s, however to pin down the exact values of each $c_{\eta, *}$, we will need at least $\{\bX^m(0), \{T_m^{(\idxcl)}\}_{\idxcl=1}^{\numcl}\}_{m = 1}^{\beta}$.  Note that data distributions $\mu^E_T$ and $\rho_T$ are two different indicators.  One might need a large number of samples to approximate $\mu^E_T$ well, but for $\rho_T$, sometimes $M = 1$ can have $\rho_T^M = \rho_T$, see the example presented in Section~\ref{sec:r-1}.
\subsection{Minimal Energy Direction}\label{sec:min_energy}
The estimator $\estphivec^E$ can only be learned up to a block-structure scaling matrix $\Cmat$ with each independent $c_\idxcl$ being either positive or negative, since we only use the information $\F^E_{\phivec}(\bX^m) = \zero$ for $m = 1, \cdots, M$.  However, for gradient-flow systems, we can use the additional information that when it reaches steady state, it also reaches its minimal energy state.  Hence, given $\{\bX^m\}_{m = 1}^M$, we obtain $\estphivec^{E, \text{test}}$ and the corresponding energy $\hat{E}_{\estphivec^{E, \text{test}}}$; then we can consider a small perturbation $\bX^m + \Delta\bX^m$, and find a sign matrix $\Cmat_{\text{sign}}^{j_*}$, constructed with $c_\idxcl$'s in $\{\pm 1\}$, such that $\hat{E}_{\Cmat_{\text{sign}}^{j_*}\estphivec^{E, \text{test}}}(\bX^m + \Delta\bX^m) > \hat{E}_{\Cmat_{\text{sign}}^{j_*}\estphivec^{E, \text{test}}}(\bX^m)$ for all (or most) $m$'s. We then take $\estphivec^E = \Cmat_{\text{sign}}^{j_*}\estphivec^{E, \text{test}}$.
\subsection{Time-Scale Recovery}\label{sec:timescale}
\begin{algorithm}[t]
\caption{Time-scale recovery of $\Cmat_*=\mathrm{diag}(c_{1, *}\bm{I}_{N_1\times N_1},\dots,c_{\numcl, *}\bm{I}_{N_\numcl\times N_\numcl})$}
\label{alg:timescale}
\begin{algorithmic}[1]
\Require $\estphivec^E$,
$\{\bX^m(0), \{T_m^{(\idxcl)}\}_{\idxcl=1}^{\numcl}\}_{m=1}^{M_{\mathrm{est}}}$,
tolerance $\mathrm{Tol}$; search domain $\mathcal C\subset\R_{+}^{\numcl}$.
\Statex
\If{$\numcl=1$} \State set $T_m := T_m^{(1)}$ for $m=1,\dots,M_{\mathrm{est}}$
  \For{$m=1,\dots,M_{\mathrm{est}}$}
    \State evolve $\dot\bX=\F_{\hat\phi^E}(\bX)$
    from $\bX^m(0)$ to its stopping time $\hat T_m$
  \EndFor
  \State \Return
  $\displaystyle
  \hat c_*=
  \frac{1}{M_{\mathrm{est}}}
  \sum_{m=1}^{M_{\mathrm{est}}}
  \frac{\hat T_m}{T_m}$
\ElsIf{$\numcl > 1$}
  \For{each candidate $\Cmat = \mathrm{diag}(c_1\bm{I}_{N_1\times N_1},\dots,c_{\numcl}\bm{I}_{N_\numcl\times N_\numcl}) \in \mathcal C$}
    \For{each initial condition $\bX^m(0)$}
      \State evolve $\dot\bX = \F_{\Cmat\hat\phivec}(\bX)$ from $\bX^m(0)$
      \For{each species $\idxcl = 1, \dots, \numcl$}
        \State $\hat T^{(\idxcl)}_m(\Cmat) \gets$ first time $t$ with $\big(\tfrac{1}{N_\idxcl}\sum_{i\in \cl_\idxcl}\norm{\dot\bx_i(t)}^2\big)^{1/2}\le\mathrm{Tol}$
      \EndFor
    \EndFor
    \State $R(\Cmat) \gets \displaystyle\sum_{m=1}^{M_{\mathrm{est}}}\sum_{k=1}^{\numcl} \big(\hat T^{(k)}_m(\Cmat) - T^{(k)}_m\big)^2$
  \EndFor
  \State \Return $\hat\Cmat_* \gets \arg\min_{\Cmat \in\mathcal C} R(\Cmat)$
\EndIf
\end{algorithmic}
\end{algorithm}
By Theorem~\ref{thm:identifiability}, when collective steady state training data is given, $\phivec^E$ is identifiable only up to the block-structure scaling matrix $\Cmat_*$, i.e. $\phivec^E \approxeq \Cmat_*\estphivec^E$, where
\[
\Cmat_* = \begin{bmatrix}c_{1, *}\bm{I}_{N_1 \times N_1} & \cdots & \zero_{N_1\times N_\numcl} \\
\vdots & \ddots & \vdots \\
\zero_{N_\numcl\times N_1} & \cdots & c_{\numcl, *}\bm{I}_{N_\numcl \times N_\numcl} \end{bmatrix}, \quad \text{for each $c_{\idxcl, *} > 0$ once the sign is fixed (Section~\ref{sec:min_energy})}.
\]
Without any additional information, $\Cmat_*$ remains unlearnable.  If, in addition to the final state patterns, we are given some $\{\bX^m(0), \{T_m^{(\idxcl)}\}_{\idxcl=1}^{\numcl}\}_{m=1}^{M_{\mathrm{est}}}$ for some $M_{\mathrm{est}} \ge 1$, $\Cmat_*$ can be recovered.  For single species systems, the trajectory driven by $c\hat\phi^E$ follows the same path as the one driven by $\hat\phi^E$ but reaches steady state faster as $c > 0$ increases, so stopping times determine $c_*$ in closed form. For multi-species systems, we recover the receiving-type scaling
parameters by matching the simulated and observed stopping times,
as described in Algorithm~\ref{alg:timescale}.
\subsection{Data-driven partitions}
It is of utmost importance that the spaces ($\Hspace^E$ and $\Hspace^A$) which we use to approximate the interaction kernels do not introduce additional ill-conditioning.  We require that for each $\varphi_{\idxcl, \idxcl'} \in \Hspace_{\idxcl, \idxcl'}^E$ (or $\Hspace_{\idxcl, \idxcl'}^A$), $\text{supp}(\varphi_{\idxcl, \idxcl'}) \cap \text{supp}(\rho^{E, T}_{\idxcl, \idxcl'}) \neq \emptyset$.  Hence we construct the basis functions for $\Hspace_{\idxcl, \idxcl'}^E$ (or $\Hspace_{\idxcl, \idxcl'}^A$) with splines or piecewise polynomials on a data-driven partition of $[r^{\min}_{\idxcl, \idxcl'}, r^{\max}_{\idxcl, \idxcl'}]$ for $\idxcl, \idxcl' = 1, \cdots, \numcl$ where
\[
r^{\min}_{\idxcl, \idxcl'} = \min_{i \in \cl_\idxcl, j \in \cl_{\idxcl'}} \norm{\bx_j^m - \bx_i^m},
\qquad
r^{\max}_{\idxcl, \idxcl'} = \max_{i \in \cl_\idxcl, j \in \cl_{\idxcl'}} \norm{\bx_j^m - \bx_i^m}.
\]
A uniform partition of $[r^{\min}_{\idxcl, \idxcl'}, r^{\max}_{\idxcl, \idxcl'}]$ may place basis functions where $\rho_{T,\idxcl, \idxcl'}$ carries little or no mass, wasting degrees of freedom and worsening the conditioning of the learning problem.  Concentrating the basis where $\rho_{T,\idxcl, \idxcl'}$ has mass keeps the function space $\Hspace^{E/A}_{\idxcl, \idxcl'}$ well-conditioned on the observable subspace.  We refine $[r^{\min}_{\idxcl, \idxcl'}, r^{\max}_{\idxcl, \idxcl'}]$ adaptively so that no sub-interval carries more empirical mass than a prescribed tolerance $P^{\mathrm{tol}}_{\idxcl, \idxcl'} \in (0, 1)$, i.e., for any sub-interval $I \subset [r^{\min}_{\idxcl, \idxcl'}, r^{\max}_{\idxcl, \idxcl'}]$, the empirical mass is $\mu_{\idxcl, \idxcl'}(I) = \int_I \rho_{T,\idxcl, \idxcl'}(r)\, dr$ evaluated from the histogram estimate of $\rho_{T,\idxcl, \idxcl'}$. The refinement procedure is described in Algorithm~\ref{alg:adaptive-partition}.
\begin{algorithm}[H]
\caption{Adaptive mass-balanced partition of $[r^{\min}_{\idxcl, \idxcl'}, r^{\max}_{\idxcl, \idxcl'}]$.}
\label{alg:adaptive-partition}
\begin{algorithmic}[1]
\Require $r^{\min}_{\idxcl, \idxcl'} < r^{\max}_{\idxcl, \idxcl'}$; $P^{\mathrm{tol}}_{\idxcl, \idxcl'}$; $\text{Iter}_{\max}$; $\rho^{M}_{T,\idxcl, \idxcl'}$
\State set $\{r^{\min}_{\idxcl, \idxcl'}, r^{\max}_{\idxcl, \idxcl'}\} \rightarrow \mathcal{P}_{\idxcl, \idxcl'}$ and $0 \rightarrow \text{Iter}$
\While{$\text{Iter} < \text{Iter}_{\max}$}
    \State flag every sub-interval $I$ of $\mathcal{P}$ with $\mu_{\idxcl, \idxcl'}(I) > P^{\mathrm{tol}}_{\idxcl, \idxcl'}$
    \If{no sub-interval is flagged}
        \State \textbf{break}
    \EndIf
    \State bisect each flagged sub-interval and insert the midpoint into $\mathcal{P}_{\idxcl, \idxcl'}$
    \State $\text{Iter} \leftarrow \text{Iter} + 1$
\EndWhile
\Return $\mathcal{P}_{\idxcl, \idxcl'}$
\end{algorithmic}
\end{algorithm}
The output partition is approximately $\rho_{T,\idxcl, \idxcl'}$-equidistributed: each cell carries at most $P^{\mathrm{tol}}_{\idxcl, \idxcl'}$ amount of the empirical mass, so basis functions concentrate where the  data does.  We perform this refinement  for $1 \le \idxcl \le \idxcl' \le \numcl$.

In the finite-dimensional formulation of Section~\ref{subsec:spectral-characterization}, this construction also helps keep the Gram matrix \(G_\idxcl\) well-conditioned by removing basis directions that are weakly observed under \(\rho_\idxcl^T\).
\section{Examples}\label{sec:example}
We validate the proposed learning framework in the three data regimes, static and semi-static identified in Section~\ref{sec:learn} and an additional regime called quasi-static where neither positions nor velocities reach a fixed value, yet an organized pattern of dynamics, i.e. a collective steady state, emerges. We use the examples below to illustrate how the data regimes and the observed pairwise-distance distribution $\rho_T$ affect kernel recovery and dynamical prediction. We consider static first-order patterns, including homogeneous rings, hexagonal crystals, soccer-ball configurations, and heterogeneous predator--prey aggregates; semi-static flocking states in first- and second-order systems; and quasi-static milling states in second-order self-propelled particle systems.

We evaluate the learned interactions at both the kernel and dynamical levels.
At the kernel level, we distinguish between the nonzero-right-hand-side
regime, where the kernels can be compared directly in the weighted norm
$\norm{\cdot}_{\rho_T}$, and the scale-ambiguous zero-right-hand-side regime,
where we report angular errors and recover the missing scaling constants.
At the dynamical level, we test whether the learned interaction preserves
the observed terminal pattern and whether it generates the correct emergent
pattern from new initial configurations.

All experiments use simulated training data. The numerical integration is performed with \texttt{solve\_ivp} from the SciPy package in Python, using the LSODA method with tolerances $\mathrm{rtol}=10^{-9}$ and $\mathrm{atol}=10^{-12}$. Unless otherwise specified, the common testing parameters are summarized in Table~\ref{table:common_param}.  Here $M$ is the number of training initial conditions (ICs), and $M_{\rm true}$ is the larger reference ensemble used to approximate the background measure $\rho_T^{M_{\rm true}}$ shown in the kernel-comparison figures.
\begin{table}[H]
\centering
\begin{tabular}{c | c | c | c | c | c} 
\hline
$M$ & $M_{\rm true}$ & $N$ & $d$ & \# Trials & Basis \\
\hline
$250$ & $2000$ & $40$ & $2$ & $10$ & Adaptive B-spline \\
\hline
\end{tabular}
\caption{Common testing parameters used. }
\label{table:common_param}
\end{table}
To test robustness, we perturb the observed positions as $\tilde\bx_i^m = \bx_i^m + \gamma_i^m(\cos\vartheta_i^m, \sin\vartheta_i^m)^\top$, with independent $\gamma_i^m \sim \mathcal N(0,\sigma^2)$ and $\vartheta_i^m \sim \mathcal U[0,2\pi)$, where the range of $\sigma$ is chosen per example and shown in the corresponding noise figures.

We use different stopping criteria according to the type of emergent pattern. For static first-order states, the simulation is terminated when the root-mean-squared speed falls below $\mathrm{Tol}$:
\begin{equation}
   \left(\sum_{i=1}^N  \frac{1}{N_{\clof_i}}\norm{\dot{\bx}_i}^2 \right)^{1/2} \le \mathrm{Tol}.
   \label{eq:c1}
\end{equation}
For first-order flocking states, where the agents may translate with a common nonzero velocity, convergence is monitored by the relative velocity dispersion:
\begin{equation}
   \left( \sum_{i=1}^N\frac{1}{N_{\clof_i}}\norm{\dot{\bx}_i-\bar{\bv}}^2 \right)^{1/2} \le \mathrm{Tol},
   \qquad
   \bar{\bv}:=\frac{1}{N}\sum_{i=1}^N \dot{\bx}_i .
   \label{eq:c1-flock}
\end{equation}
For semi-static states in second-order systems, the corresponding criterion is applied to the acceleration:
\begin{equation}
    \left(\sum_{i=1}^N\frac{1}{N_{\clof_i}}\norm{\dot{\bv}_i}^2 \right)^{1/2} \le \mathrm{Tol}.
    \label{eq:c2}
\end{equation}
For quasi-static milling states, we first define two order parameters, the flocking score $I_{\mathrm{flock}}$ and milling score $I_{\mathrm{mill}}$, as follows

\begin{minipage}{0.45\textwidth}
\begin{equation}\label{eq:Imill}
I_{\mathrm{mill}} \coloneqq
\frac{
\displaystyle\sum_{i=1}^N \left| (\bx_i - \bar{\bx}) \times \bv_i \right|
}{
\displaystyle\sum_{i=1}^N \| \bx_i - \bar{\bx} \| \, \| \bv_i \|
}
\end{equation}
\end{minipage}\,
\begin{minipage}{0.45\textwidth}
\begin{equation}\label{eq:Iflock}
I_{\mathrm{flock}} \coloneqq
\frac{\left\| \displaystyle\sum_{i=1}^N \bv_i \right\|}{\displaystyle\sum_{i=1}^N \|\bv_i\|}
\end{equation}
\end{minipage}

We report the scores for flocking and milling in the corresponding semi-static and quasi-static examples. For first-order moving flocks, we evaluate $I_{\mathrm{flock}}$ with $\bv_i$ replaced by the instantaneous velocity $\dot{\bx}_i$. Unless otherwise specified, we set $\mathrm{Tol}=10^{-9}$ for static and semi-static regimes and $\mathrm{Tol}=10^{-6}$ for quasi-static milling states.
\subsection{Static Patterns from First-Order Systems}
We begin with static configurations of the first-order system~\eqref{eq:first_order}, where the dynamics evolves to an equilibrium with $\dot\bx_i=\zero$. All four examples below therefore fall into the collective steady state training data scenario of Theorem~\ref{thm:identifiability}: each kernel is identifiable only up to the scaling matrix $\Cmat_*$. Accordingly, we report the angle $\theta_{\idxcl, \idxcl'}$ between $\hat\phi^E_{\idxcl, \idxcl'}$ and $\phi^E_{\idxcl, \idxcl'}$ (invariant to the scaling), the constant recovered by the first-stopping-time procedure, and the trajectory error $\mathrm{Err}_{\mathrm{Traj}}$. Figures~\ref{fig:1_phi_comp} and~\ref{fig:MS-static_phi_comp}, together with Tables~\ref{tab:traj-performance} and \ref{tab:cstar-recovery}, collect these quantities across all four cases. Throughout, each entry is reported as mean $\pm$ standard deviation across the $10$ learning trials; for $\mathrm{Err}_{\mathrm{Traj}}$ and $\hat\Cmat_*$ the per-trial value is itself the mean (or std) over the $M$ trajectories. The subsections that follow examine each in turn.
\begin{figure}[H]
    \centering
    \begin{subfigure}[t]{0.32\linewidth}
        \centering
        \includegraphics[width=\linewidth]{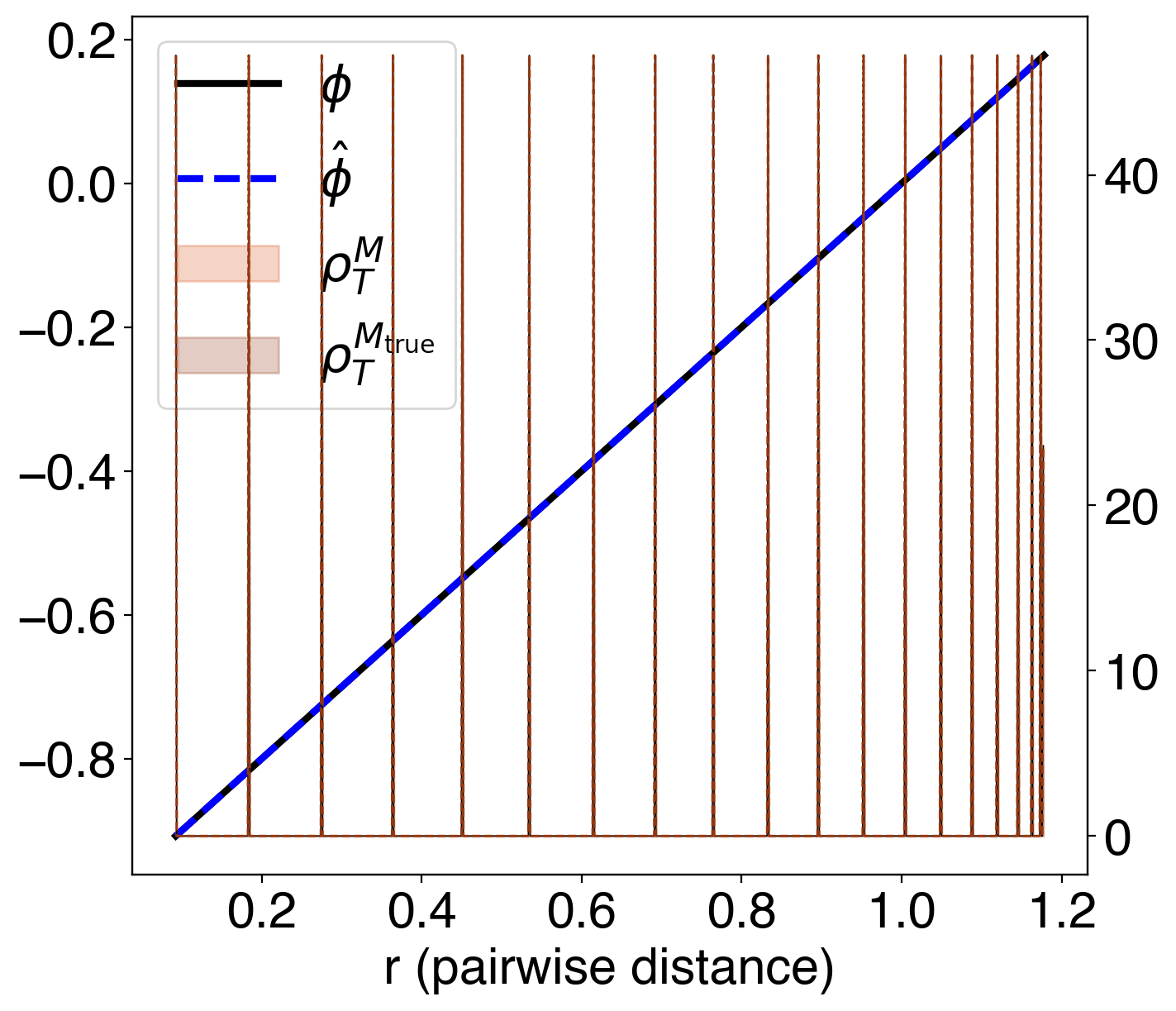}
        \caption{$r-1$: ring}
        \label{fig:r-1_phi_comp}
    \end{subfigure}
    \hfill
    \begin{subfigure}[t]{0.32\linewidth}
        \centering
        \includegraphics[width=\linewidth]{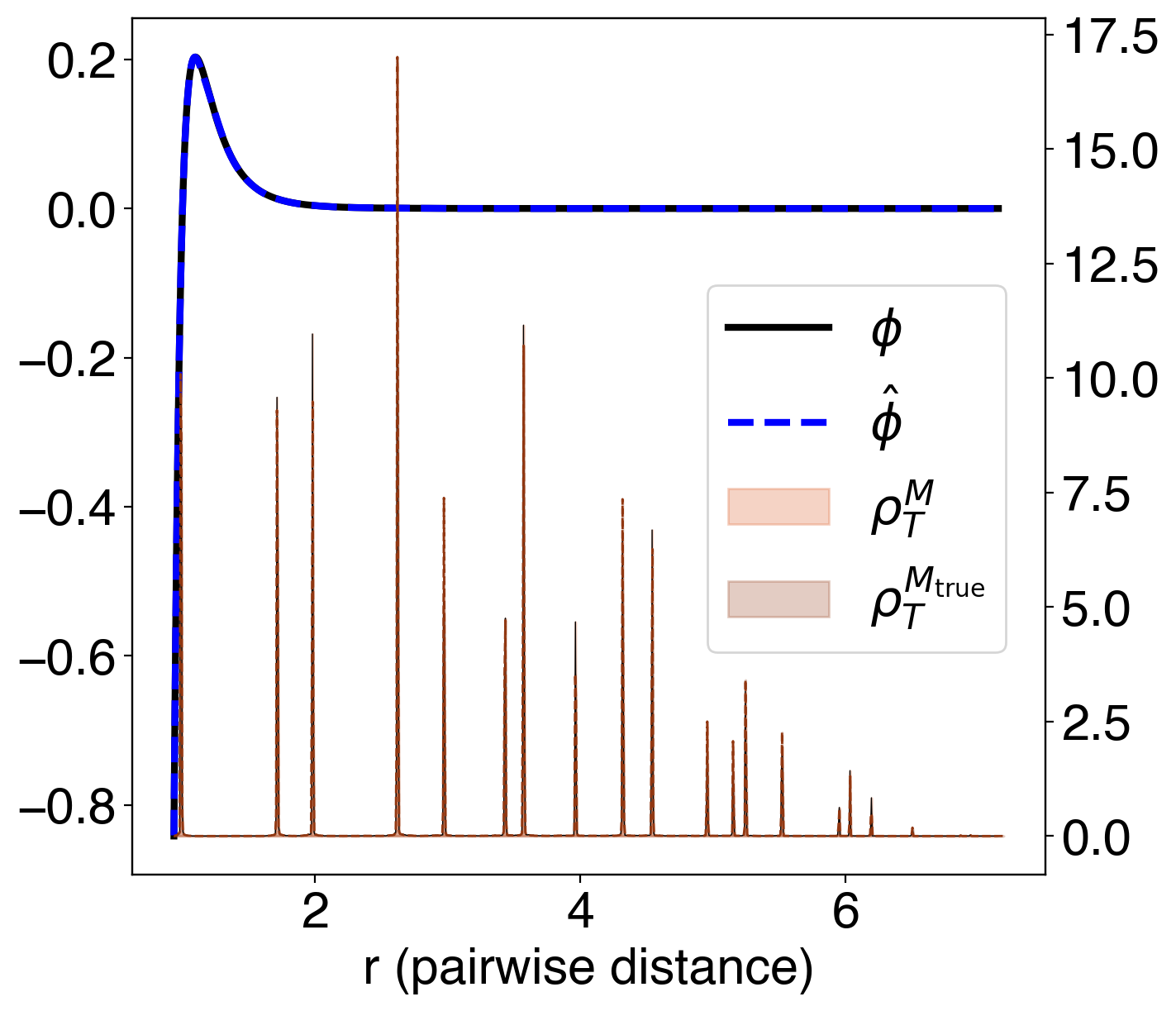}
        \caption{ $LJ$: crystal }
        \label{fig:LJ_phi_comp}
    \end{subfigure}
    \hfill
    \begin{subfigure}[t]{0.32\linewidth}
        \centering
        \includegraphics[width=\linewidth]{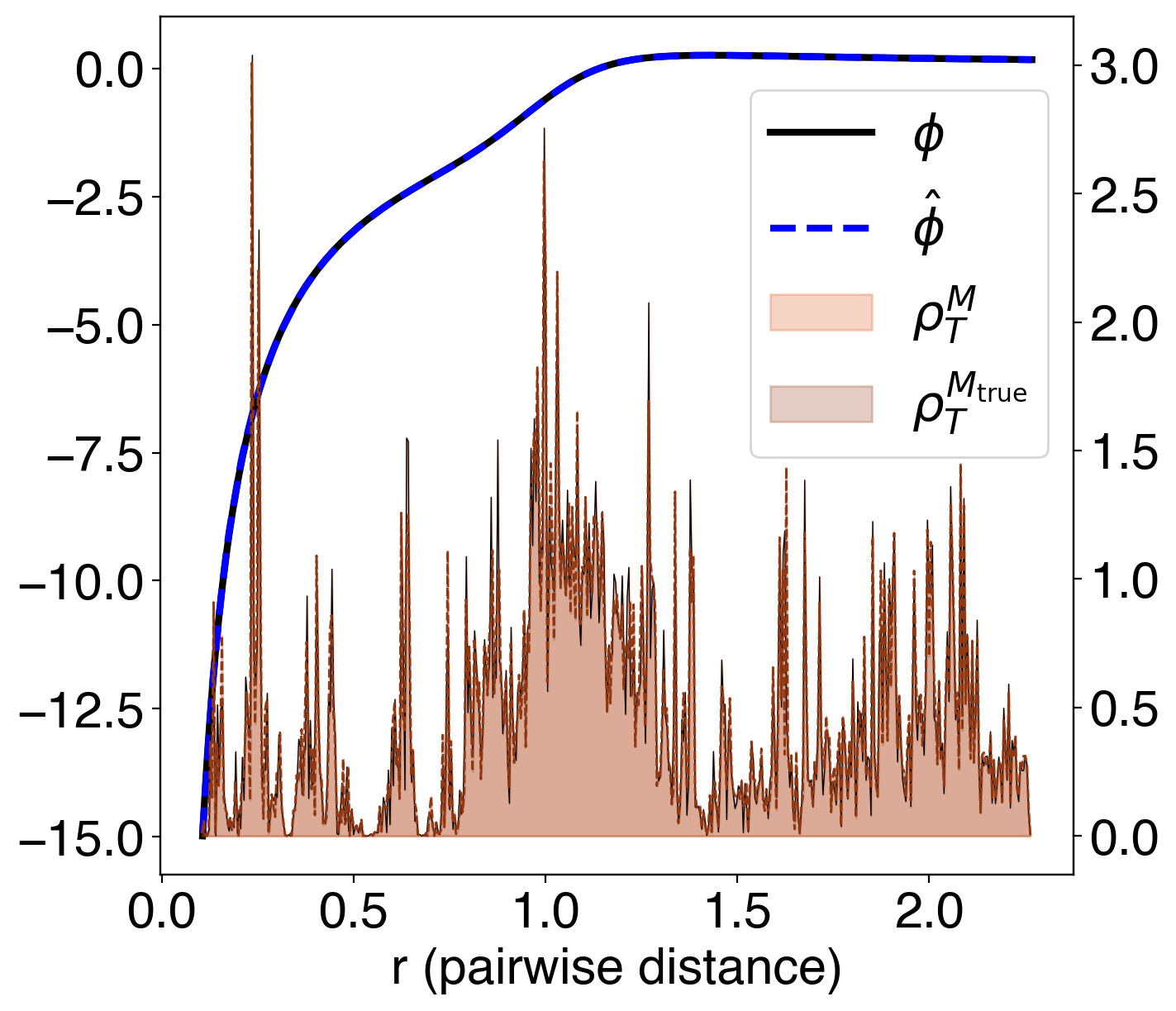}
        \caption{$\tanh$: soccer ball}
        \label{fig:tanh_phi_comp}
    \end{subfigure}
    \caption{Static patterns: $\phi$ vs.\ $\hat\phi$ (normalized); background
$\rho_T^{M}$ vs.\ $\rho_T^{M_{\text{true}}}$.}
    \label{fig:1_phi_comp}
\end{figure}
At a steady state the only information about $\phi^E$ is the equilibrium pairwise-distance distribution $\rho_T$, and $\phi$ is recoverable only on $\supp{\rho_T}$. What is learnable is thus governed by the geometry of that support. We make this the organizing axis of the three homogeneous examples, choosing kernels whose equilibria yield progressively richer $\rho_T$: the ring kernel $\phi(r)=r-1$ (Section~\ref{sec:r-1}) places all agents on one circle and supports $\rho_T$ on a \emph{finite point set}; the Lennard--Jones kernel (Section~\ref{sec:LJ}) forms a crystal whose $\rho_T$ concentrates on a few \emph{sharp clusters}; and the $\tanh$ kernel (Section~\ref{sec:tanh}) yields an equilibrium whose $\rho_T$ fills a \emph{continuous interval}. This finite~$\to$~clustered~$\to$~continuous progression runs from the most to the least ill-conditioned steady-state data. Within each example we further vary the agent number $N$, which sets how densely $\rho_T$ samples the interaction range, and compare against a B-spline interpolation of $\phi$ at the observed distances; this isolates how far learnability is capped by the coverage of $\rho_T$.
\begin{figure}[H]
    \centering
    \includegraphics[width=\linewidth]{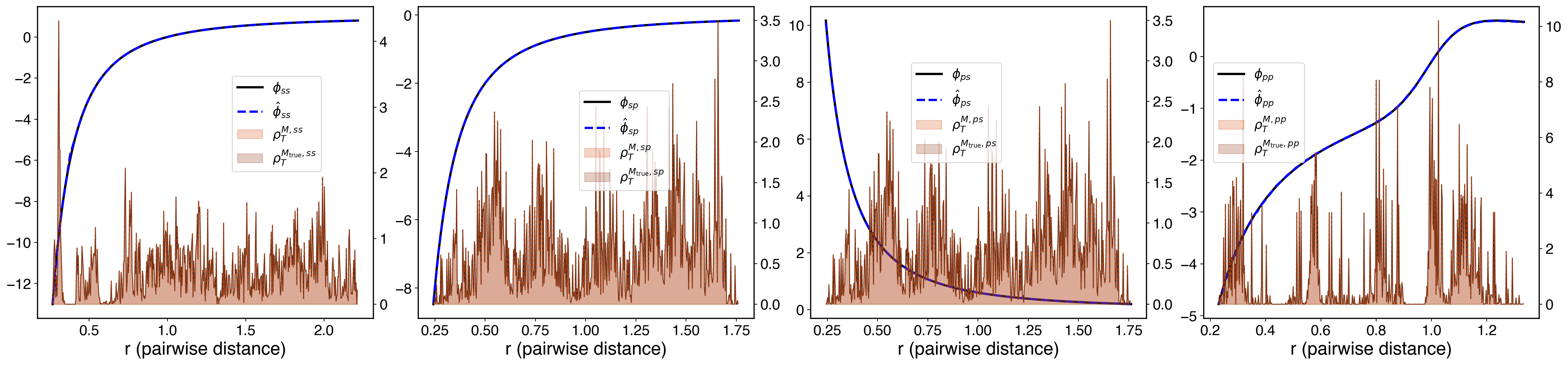}
    \caption{Predator--prey (static): $\phi_{k_1,k_2}$ vs.\ $\hat\phi_{k_1,k_2}$
(normalized) per type-pair; background $\rho^{M}_{T,k_1,k_2}$ vs.\
$\rho^{M_{\text{true}}}_{T,k_1,k_2}$. }
    \label{fig:MS-static_phi_comp}
\end{figure}
The fourth case, predator--prey (Section~\ref{sec:ms-static}), carries the same question into a heterogeneous system with $\numcl=2$ types: each ordered pair now has its own distribution $\rho_{T,k_1,k_2}$, while kernels with the same receiving type
share a common scaling constant, and four kernels must be recovered simultaneously from  shared configurations. Initial positions are drawn from $\mu^{\bx}$ uniform on $[0,1]^d$ in the homogeneous cases, and from $\bm{\mu}^{\bx}=(\mu^{\bx}_s,\mu^{\bx}_p)$ with $\mu^{\bx}_s$ uniform on $B([0,0],0.4)$ and $\mu^{\bx}_p$ uniform on $B([2.5,0],0.1)$ for predator--prey. All kernels use the $\rho_T$-adaptive B-spline basis of Section~\ref{sec:learn} (Table~\ref{table:common_param}), with $P^{\mathrm{tol}}$ set per example to match the sharpness of $\rho_T$.
\begin{table}[t]
\small
\centering
\setlength{\tabcolsep}{5pt}
\begin{tabular}{@{}lccc@{}}
\toprule
Pattern & $\theta$
& $\mathrm{Err}_{\mathrm{Traj}}$ (Mean)
& $\mathrm{Err}_{\mathrm{Traj}}$ (Std)  \\
\midrule
Ring
& $0$
& $(4.92 \pm 0.03)\times 10^{-7}$
& $(5.10 \pm 0.21)\times 10^{-8}$ \\

Crystal  
& $9.81\times 10^{-4} \pm 6.46\times 10^{-8}$
& $(1.19 \pm 1.74)\times 10^{-4}$
& $(1.64 \pm 2.67)\times 10^{-3}$ \\

Soccer ball 
& $2.54\times 10^{-2} \pm 3.29\times 10^{-5}$
& $(3.87 \pm 0.28)\times 10^{-3}$
& $(4.32 \pm 1.22)\times 10^{-3}$ \\

\addlinespace
\multirow{4}{*}{\begin{tabular}[c]{@{}c@{}}Predator--Prey\\(Static)\end{tabular}}
& $(s,s):1.06\times 10^{-3} \pm 2.82\times 10^{-7}$
& \multirow{4}{*}{$(3.42\pm 1.31)\times 10^{-3}$}
& \multirow{4}{*}{$(3.03\pm 1.71)\times 10^{-3}$} \\
& $(s,p): 1.77\times 10^{-3} \pm 3.24\times 10^{-7}$
&  &  \\
& $(p,s): 1.98 \times 10^{-2}\pm 2.43\times 10^{-4} $
&  &  \\
& $(p,p): 4.41\times 10^{-3} \pm 2.73\times 10^{-6}$
&  &  \\
\bottomrule
\end{tabular}
\caption{Static patterns from first-order systems: angle $\theta$ and trajectory error $\mathrm{Err}_{\mathrm{Traj}}$ ($M=250$, $N=40$; $N_s=28$, $N_p=12$ for predator--prey, with per-pair angles $\theta_{ss},\theta_{sp},\theta_{ps},\theta_{pp}$).}
\label{tab:traj-performance}
\end{table}
\begin{table}[t]
\centering
\setlength{\tabcolsep}{6pt}
\begin{tabular}{@{}lccc@{}}
\toprule
Pattern
& $c_*$
& $\hat c_*$ (Mean)
& $\hat c_*$  (Std)\\
\midrule
Ring
& $0.38 \pm 2.11\times 10^{-8}$
& $0.37 \pm 0$
& $0$ \\
Crystal
& $0.41\pm 0.03$
& $0.40 \pm 9.03\times 10^{-3}$
& $(1.87 \pm 0.13)\times 10^{-2}$ \\
Soccer ball
& $41.89 \pm 0.59$
& $41.37 \pm 1.36$
& $0.75\pm 0.08$\\
Prey:  $s$  & $44.85\pm 0.53$
& $43.79\pm 0.92$
& $0.79\pm 0.43$\\
Predator: $p$&$42.12\pm 0.17$
& $43.74\pm 1.24$
& $0.87\pm 0.33$\\
\bottomrule
\end{tabular}
\caption{Static patterns from first-order systems: recovered scaling $\hat c_*$ vs.\ true $c_*$ ($M=250$, $N=40$).}
\label{tab:cstar-recovery}
\end{table}
\subsubsection{Ring Pattern}\label{sec:r-1}
\begin{figure}[H]
    \centering
    \begin{subfigure}[t]{0.45\linewidth}
        \centering
        \includegraphics[width=\linewidth]{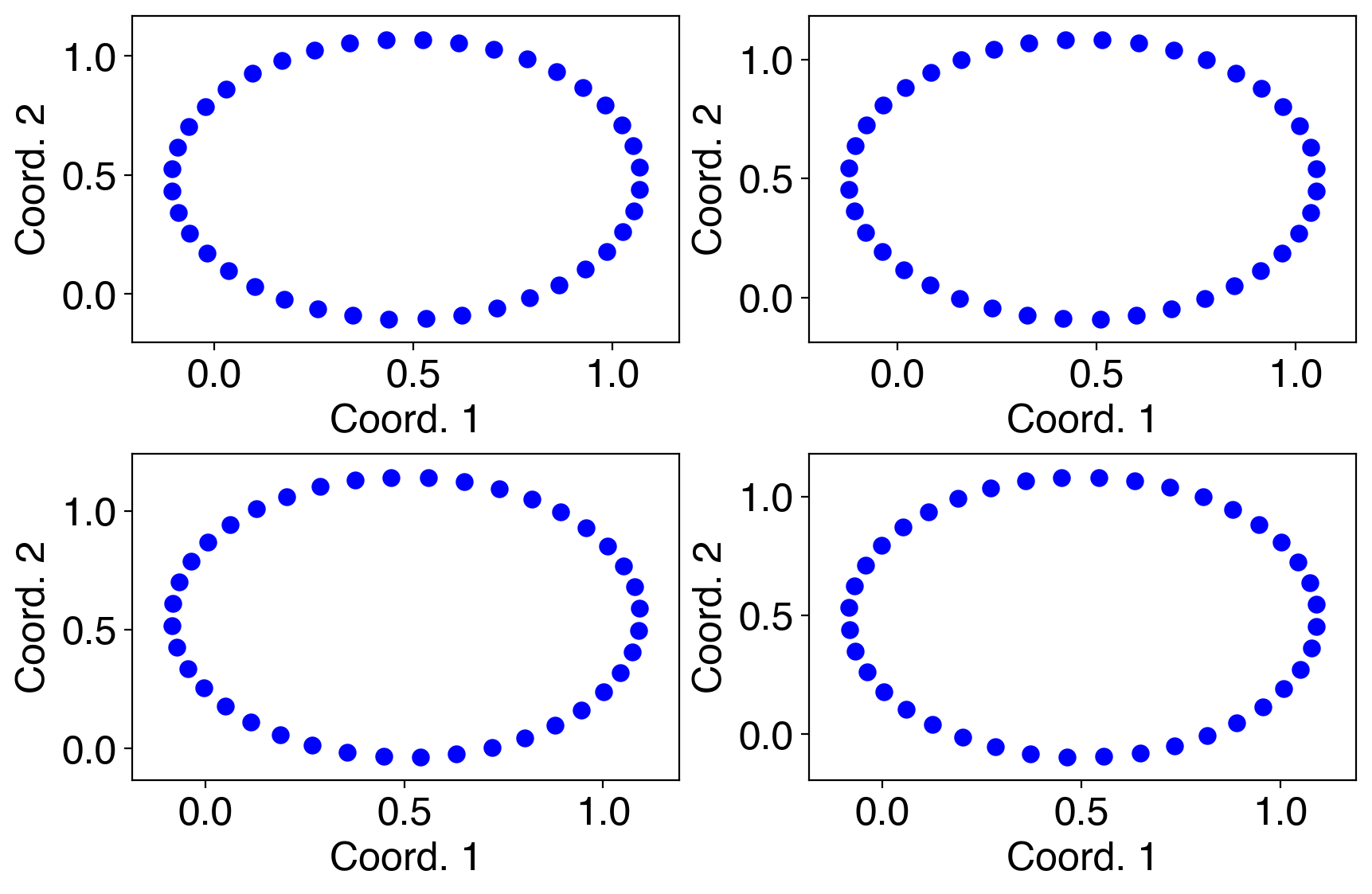}
        \caption{Steady state evolved under $\hat\phi$.}
        \label{fig:r-1_traj_comp}
    \end{subfigure}
    \hfill
    \begin{subfigure}[t]{0.45\linewidth}
        \centering
        \includegraphics[width=\linewidth]{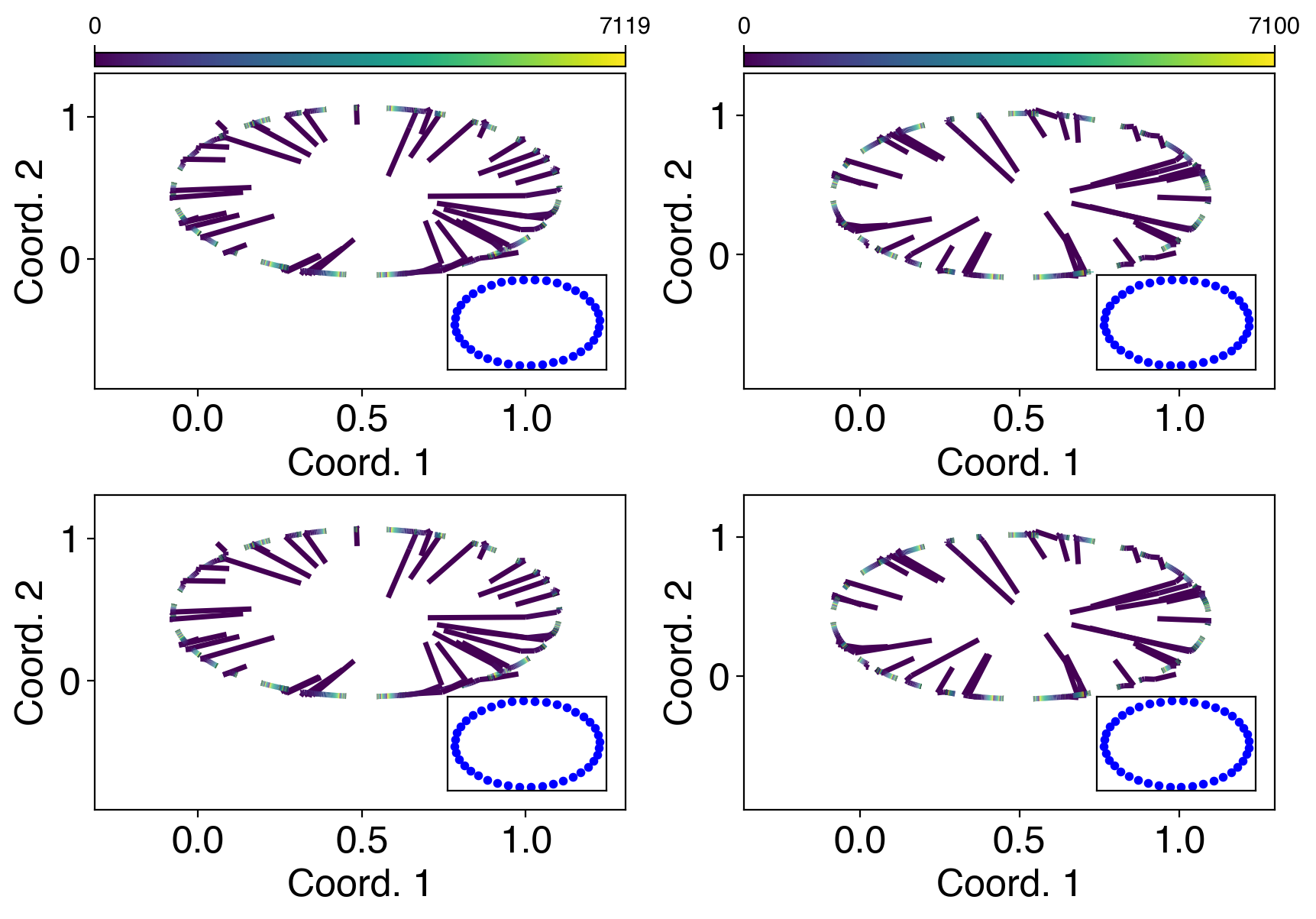}
        \caption{True (top) vs.\ learned $\hat c_*\hat\phi$ (bottom). }
        \label{fig:r-1_ic_traj_comp}
    \end{subfigure}
    \caption{Ring recovery: (left) observed steady state evolved under $\hat\phi$;
(right) trajectories from two initial conditions, true (top) vs.\ learned
$\hat c_*\hat\phi$ (bottom), final configurations inset.}
    \label{fig:r-1_combined}
\end{figure}
As the finite-support extreme of the progression, we take $\phi(r)=r-1$ from~\cite{pattern2011}, whose equilibrium spaces the agents evenly on a unit circle~\cite{pattern2011,pattern2012}. The resulting $\rho_T$ is supported on only $\lfloor \tfrac{N}{2} \rfloor$ distinct distances, the sparsest data of any example considered here. We use a degree-$1$ B-spline on a coarse $\rho_T$-adaptive partition ($P^{\mathrm{tol}}=0.5$).
\begin{figure}[H]
    \centering
    \begin{subfigure}[t]{0.46\linewidth}
        \centering
        \includegraphics[width=\linewidth]{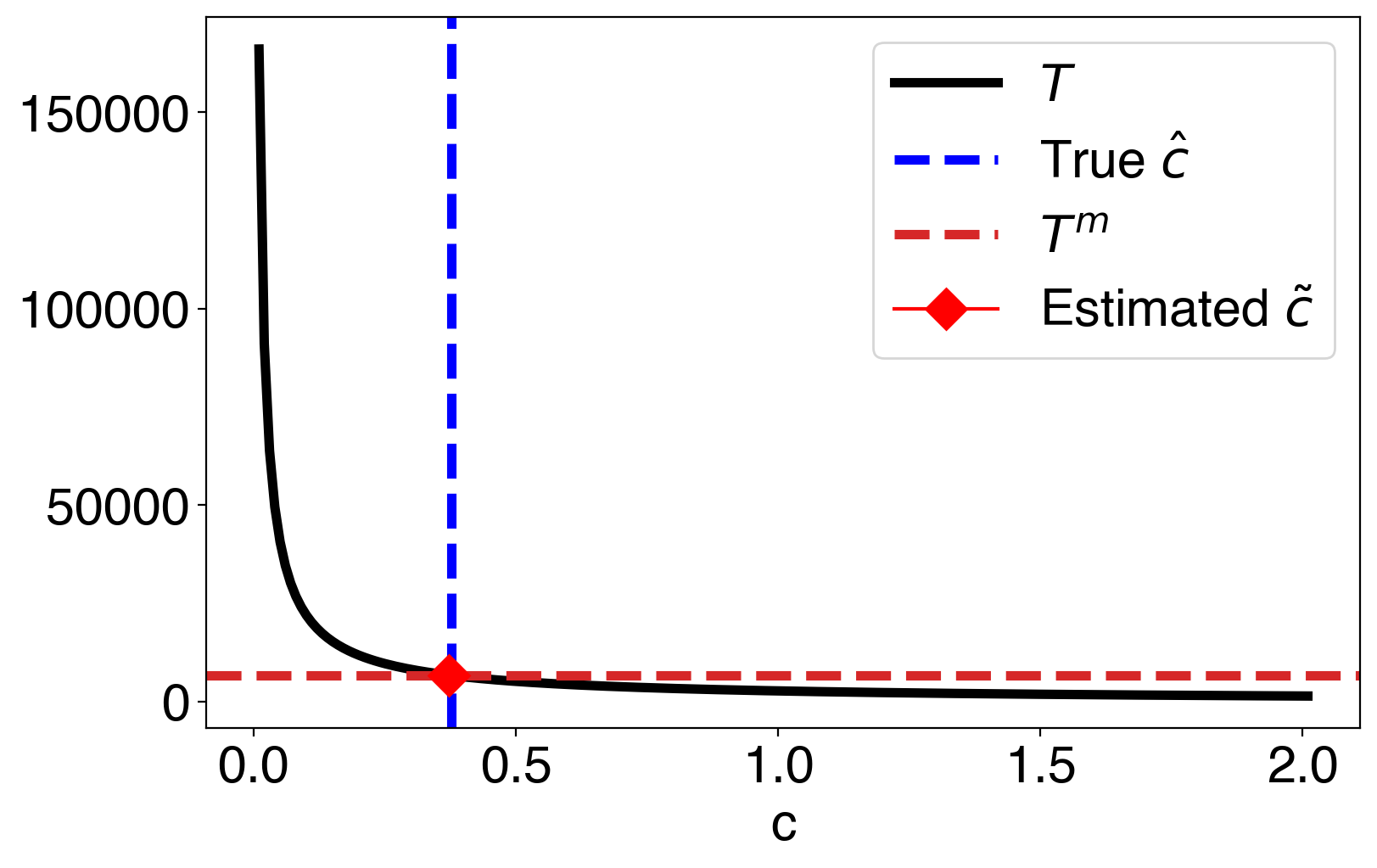}
        \caption{Estimated optimal $C_*$.}
        \label{fig:r-1_C_comp}
    \end{subfigure}
    \hfill 
    \begin{subfigure}[t]{0.46\linewidth}
        \centering
        \includegraphics[width=\linewidth]{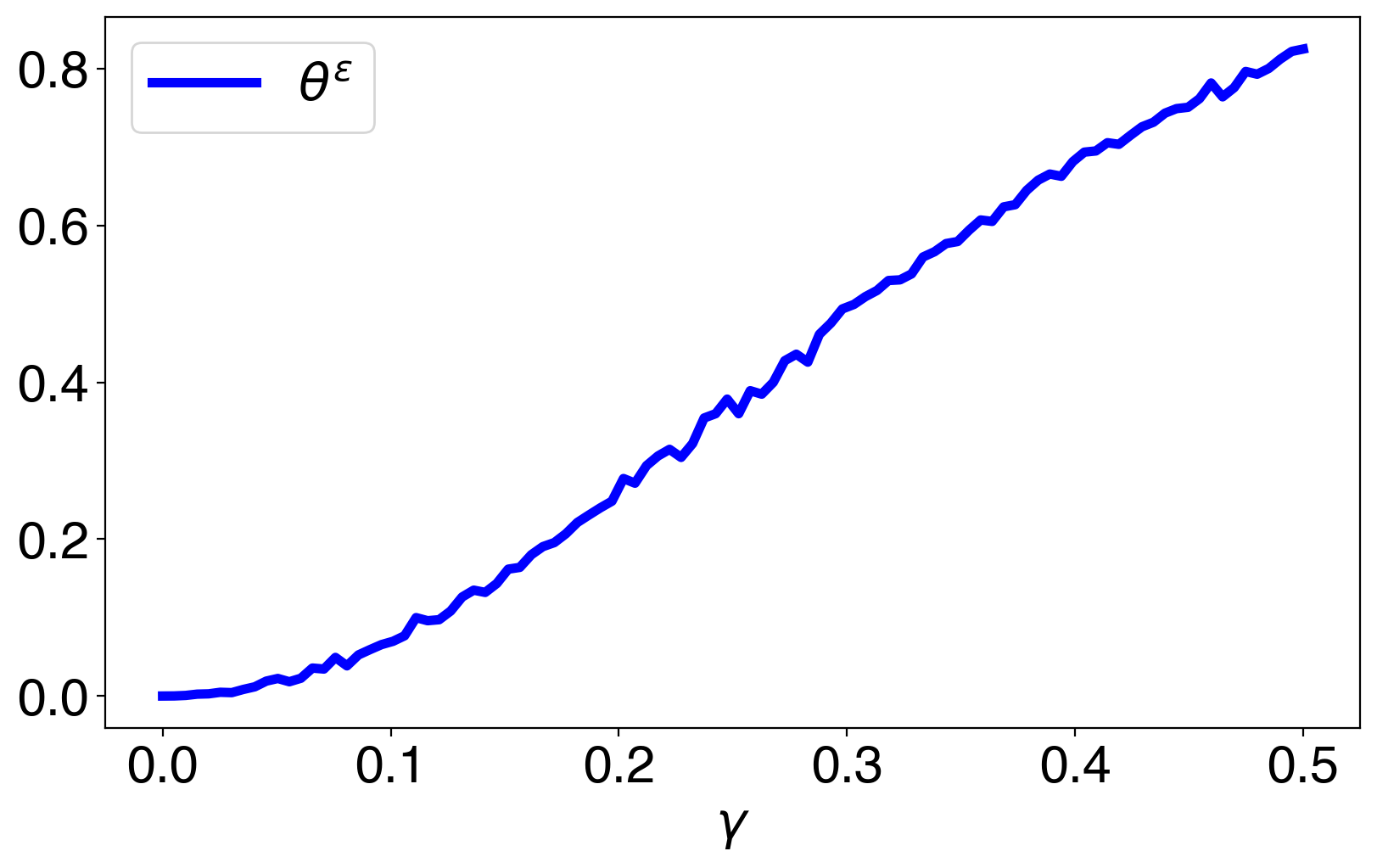}
        \caption{Noise effects in learning.}
        \label{fig:r-1_noise}
    \end{subfigure}
    \caption{Additional learning from ring steady patterns.}
\end{figure}
The two panels of Figure~\ref{fig:r-1_combined} verify the recovery at two levels. The left panel evolves the observed steady state under $\hat\phi$ for a further $T\gg1$: the configuration does not drift, confirming $\F_{\hat\phi}(\bX_T)\approx\zero$, i.e.\ $\hat\phi$ admits the observed pattern as a fixed point. The right panel makes the stronger check, evolving $\hat c_*\hat\phi$ from two \emph{fresh} initial conditions; the learned trajectories track the true ones from start to convergence, so the steady-state snapshots suffice to recover the full dynamics, not merely one equilibrium. We use this left/right reading in the examples that follow.
\begin{figure}[H]
    \centering
    \includegraphics[width=\linewidth]{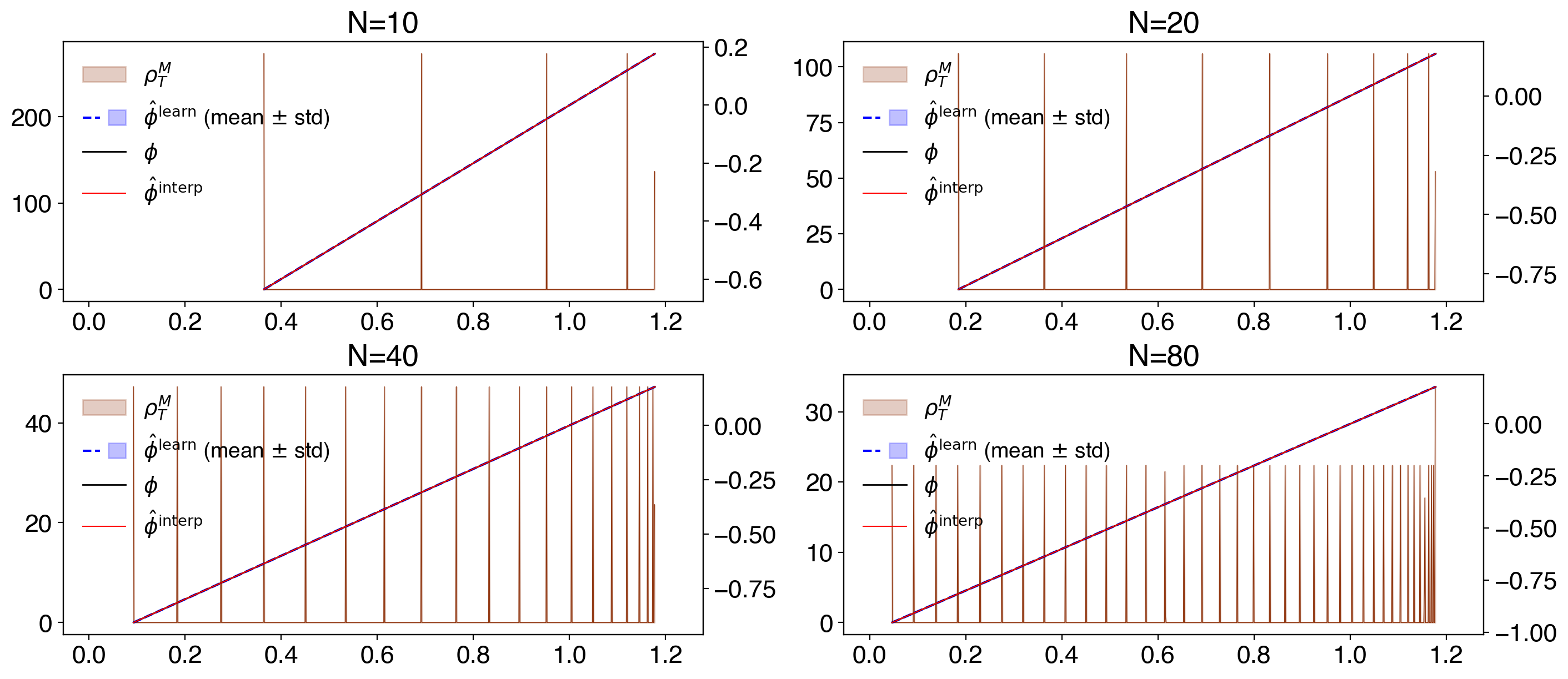}
    \caption{Ring: learning of $\phi$ with $N \in \{10, 20, 40, 80\}$, compared 
   to B-spline interpolation $\hat\phi^{\mathrm{interp}}$. Background shows 
   $\rho_T^{M}$.}
    \label{fig:r-1_rhoT}
\end{figure}
On the kernel level (panel $(a)$ of Figure~\ref{fig:1_phi_comp}), $\hat\phi$ matches $\phi$ exactly on $\supp{\rho_T}$: the ring row of Table~\ref{tab:traj-performance} reports $\theta=0$ to machine precision. This is sharper than in the later examples and is intrinsic to the finite support. A degree-$1$ basis on the adaptive partition reproduces $\phi$ exactly at the few realized distances, so the observable component carries no approximation error and all residual ambiguity lies in the scaling null space, recovered as $\hat c_*$ in Table~\ref{tab:cstar-recovery}.

Because $\rho_T$ collapses onto isolated points, the ring is also the most stringent test against an interpolation baseline $\hat\phi^{\mathrm{interp}}$
built from the true values of $\phi$ at the observed distances (Figure~\ref{fig:r-1_rhoT}). As $N$ decreases and the points grow sparser, $\hat\phi^{\mathrm{interp}}$ becomes unstable between them, whereas the $\rho_T$-weighted variational estimate stays close to $\phi$ across all tested $N$. Even at this degenerate end of the progression, learnability is limited only by the coverage of $\rho_T$ and not by the conditioning of the recovery, a baseline that the richer-support cases below build on.
\subsubsection{Crystal Formation}\label{sec:LJ}
The Lennard--Jones (LJ) potential~\cite{FISCHER2023113876, https://doi.org/10.1002/andp.202400115} is a classical model for intermolecular interactions, combining short-range repulsion with long-range attraction, and induces crystalline structures at equilibrium~\cite{FISCHER2023113876}. It sits at the middle rung of the progression: the crystal equilibrium concentrates $\rho_T$ on a few narrow clusters, richer than the ring's point set but still far from continuous. We take the first-order system~\eqref{eq:first_order} with
\[
\phi(r) =
\begin{cases}
\varepsilon\left(\dfrac{A}{r^{8}} - \dfrac{B}{r^{14}}\right), & r \ge r_c, \\[6pt]
\varepsilon\left(\dfrac{A}{r_c^{8}} - \dfrac{B}{r_c^{14}}\right), & r < r_c,
\end{cases}
\qquad \varepsilon = 0.1,\; A=B = 10,\; r_c = 0.5,
\]
where $\phi(r)=U'(r)/r$ is the gradient-flow form of the $12$--$6$ LJ potential~\cite{jones1924determination}, so the $r^{-14}$ term gives short-range repulsion and the $r^{-8}$ term long-range attraction, and $r_c$ regularizes the singularity of $\phi$ near the origin. The clusters are sharp, so we use a finer partition ($P^{\mathrm{tol}}=0.01$) than for the ring.
\begin{figure}[H]
    \centering
    \begin{subfigure}[t]{0.46\linewidth}
        \centering
        \includegraphics[width=\linewidth]{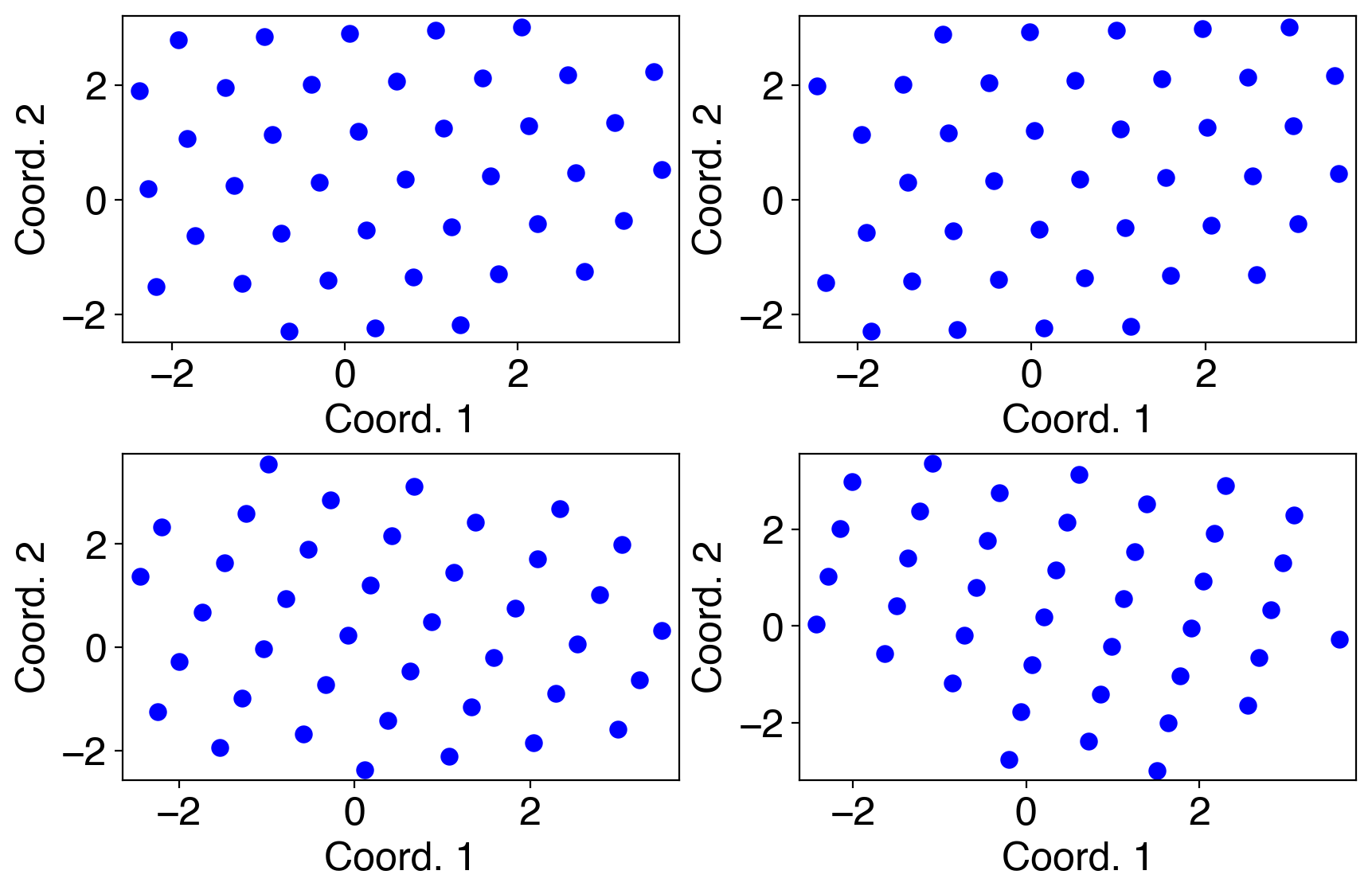}
        \caption{Evolution of the crystal
         steady-state configuration under $\hat\phi$.}
        \label{fig:LJ_traj_comp}
    \end{subfigure}
    \hfill
    \begin{subfigure}[t]{0.46\linewidth}
        \centering
        \includegraphics[width=\linewidth]{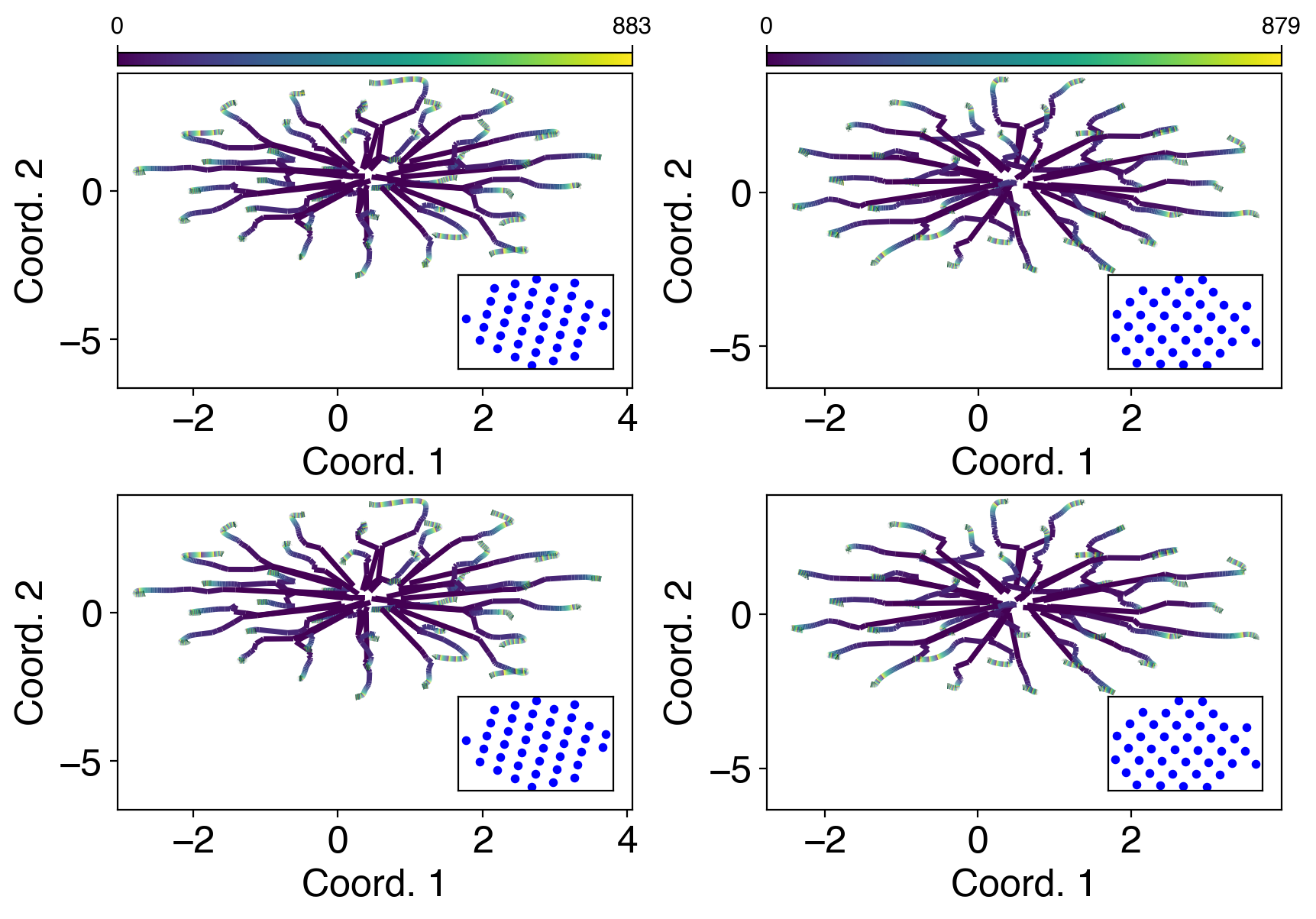}
        \caption{Two trajectories: true dynamics (top) vs. learned dynamics with $\hat{c}_*\hat{\phi}$ (bottom). }
        \label{fig:LJ_ic_traj_comp}
    \end{subfigure}
    \caption{Crystal recovery: steady-state evolution (left) and full trajectories (right, with final configurations in the insets) under true and learned dynamics.}
    \label{fig:LJ_combined}
\end{figure}
Similar to the results shown in the case of $\phi(r) = r - 1$, the two panels of Figure~\ref{fig:LJ_combined} confirm that $\hat\phi$ admits the observed crystal as a fixed point and that $\hat c_*\hat\phi$ reproduces the full trajectories from fresh initial conditions. The new feature here is that $\rho_T$ leaves gaps between clusters, so the observable range no longer covers the whole interaction range and the recovery is no longer exact: the Crystal row of Table~\ref{tab:traj-performance} reports $\theta=9.81\times10^{-4}$ rather than zero. The hardest region is the repulsive core, which the equilibrium barely samples because agents never approach that closely; what is unlearnable is set by the geometry of $\supp{\rho_T}$, not by the shape of $\phi$. The recovered $\hat c_*$ (Table~\ref{tab:cstar-recovery}) also carries larger variance than for the ring, as the multi-modal $\rho_T$ makes the first-stopping-time criterion more sensitive.

Varying $N$ confirms this finding (Appendix, Figure~\ref{fig:LJ_rhoT}). At small $N$ the clusters cover only part of the interaction range and both $\hat\phi$ and the interpolation baseline $\hat\phi^{\mathrm{interp}}$ miss the repulsive region; as $N$ grows the clusters spread, the gaps close, and both reconstructions converge to $\phi$ on $\supp{\rho_T}$, with $\hat\phi$ showing negligible variance across trials. The clustered regime thus recovers $\phi$ wherever $\rho_T$ has mass but cannot reach across the inter-cluster gaps, a limitation the continuous support of the next example removes.
\subsubsection{Soccer Ball Pattern}\label{sec:tanh}
\begin{figure}[H]
    \centering
    \begin{subfigure}[t]{0.46\linewidth}
        \centering
        \includegraphics[width=\linewidth]{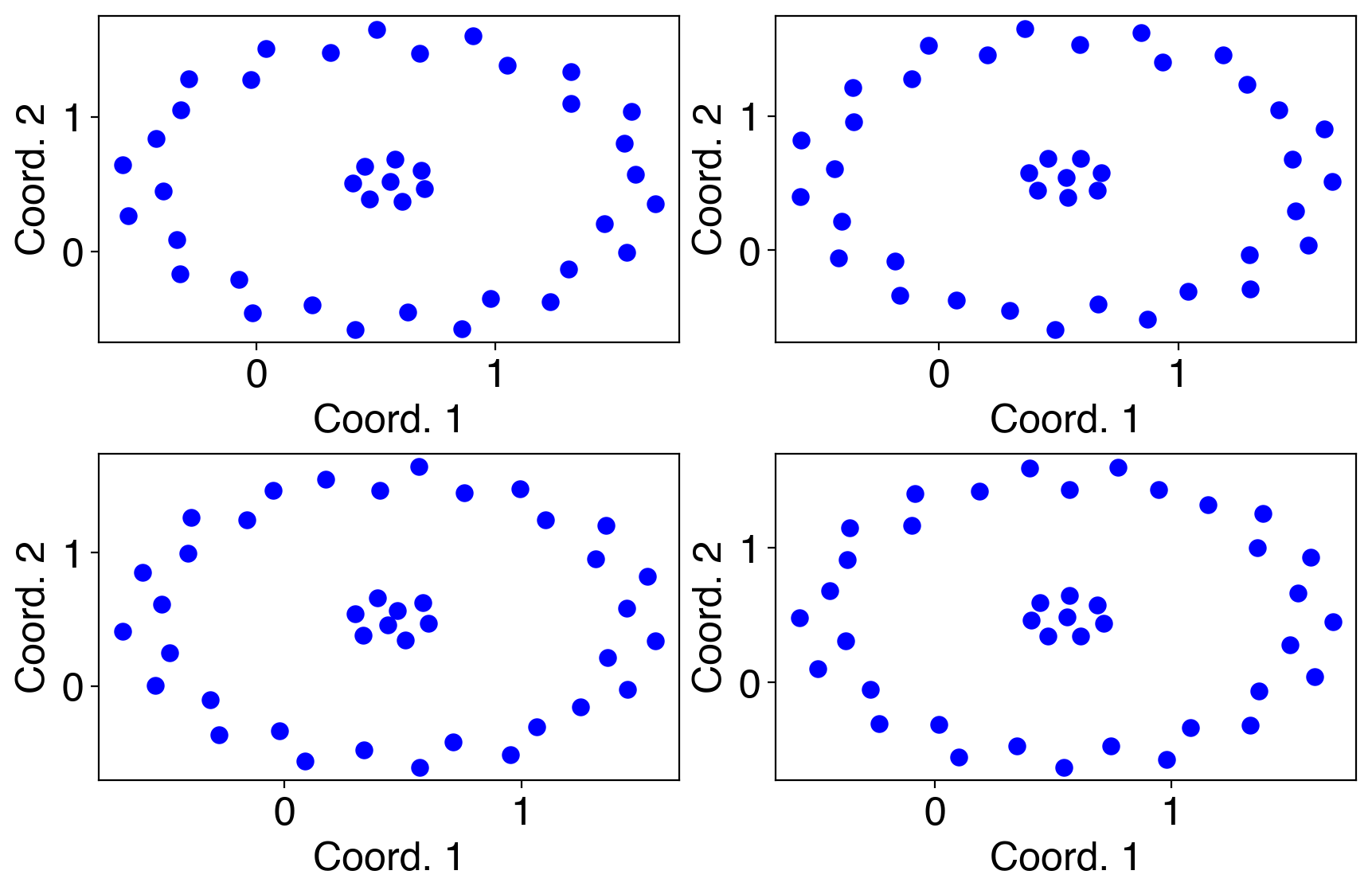}
        \caption{Evolution of the soccer ball
         steady-state configuration under $\hat\phi$.}
        \label{fig:tanh_traj_comp}
    \end{subfigure}
    \hfill
    \begin{subfigure}[t]{0.46\linewidth}
        \centering
        \includegraphics[width=\linewidth]{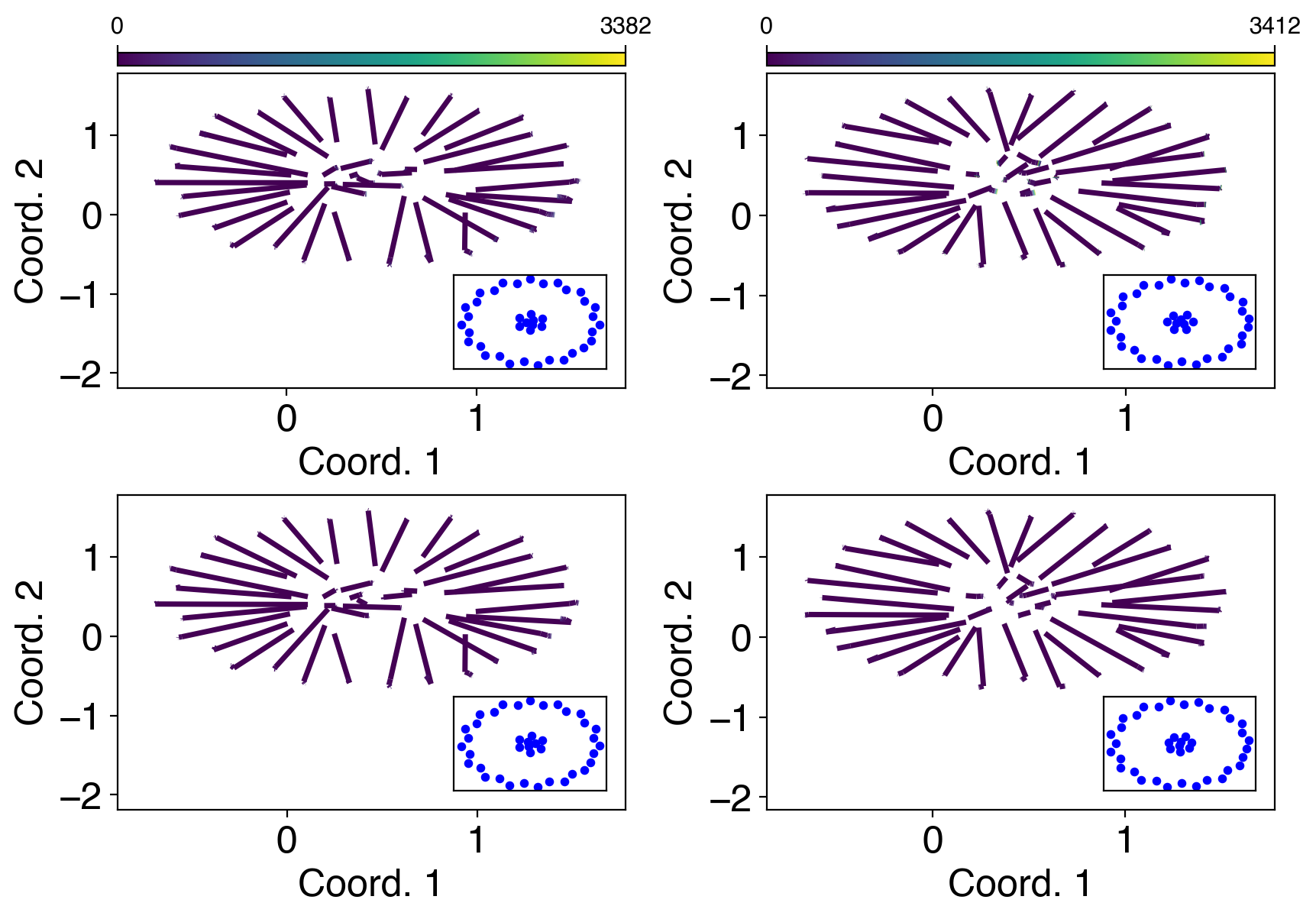}
        \caption{Two trajectories: true dynamics (top) vs. learned dynamics with $\hat{c}_*\hat{\phi}$ (bottom). }
        \label{fig:tanh_ic_traj_comp}
    \end{subfigure}
    \caption{Soccer ball recovery: steady-state evolution (left) and full trajectories (right, with final configurations in the insets) under true and learned dynamics.}
    \label{fig:tanh_combined}
\end{figure}
The tanh-type force~\cite{pattern2011} switches smoothly but sharply from short-range repulsion to long-range attraction, producing a rich family of two-dimensional equilibria, from rings and annuli to soccer-ball structures with $N$-fold symmetry~\cite{pattern2011, von2012soccer}. We take the first-order system~\eqref{eq:first_order} with
\[
\phi(r) =
\begin{cases}
-\dfrac{\tanh\!\big(a(1 - r)\big) + b}{r}, & r \ge r_c, \\[6pt]
-\dfrac{\tanh\!\big(a(1 - r_c)\big) + b}{r_c}, & r < r_c,
\end{cases}
\qquad a = 5,\ b = 0.6,\ r_c = 0.05,
\]
where $r_c$ regularizes the $1/r$ factor near the origin. This is the continuous end of the progression: unlike the ring and LJ equilibria, the resulting $\rho_T$ fills a full interval of distances rather than collapsing onto points or clusters. We use the finest partition of the three ($P^{\mathrm{tol}}=0.005$) to resolve its fine-scale structure.

As in the previous cases, the two panels of Figure~\ref{fig:tanh_combined} verify that $\hat\phi$ admits the observed pattern as a fixed point and that $\hat c_*\hat\phi$ reproduces the trajectories from fresh initial conditions (Tables~\ref{tab:traj-performance} and~\ref{tab:cstar-recovery}, Soccer-ball row). Because $\rho_T$ now covers the interaction range without gaps, $\hat\phi$ tracks $\phi$ throughout $\supp{\rho_T}$ with no unlearnable region of the kind seen in LJ. The same continuity makes the recovery insensitive to $N$ (Appendix, Figure~\ref{fig:tanh_rhoT}): even at small $N$ the coverage is already adequate, and both $\hat\phi$ and the interpolation baseline $\hat\phi^{\mathrm{interp}}$ are accurate across all tested $N$.

These three homogeneous examples sweep $\rho_T$ from a finite point set (ring) through sharp clusters (LJ) to a continuous interval (tanh), that is, from the most to the least ill-conditioned steady-state data. Across the whole range the $\rho_T$-adaptive recovery stays accurate on $\supp{\rho_T}$, and the only intrinsic limit is the coverage of $\rho_T$: it shrinks the unlearnable region as the support grows and the gap between small-$N$ and large-$N$ performance closes. We turn next to whether this picture survives in a heterogeneous system.
\subsubsection{Predator--Prey Model: Static Patterns}\label{sec:ms-static}
\begin{figure}[H]
    \centering
        \begin{subfigure}[t]{0.46\linewidth}
        \centering
        \includegraphics[width=\linewidth]{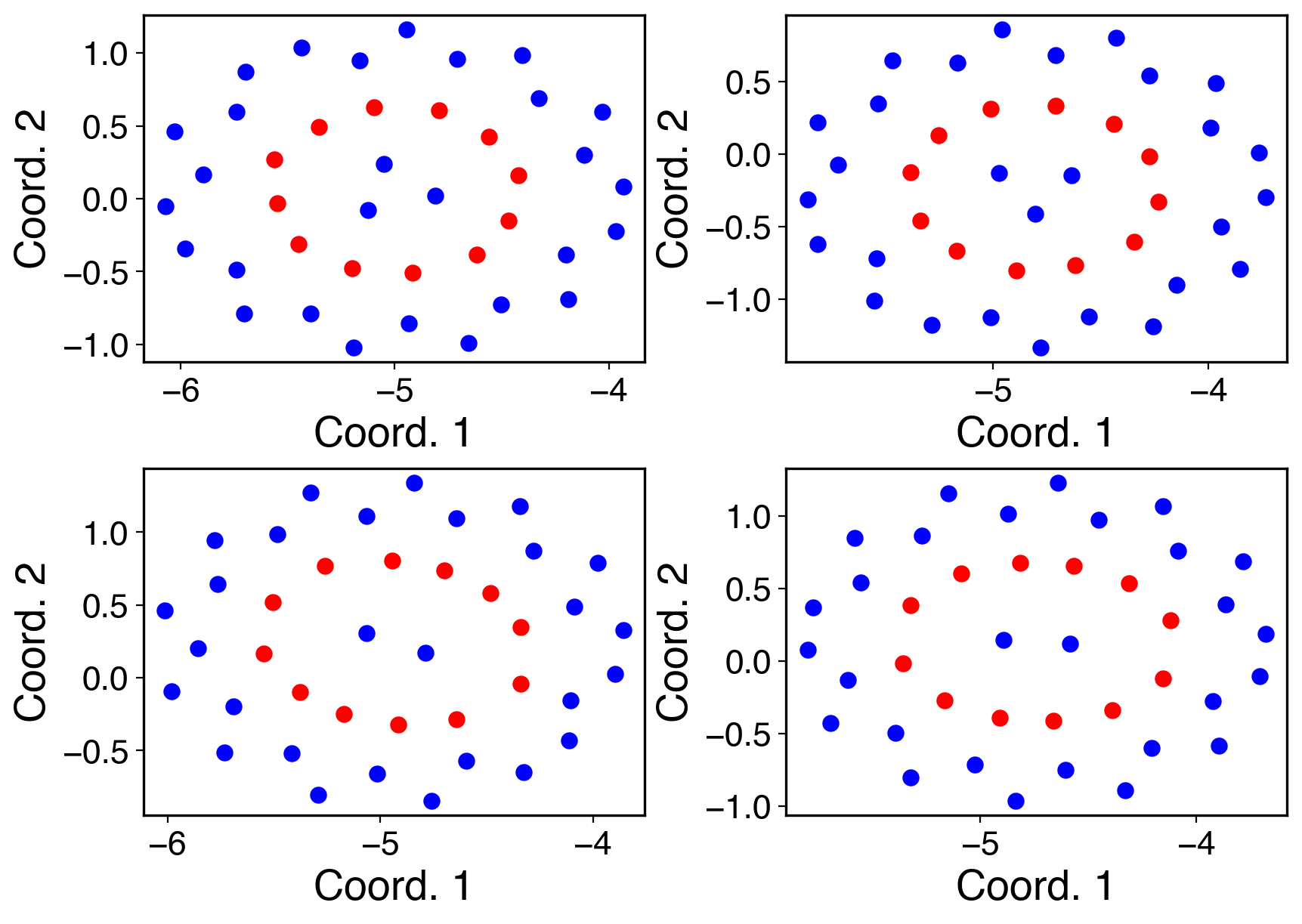}
        \caption{Steady state evolved under $\hat\phi$. }
         \label{fig:MS-static_traj_comp}
    \end{subfigure}
    \hfill
        \begin{subfigure}[t]{0.46\linewidth}
        \centering
        \includegraphics[width=\linewidth]{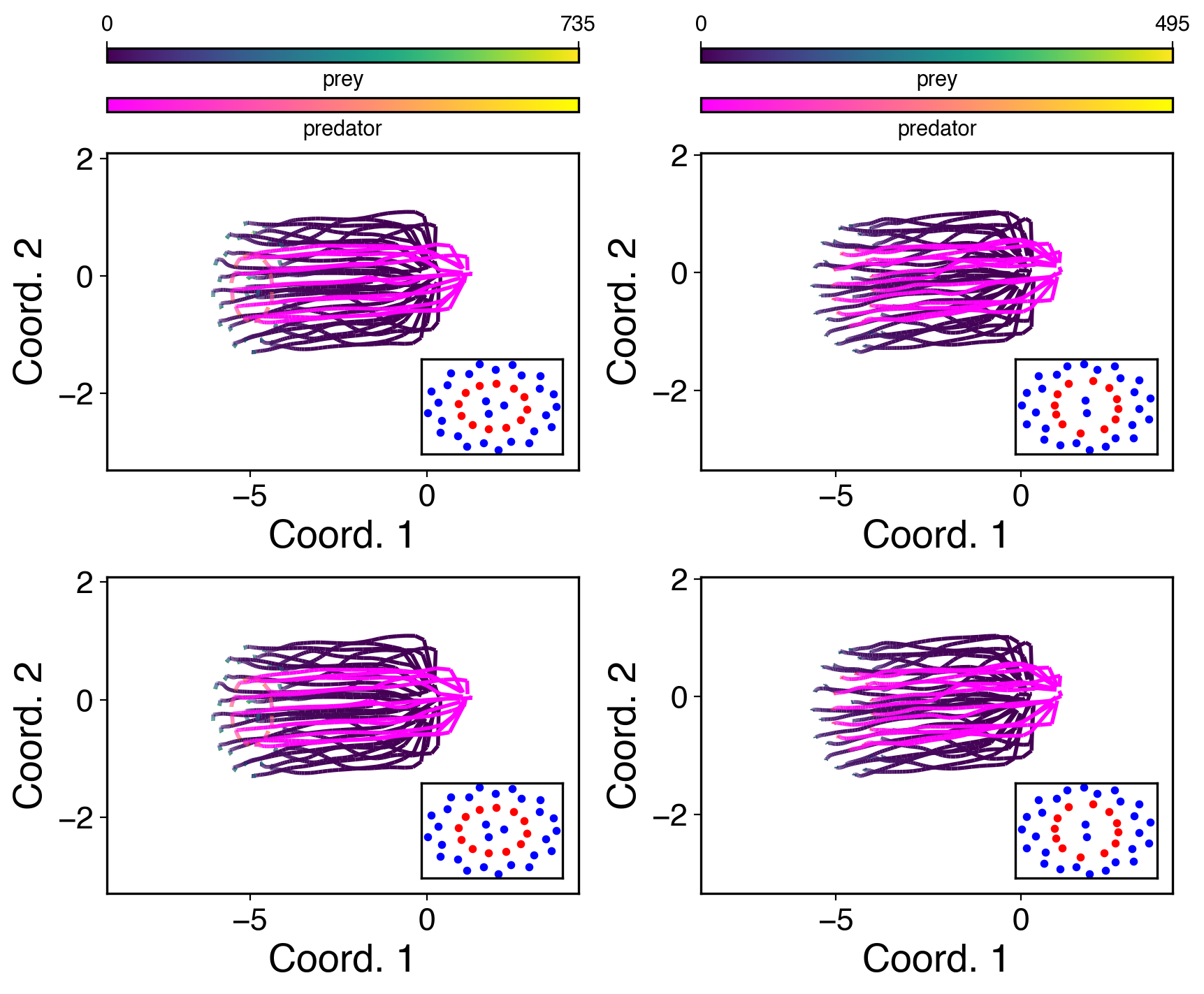}
        \caption{True (top) vs.\ learned (bottom). 
    }
         \label{fig:ms-static_ic_traj_comp}
    \end{subfigure}
   \caption{Predator--prey (static): (left) observed steady state evolved under
$\hat\phi$; (right) trajectories from two initial conditions, true (top) vs.\
learned (bottom), final configurations inset. Prey in blue, predators in red.}
   \label{fig:MS-static_combined}
\end{figure}
To test whether the same picture persists when several kernels must be recovered at once, we move to a heterogeneous system of two species, prey~($s$) and predator~($p$), with $\numcl=2$, $N_s=28$, $N_p=12$. The first-order system~\eqref{eq:first_order} now carries four kernels $\phi_{k_1,k_2}$, $k_1,k_2\in\{s,p\}$,
\begin{align*}
\phi_{ss}(r) &= a_{ss} - \frac{1}{r^2},
& \phi_{sp}(r) &= -\frac{b_{sp}}{r^2}, \\
\phi_{ps}(r) &= \frac{c_{ps}}{r^2}, 
& \phi_{pp}(r) &= -\frac{\tanh\!\big(d_{pp}(1 - r)\big) + e_{pp}}{r},
\end{align*}
with $a_{ss}=1$, $b_{sp}=0.5$,  $c_{ps}=0.6$,  $d_{pp}=8$, $e_{pp}=0.1$, and a common cutoff $r_c=0.01$. These combine prey self-organization through repulsion, an asymmetric coupling in which prey are repelled by predators while predators are attracted to prey, and tanh-type predator self-interaction. The equilibrium is a static configuration with the prey forming a structured shell around a tighter predator cluster (Figure~\ref{fig:MS-static_combined}). Each kernel is learned on its own $\rho_{T,k_1,k_2}$-adaptive basis with $P^{\mathrm{tol}}=0.01$.

\begin{figure}[H]
    \centering
    \begin{subfigure}[t]{0.46\textwidth}
        \centering
        \includegraphics[width=\linewidth]{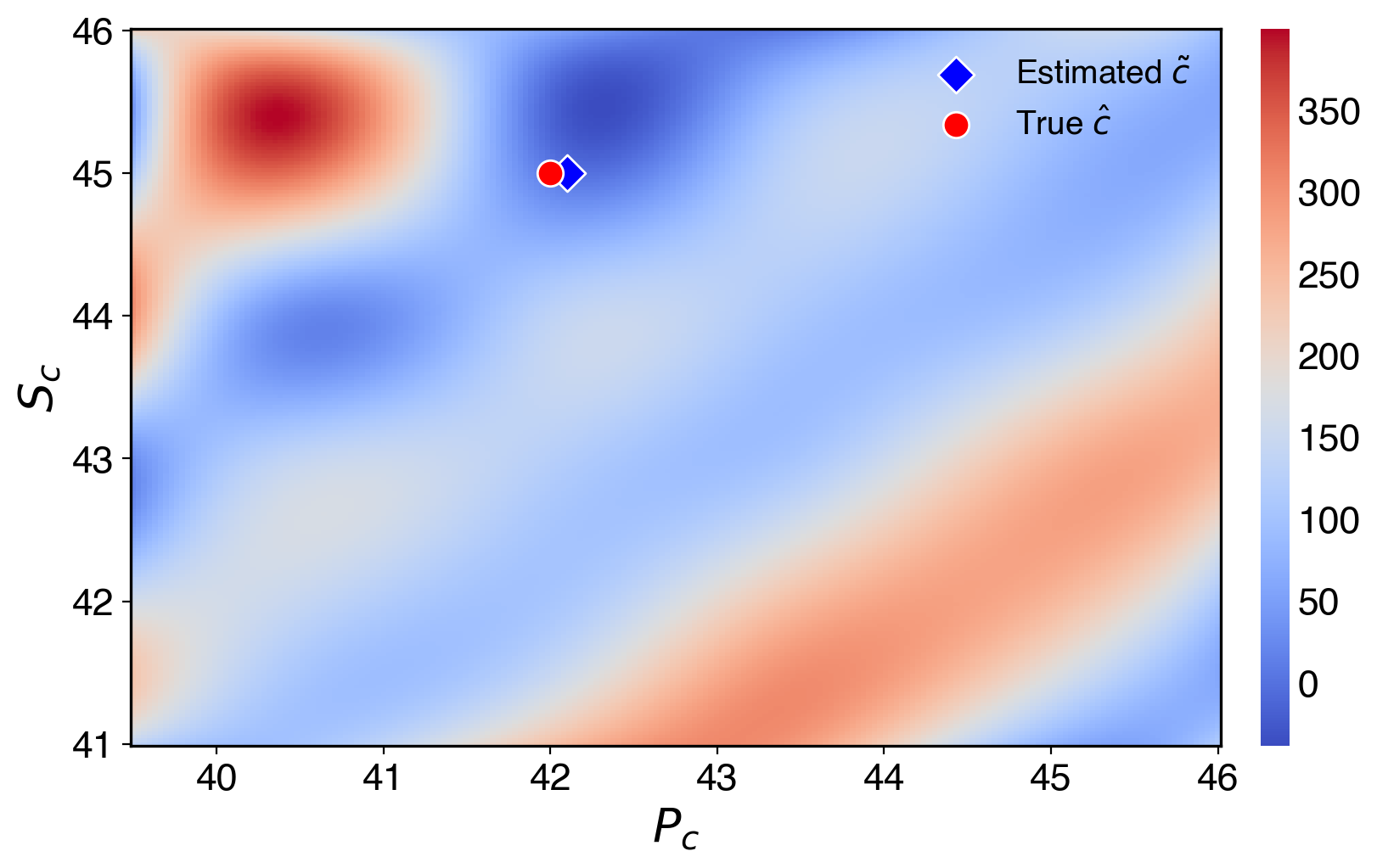}
        \caption{Estimated optimal $C_*$.}
        \label{fig:MS-static_C_comp}
    \end{subfigure}%
    \hfill
    \begin{subfigure}[t]{0.46\textwidth}
        \centering
        \includegraphics[width=\linewidth]{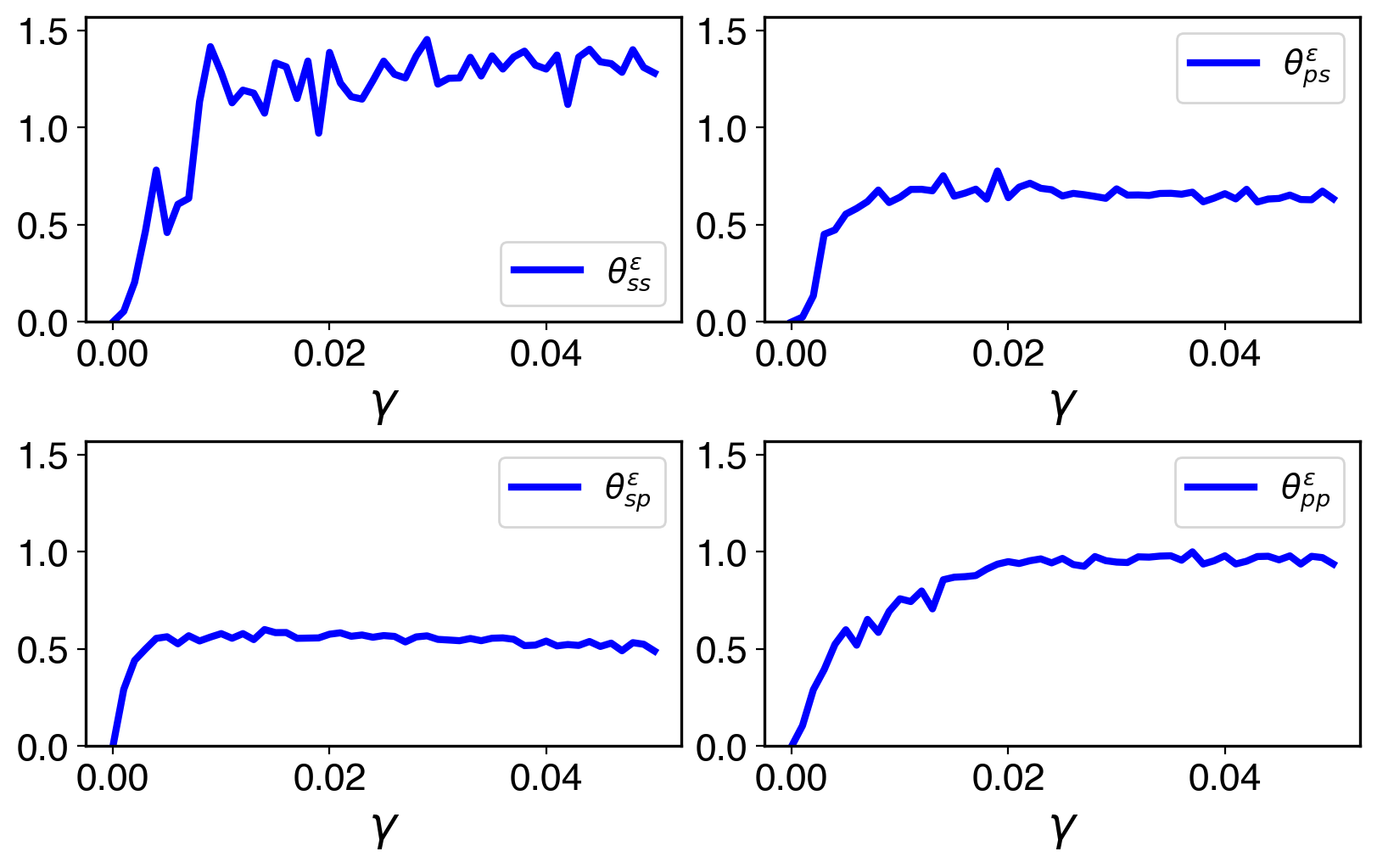}
        \caption{Noise effects in learning.}
        \label{fig:MS-static_noise}
    \end{subfigure}
    \caption{Additional learning from heterogeneous steady patterns.}
\end{figure}
As in the single-species cases, the two panels of Figure~\ref{fig:MS-static_combined} verify that the four learned kernels admit the observed pattern as a fixed point and that the rescaled kernels reproduce the trajectories from fresh initial conditions. The new ingredient is the scaling ambiguity: each species now carries its own constant, so the recovery returns a pair $(c_{s,*},c_{p,*})$, obtained jointly by the first-stopping-time procedure (Tables~\ref{tab:traj-performance} and~\ref{tab:cstar-recovery}).

The coverage principle of the homogeneous examples carries over directly, now across four distributions of differing geometry. The pair $\rho_{T,pp}$ is the narrowest, since the few predators realize few pairwise distances. The largest angles in Table~\ref{tab:traj-performance} are $\theta_{ps}$ and $\theta_{pp}$, the two kernels acting on the predators: with only $N_p=12$ predators, the equations constraining them are the fewest, so these kernels are the least accurately recovered, whereas the kernels acting on the prey, $\theta_{ss}$ and $\theta_{sp}$, are near-exact.  Since $\phi_{sp}$ and $\phi_{ps}$ share the same distance distribution, $\rho_{T,sp}=\rho_{T,ps}$, their gap reflects this difference in the number of constraints rather than in coverage. Varying $N$ at fixed species ratio, $(N,N_s,N_p)\in\{(10,7,3),(20,14,6),(40,28,12)\}$, confirms this (Appendix, Figure~\ref{fig:MS-static_rhoT}): the predator-related kernels are the last to converge, and all four $\hat\phi_{k_1,k_2}$ recover $\phi_{k_1,k_2}$ on their supports once $N$ is large enough. The heterogeneous setting thus inherits the single-species behavior kernel by kernel, with learnability set by the coverage of each $\rho_{T,k_1,k_2}$.
\subsection{Semi-static Patterns}
We next consider semi-static configurations, in which the system does not come to rest but approaches a coherent flocking state. In the limiting regime, the agents move with a common velocity $\bv_*$ while maintaining an approximately fixed relative configuration. This includes moving flocks in first-order multi-species systems and flocking states in second-order systems.

The key point in this subsection is that these two flocking mechanisms fall into different identifiability regimes. For first-order moving flocks, $\dot{\bx}_i \to \bv_* \neq \zero$, so the right-hand side of the learning equation is nonzero; accordingly, kernel accuracy is measured in the weighted norm $\norm{\cdot}_{\rho_T}$ without any scaling correction. By contrast, for second-order flocking states, $\dot{\bv}_i\to \zero$ and $\bv_i\to\bv_*$, so the right-hand side of the learning equation vanishes in the limiting flocking state. These examples therefore fall into the scale-ambiguous zero-right-hand-side
regime, where we report angular errors and recover the corresponding
scaling constant.

Initial positions are sampled as in the static examples. For second-order systems, the initial velocities are additionally drawn from $\mu^{\bv}$, uniform on $[0,1]^d$. First-order flocking examples are monitored using the relative-velocity criterion~\eqref{eq:c1-flock}, while second-order flocking examples use the acceleration criterion~\eqref{eq:c2}, with $\mathrm{Tol}=10^{-9}$. The interaction kernels are approximated using the adaptive B-spline bases specified in Table~\ref{table:common_param}. 

Dynamical recovery in this regime is evaluated by the flocking score $I_{\mathrm{flock}}$ defined in~\eqref{eq:Iflock}, computed from trajectories evolved under the learned interaction $\estphivec$ from time $T$ to time $2T$. Here $I_{\mathrm{flock}}$ serves as the primary dynamical diagnostic, since the emergent behavior is characterized by velocity alignment rather than convergence to a fixed spatial configuration.
\subsubsection{First-order Predator--Prey Model: Flocking}
\begin{figure}[H]
    \centering
    \begin{subfigure}[t]{0.46\linewidth}
        \centering
        \includegraphics[width=\linewidth]{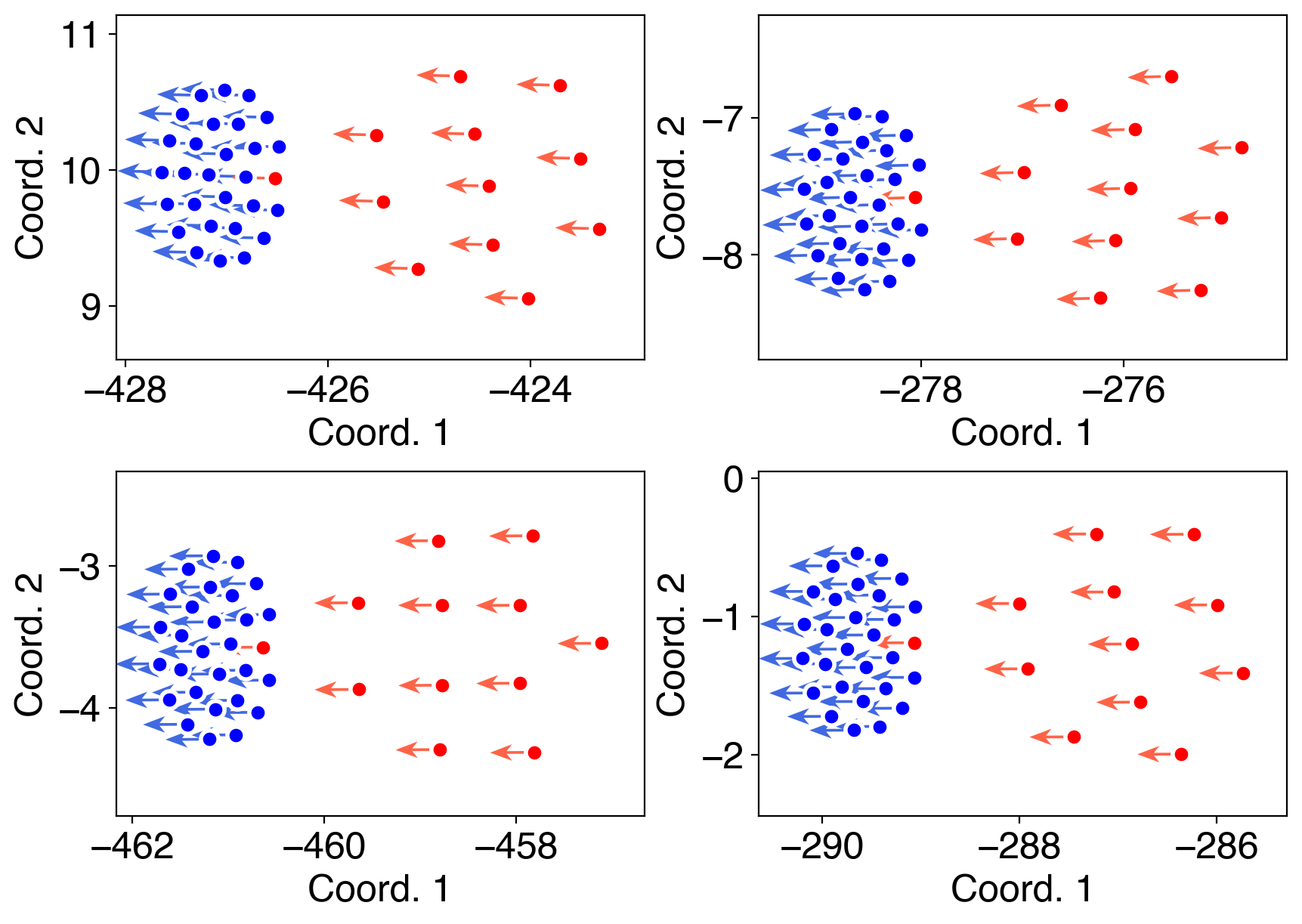}
        \caption{Evolution from the observed flocking configuration under the learned interaction $\hat \phi$.}
        \label{fig:MS-flock_traj_comp}
    \end{subfigure}
    \hfill
    \begin{subfigure}[t]{0.46\linewidth}
        \centering
        \includegraphics[width=\linewidth]{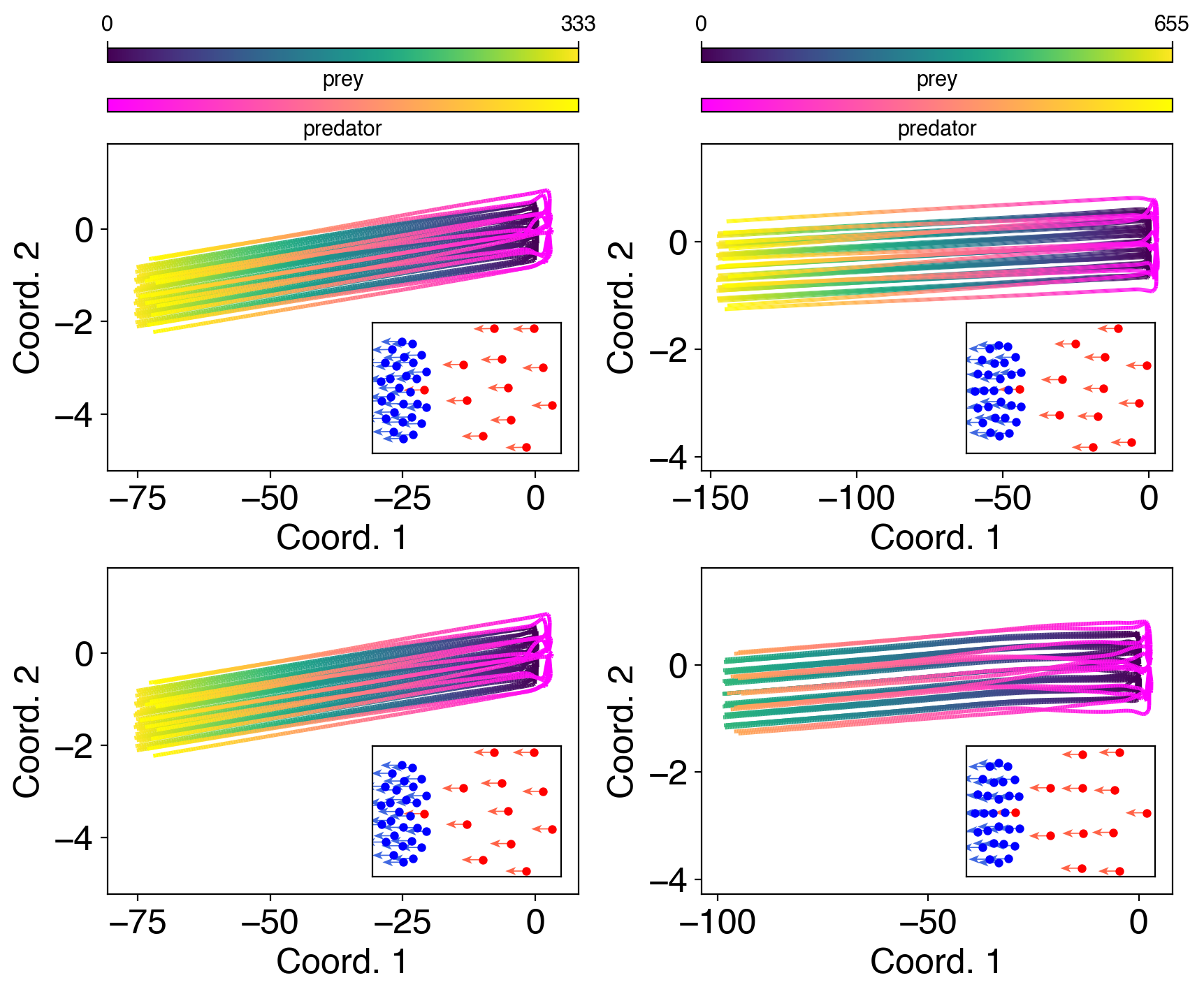}
        \caption{True (top) vs.\ learned (bottom).}
        \label{fig:ms-flock_ic_traj_comp}
    \end{subfigure}
    \caption{Predator--prey (flocking): (left) preservation of the observed terminal flocking state under $\hat \phi$; (right) trajectories from two initial conditions, true (top) vs. learned (bottom), final configurations inset. Prey in blue, predators in red.
}
\label{fig:MS-flock_combined}
\end{figure}
We first consider a first-order predator--prey system in a flocking regime. The
four directed interaction kernels have the same functional forms as in the
static predator--prey example, but their parameters are retuned so that the
system approaches a moving flock:
\begin{align*}
\phi_{ss}(r) &= a_{ss} - \frac{1}{r^2}, 
&\phi_{sp}(r) &= -\frac{b_{sp}}{r^2}, 
 \\
 \phi_{ps}(r) &= \frac{c_{ps}}{r^2},
& \phi_{pp}(r) &= -\frac{\tanh\!\big(d_{pp}(1-r)\big)+e_{pp}}{r},
\end{align*}
with $a_{ss}=2, b_{sp}=0.5, c_{ps}=0.5, d_{pp}=4, e_{pp}=0.5$, and a common cutoff $r_c=0.1$. In the terminal regime, the agents share a nonzero common velocity, i.e. $\dot{\bx}_i \to \bv_* \neq \zero$, while their relative positions remain approximately fixed. Each kernel is learned independently using a $\rho_T^{\idxcl_1,\idxcl_2}$-adaptive B-spline basis with refinement tolerance $P^{\mathrm{tol}}=0.008$.

This example represents the well-conditioned branch of the semi-static regime. At the learned terminal state, the first-order equation gives
\[
    \F_{\phivec^E}(\bX_T)=\bV_T,
    \qquad
    \bV_T=[\bv_T;\ldots;\bv_T]\in\R^{Nd},
\]
so the right-hand side of the learning equation is nonzero. Thus the problem falls into the nonzero-right-hand-side regime: the kernels are identified directly in the corresponding $\rho_T$-weighted norms, without an angular comparison or a post-processing scaling step.

The two panels of Figure~\ref{fig:MS-flock_combined} test the learned dynamics at the pattern level. Starting from the flocking configuration at time $T$ (observed terminal configuration), the learned interaction preserves the predator--prey flock; starting from two new initial configurations, it generates trajectories that closely match those of the true dynamics. Figure~\ref{fig:MS-flock_phi_comp} compares the four learned kernels with the true kernels and shows accurate recovery on the supports of the corresponding pairwise-distance distributions. Additional tests on the dependence on $N$ and robustness to observational noise are reported in Figures~\ref{fig:MS-flock_rhoT} and~\ref{fig:MS-flock_noise}.

\begin{figure}[H]
    \centering
    \includegraphics[width=\linewidth]{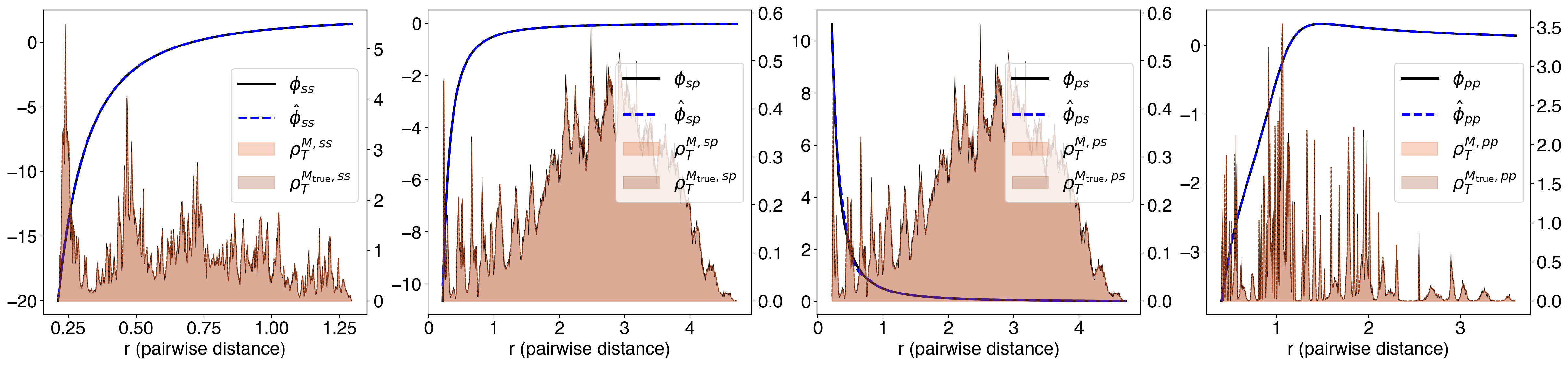}
    \caption{True kernels $\phi_{\idxcl_1,\idxcl_2}$ and learned kernels
    $\hat\phi_{\idxcl_1,\idxcl_2}$ for each directed species pair. The
    background measures show $\rho_T^{M,\idxcl_1,\idxcl_2}$ and
    $\rho_T^{M_{\rm true},\idxcl_1,\idxcl_2}$.}
    \label{fig:MS-flock_phi_comp}
\end{figure}
Quantitative results are summarized in Table~\ref{tab:MS-flocking}.
Since the learning equation has a nonzero right-hand side in this example,
the kernel errors are reported directly in the corresponding
$\rho_T^{\idxcl_1,\idxcl_2}$-weighted norms. The learned dynamics also produce a coherent flock, as reflected by the flocking score $I_{\mathrm{flock}}$ computed from trajectories evolved under the learned interaction.
\begin{table}[H]
\centering
\begin{tabular}{c c c c c | c c} 
\toprule
& $\rho_T^{ss}$ err. & $\rho_T^{sp}$ err. & $\rho_T^{ps}$ err. & $\rho_T^{pp}$ err. & $I_{\mathrm{flock}}$ (Mean) & $I_{\mathrm{flock}}$ (Std) \\
\midrule
Mean 
& $0.041$ & $0.026$  & $0.005$ & $0.005$ 
& $0.933$ & $0.101$ \\ 
Std  
& $0.011$ & $0.011$ & $0.002$ & $0.002$ 
& $0.005$ & $0.010$ \\
\bottomrule
\end{tabular}
\caption{Predator--prey flocking ($N_s=28$, $N_p=12$, $M=250$): relative kernel errors per type-pair, measured in the weighted norm $\|\cdot\|_{\rho_T^{k_1,k_2}}$, and the flocking score $I_{\mathrm{flock}}$, following the reporting convention of Section~\ref{sec:example}.}
\label{tab:MS-flocking}
\end{table}
This example shows that the conditioning of the learning problem is determined not only by the interaction laws but also by the type of collective pattern observed in the data. The static predator--prey pattern leads to a scale-ambiguous problem, whereas the moving predator--prey flock has a nonzero right-hand side and therefore fixes the scaling directly. This provides the first well-conditioned semi-static example before the second-order flocking cases below, where the limiting acceleration vanishes and the scaling ambiguity reappears.
\subsubsection{Second-order Fish Flocking}\label{sec:fish-flocking}
We turn to the second-order self-propelling particle model in~\cite{PhysRevLett.96.104302}. Each agent carries a self-propulsion and friction force together with a generalized Morse interaction; in the notation of~\eqref{eq:second_order} the alignment kernel is $\phi^A\equiv0$ and only the potential kernel $\phi^E$ is learned, with environmental force
\begin{equation}
    \f(\bx_i, \bv_i) = \alpha \bv_i - \beta \|\bv_i\|^2 \bv_i,
    \label{eq:morse_force}
\end{equation}
which drives the speed toward $\sqrt{\alpha/\beta}$, and
\begin{equation}
    \phi^E(r) =
\begin{cases}
\dfrac{1}{r}\!\left( \dfrac{C_a}{\ell_a} e^{-r/\ell_a}
                   - \dfrac{C_r}{\ell_r} e^{-r/\ell_r} \right), & r \ge r_c, \\[8pt]
\dfrac{1}{r_c}\!\left( \dfrac{C_a}{\ell_a} e^{-r_c/\ell_a}
                     - \dfrac{C_r}{\ell_r} e^{-r_c/\ell_r} \right), & r < r_c,
\end{cases}\label{eq:morse}
\end{equation}
with $C_a=1$, $\ell_a=1$, $C_r=0.6$, $\ell_r=0.5$, $\alpha=1$, $\beta=0.5$, and cutoff $r_c=0.01$. This same force and kernel are reused in the milling examples of Section~\ref{sec:ring_milling}, with only $C_r$ and the initial velocities changed. The parameters place the system in the catastrophic region of the H-stability diagram~\cite{PhysRevLett.96.104302}, where random initial velocities typically produce milling. Here we instead draw $\mu^{\bv}$ uniform on $[0,1]^d$, biasing the agents toward a common direction and selecting a flocking branch: they align and translate together at the characteristic speed with a rigid configuration.
\begin{figure}[H]
    \centering
        \begin{subfigure}[t]{0.46\linewidth}
        \centering
        \includegraphics[width=\linewidth]{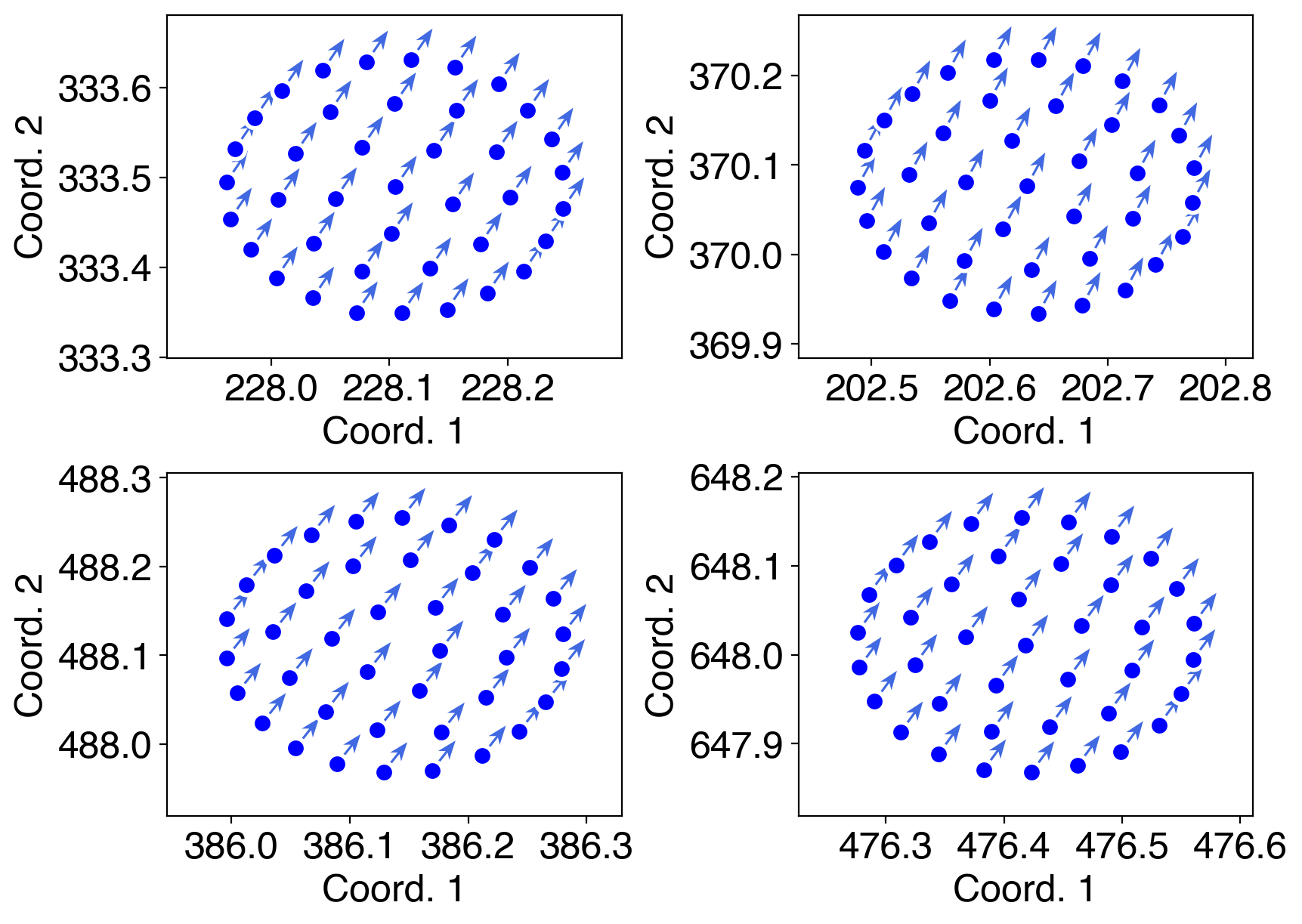}
        \caption{Flock evolved under $\hat\phi^E$. }
        \label{fig:fish-flock_traj_comp}
    \end{subfigure}
    \hfill
        \begin{subfigure}[t]{0.46\linewidth}
        \centering
        \includegraphics[width=\linewidth]{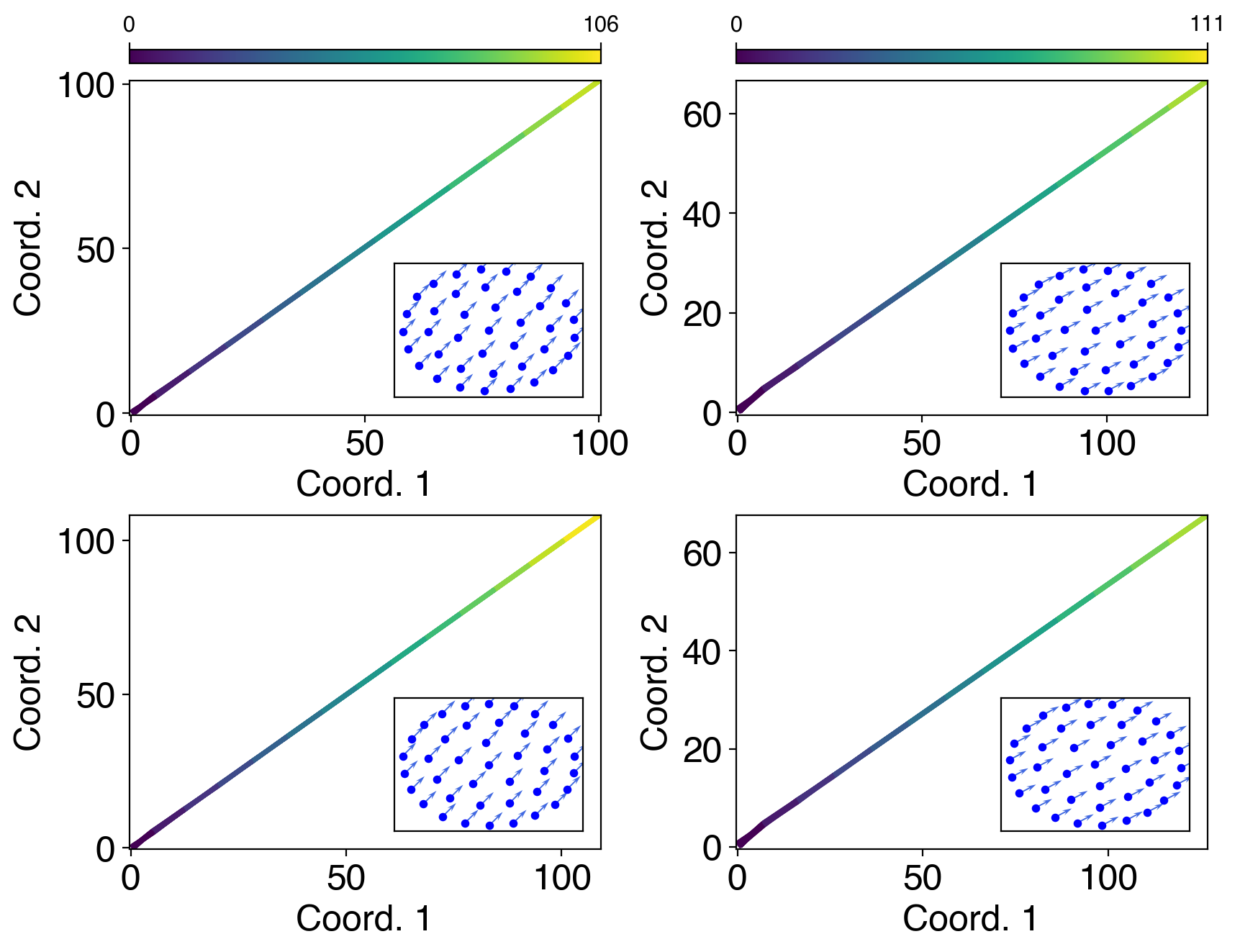}
        \caption{True (top) vs.\ learned $\hat c_*\hat\phi^E$ (bottom).
    }
          \label{fig:fish-flock_ic_traj_comp}
    \end{subfigure}
   \caption{Fish flocking: (left) observed flocking state evolved under $\hat\phi^E$;
(right) trajectories from two initial conditions, true (top) vs.\ learned
$\hat c_*\hat\phi^E$ (bottom), final configurations inset.}
   \label{fig:fish-flock_combined}
\end{figure}
This flocking state is a second-order flocking state, so its identifiability is opposite to the first-order case. The data entering the second-order learning problem~\eqref{eq:loss_2nd} are accelerations: in the flocking state $\dot\bv_i\to\zero$, and because the speed settles at $\|\bv_i\|^2=\alpha/\beta$, the self-propulsion force vanishes, $\f\equiv\zero$, on the data. The force is not absent but evaluated at its own zero set, so the right-hand side collapses to $\F_{\phi^E}(\bX_T)=\zero$. Where the first-order moving flock saw a nonzero
$\dot\bx_i=\bv_*$ and was well-posed, the second-order flock therefore
falls back into the scale-ambiguous zero-right-hand-side regime:
$\phi^E$ is identified up to
$\phi^E \approx c_*\hat\phi^E$,
with $c_*$ recovered by the first-stopping-time procedure. We use a degree-$1$ B-spline basis with $P^{\mathrm{tol}}=0.01$.

As in the previous cases, the two panels of Figure~\ref{fig:fish-flock_combined} verify that $\hat\phi^E$ preserves the observed flocking state and that the rescaled $\hat c_*\hat\phi^E$ reproduces the trajectories from fresh initial conditions; the normalized kernel comparison is in Figure~\ref{fig:fish-flock_phi_comp}. Since $\rho_T$ is concentrated on a narrow interval, the $N$-dependence (Appendix, Figure~\ref{fig:fish-flock_rhoT}) mirrors the support-coverage behavior of the static examples: small $N$ leaves gaps in the observed distances, while larger $N$ stabilizes the recovery on $\supp{\rho_T}$. Scaling recovery and noise robustness are reported in the Appendix (Figures~\ref{fig:fish-flock_C_comp} and~\ref{fig:fish-flock_noise}).
\begin{figure}[H]
    \centering
    \includegraphics[width=\linewidth]{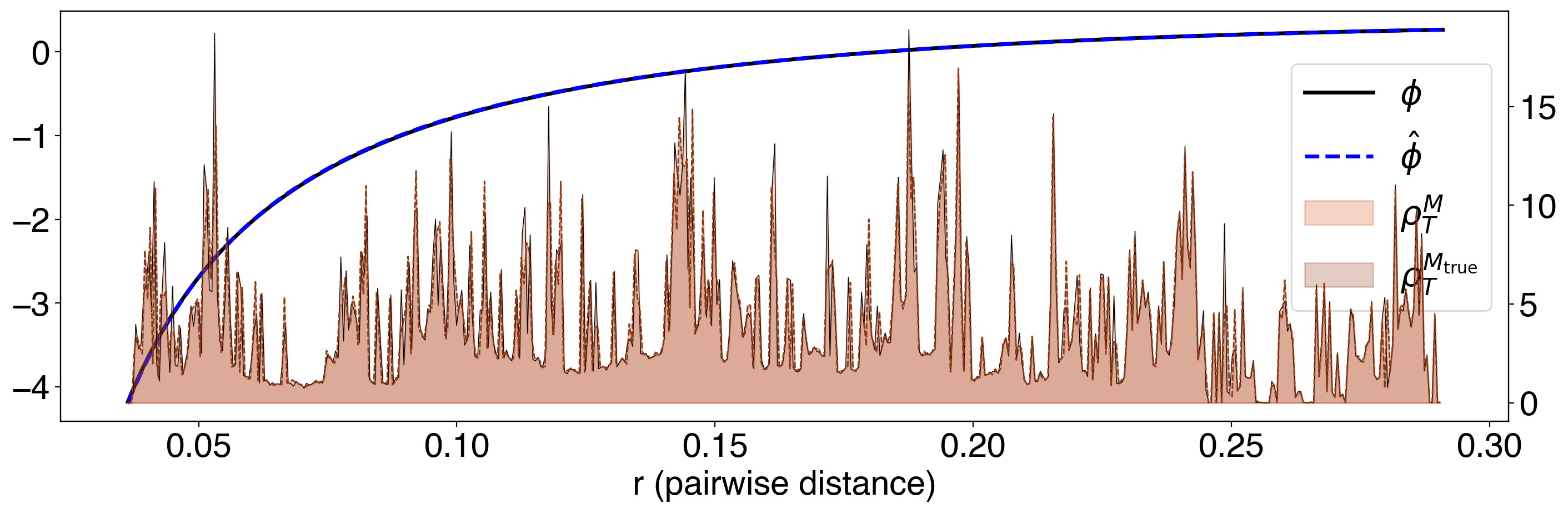}
    \caption{Fish flocking: normalized $\phi^E$ and $\hat\phi^E$; background
$\rho_T^{M}$ vs.\ $\rho_T^{M_{\text{true}}}$.}
    \label{fig:fish-flock_phi_comp}
\end{figure}
Table~\ref{tab:fish-flocking} reports $\theta$, the recovered $\hat c_*$, and the flocking score $I_{\mathrm{flock}}$~\eqref{eq:Iflock} evaluated under $\hat c_*\hat\phi^E$. The near-unit $I_{\mathrm{flock}}$ confirms that the learned dynamics reproduce the coherent translating flock. The example makes the semi-static dichotomy precise: a moving flock is not automatically well-posed. What matters is the right-hand side of the learning equation, and in second-order flocking the measured acceleration vanishes, so the scaling ambiguity of the static regime reappears even though the emergent pattern is a coherent flock.
\begin{table}[H]
\centering
\begin{tabular}{ c c c c c | c c} 
\toprule
 & $\theta$ & $c_*$ & $\hat c_*$ (Mean) & $\hat c_*$ (Std) 
 & $I_{\mathrm{flock}}$ (Mean) & $I_{\mathrm{flock}}$ (Std) \\ \hline
 Mean & $1.41\times 10^{-4}$ & $11.81$ & $11.77$ & $3.32\times 10^{-2}$ 
      & $0.9999999$ & $1.22\times 10^{-7}$ \\ 
 Std  & $7.27\times 10^{-9}$ & $0.14$  & $0.27$  & $1.69\times 10^{-2}$ 
      & $2.33\times 10^{-8}$ & $3.67\times 10^{-7}$ \\ 
\bottomrule
\end{tabular}
\caption{Fish flocking ($N=40$, $M=250$): angle $\theta$, true and recovered
scaling $c_*,\hat c_*$, and flocking score $I_{\mathrm{flock}}$, following the
reporting convention of Section~\ref{sec:example}.}
\label{tab:fish-flocking}
\end{table}
\subsubsection{Second-order Flocking with Anticipation Dynamics}\label{sec:anticipate}
Our final semi-static example is the anticipation dynamics proposed in~\cite{AD2021}, in which agents react to their neighbors' anticipated positions $\bx_i^\tau = \bx_i + \tau \bv_i$. It is a second-order variant of~\eqref{eq:second_order} in which the single-variable energy kernel is replaced by a two-variable kernel $\phi^E(r,\xi)$, with $\xi := (\bx_j-\bx_i)^\top (\bv_j-\bv_i)$. With $\f \equiv \zero$, and
\[
\dot\bv_i = \sum_{j \neq i}\frac{\tau}{N}
\Big(\phi^E\big(\norm{\bx_j - \bx_i}, \xi\big)(\bx_j - \bx_i)
   + \phi^A(\norm{\bx_j - \bx_i})(\bv_j - \bv_i)\Big),
\]
where, crucially, both kernels are induced by a single radial potential $U$, i.e.
\[
\phi^A(r) = \frac{U'(r)}{r},
\qquad
\phi^E(r,\xi) = -\frac{U'(r)}{r^3}\,\xi + \frac{U''(r)}{r^2}\,\xi
               + \frac{1}{\tau}\frac{U'(r)}{r}.
\]
The $\xi$-dependence comes from the radial and tangential parts of the Hessian $D^2U$ in the $\tau$-expansion of the dynamics~\cite{AD2021}, so $\phi^E$ and $\phi^A$ are tied through $U$, a coupling which the authors use to prove velocity alignment and spatial concentration.

We take the cutoff-regularized convex potential
\[
U(r) = \frac{r^2}{2} - \ln r + C, \qquad r \ge r_c, \quad r_c = 0.1,
\]
which gives
\[
\phi^A(r) = 1 - \frac{1}{r^2},
\qquad
\phi^E(r,\xi) = \frac{2}{r^4}\,\xi + \frac{1}{\tau}\Big(1 - \frac{1}{r^2}\Big),
\]
with anticipation parameter $\tau = 40$. Initial positions are drawn from $\mu^{\bx}$ uniform on $[0,1]^d$, and initial velocities from $\mu^{\bv}$ uniform on $[0,1]^d$; the positive velocity components bias the agents toward a common direction, so the system selects a flocking branch in which $\dot\bv_i \to \zero$, $\bv_i \to \bv_*$, and the configuration is rigid. We learn on a degree-$1$ B-spline basis adapted to the pairwise-distance distribution $\rho_T$, with $P^{\mathrm{tol}} = 0.01$.

This example sharpens the dichotomy of the previous two. On flocking data all velocities coincide, so $\bv_j - \bv_i = \zero$ and hence $\xi = 0$. Two things follow. First, the alignment term $\phi^A(r_{i,j})(\bv_j - \bv_i)$ is identically zero, so $\phi^A$ contributes nothing to the loss and cannot be identified directly from flocking data. Second, the Hessian part of $\phi^E$ vanishes with $\xi$, and since the force carries the prefactor $\tau$, the observed contribution reduces to $\tau\,\phi^E(r,0) = U'(r)/r$, which is exactly the form of $\phi^A$.  On the data the two kernels therefore collapse onto the same scalar quantity $U'(r)/r$, and the learning problem reduces to $\F^E_{\varphi}(\bX_T) = \zero$. As in fish flocking, this is the scale-ambiguous zero-right-hand-side
regime, with $\varphi := U'(r)/r$ identified up to
$\varphi \approx c_*\hat\varphi$ and $c_*$ recovered by the
first-stopping-time procedure.

The shared-potential structure is what makes the otherwise unlearnable $\phi^A$ recoverable. Rather than fit the two kernels separately, we learn the single scalar function $\varphi = U'(r)/r$ and integrate it back to the potential,
\[
\hat U(r) = \int_{r_c}^{r} s\,\hat\varphi(s)\, ds,
\]
recovered up to the same multiplicative scaling $c_*$ and an arbitrary additive constant. Differentiating the B-spline representation of $\hat\varphi$ then gives $\hat U', \hat U''$, and hence the reconstructed kernels $\hat\phi^A$ and the two-variable function $\hat\phi^E$, which inherit the scaling $c_*$. Figure~\ref{fig:anticipate_kernels_comp} compares these with the truth: the one-dimensional $\hat\phi^A$ and $\hat U$ agree on the observed $r$-support, while the $\xi$-dependence of the surface $\hat\phi^E$ is supplied by the structure
rather than observed, since the data lie on the slice $\xi=0$.
\begin{figure}[H]
    \centering
    \begin{subfigure}[c]{0.34\linewidth}
        \centering
        \includegraphics[width=\linewidth]{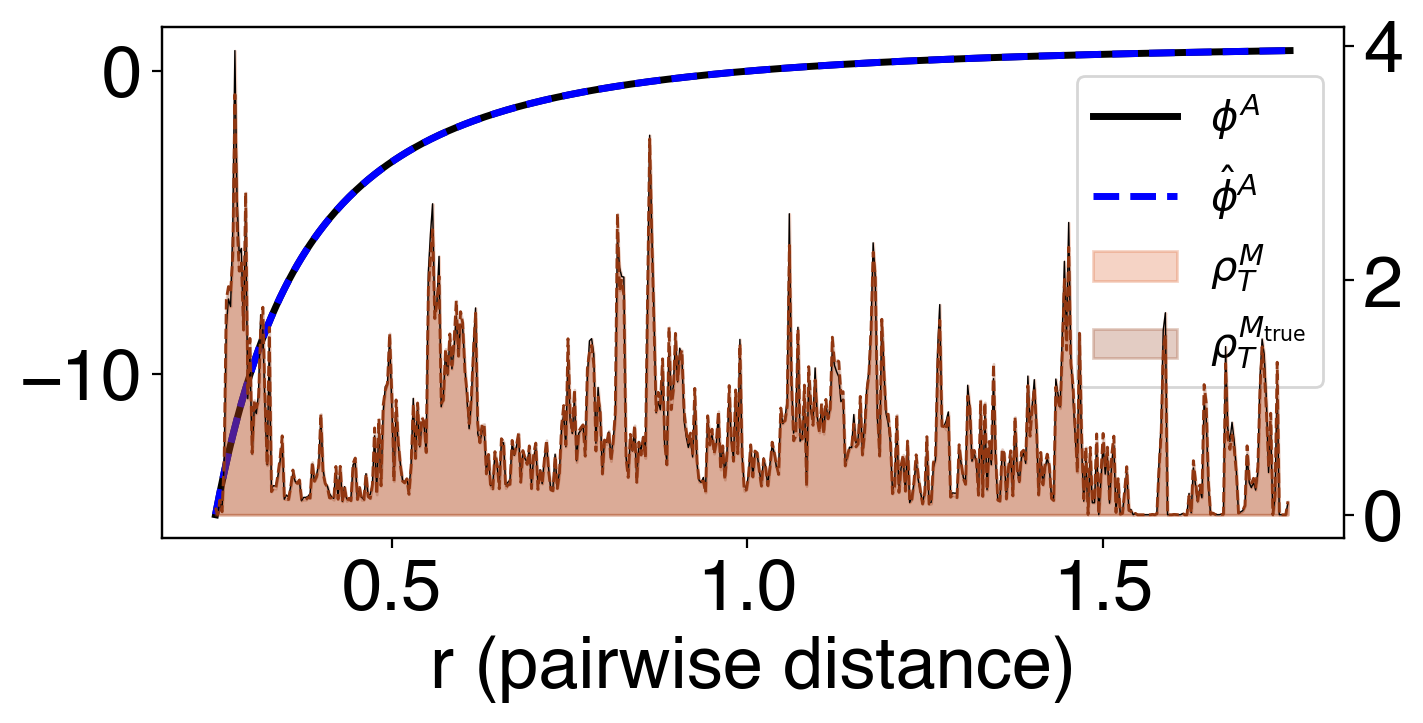}\\
        \includegraphics[width=\linewidth]{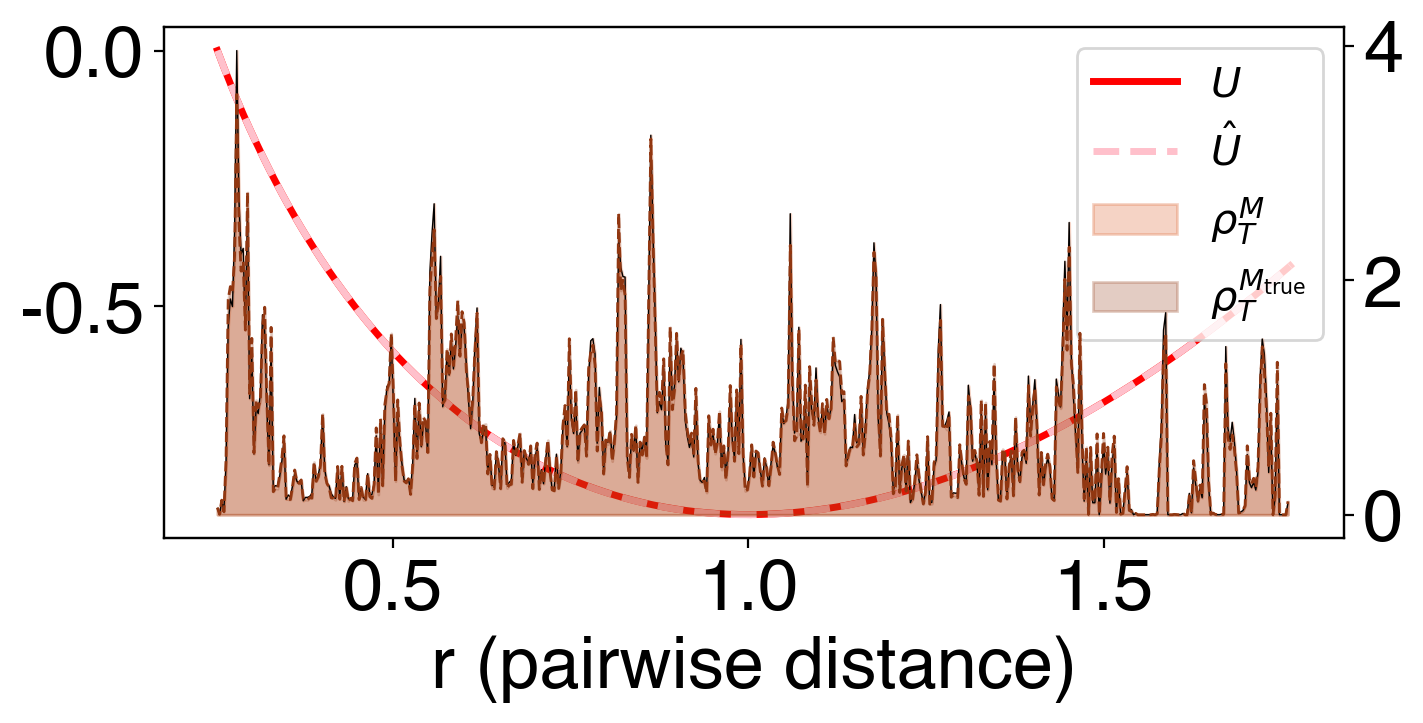}
    \end{subfigure}
    \hfill
    \begin{subfigure}[c]{0.64\linewidth}
        \centering
        \includegraphics[width=\linewidth]{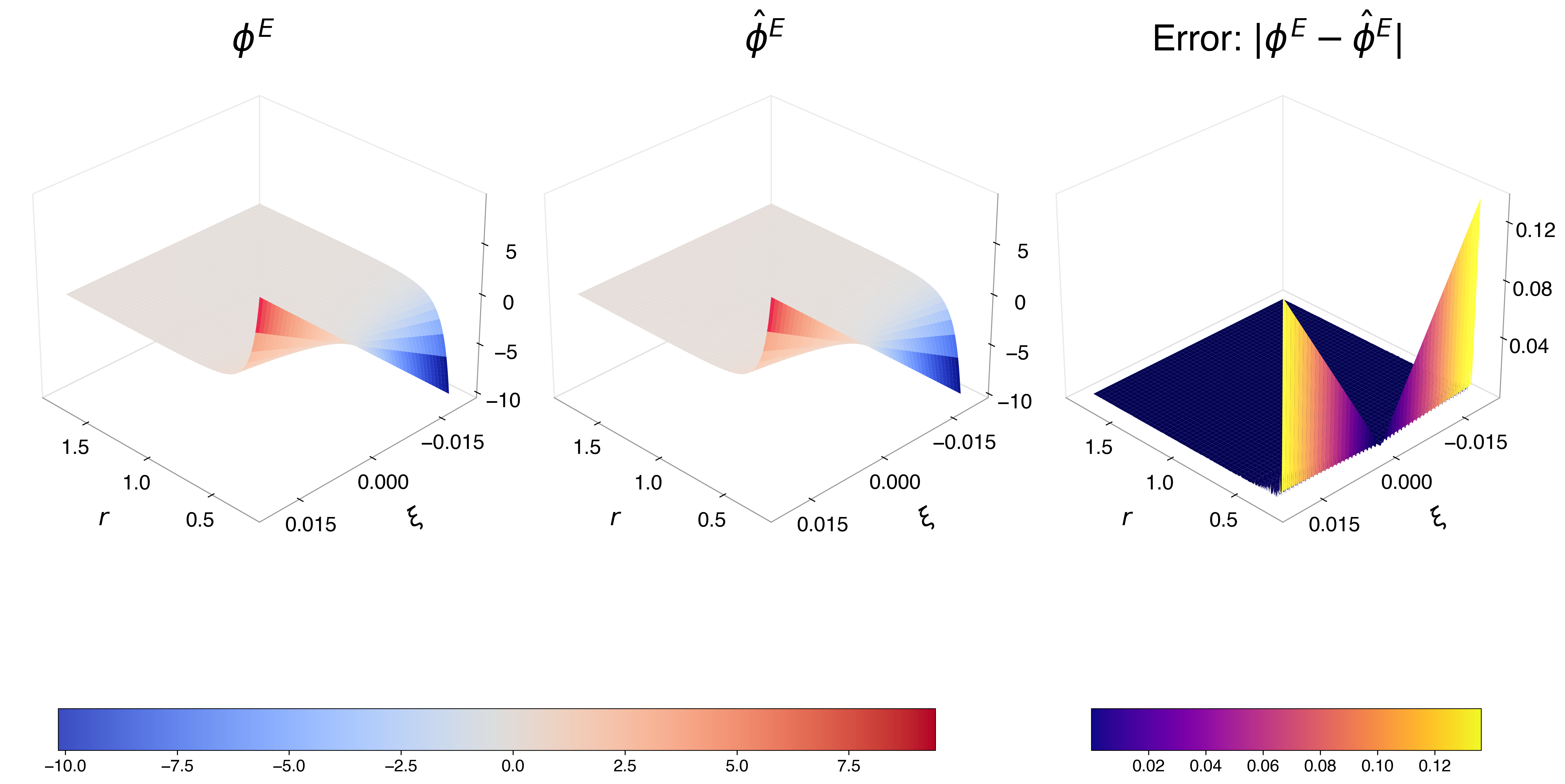}
    \end{subfigure}
    \caption{Anticipated flocking, normalized reconstructions: $\phi^A$ and $\hat\phi^A$ (top left), $U$ and $\hat U$ (bottom left), and the two-variable $\phi^E$ and $\hat\phi^E$ (right). Background: pairwise-distance measures $\rho_T^{M}$ and $\rho_T^{M_{\rm true}}$.}
    \label{fig:anticipate_kernels_comp}
\end{figure}
As in the previous cases, the two panels of Figure~\ref{fig:an-flock_combined} verify that the reconstructed kernels preserve the observed flocking state and that their rescaled form reproduces the trajectories from fresh initial conditions. The $N$-dependence matches fish flocking: the rigid flock gives a compactly supported $\rho_T$, so small $N$ leaves gaps and larger $N$ stabilizes the recovery (Appendix, Figure~\ref{fig:anticipate_rhoT}). Scaling recovery and noise robustness are in the Appendix (Figures~\ref{fig:anticipate_C_comp} and~\ref{fig:anticipate_noise}).
\begin{figure}[H]
    \centering
    \begin{subfigure}[t]{0.46\linewidth}
        \centering
        \includegraphics[width=\linewidth]{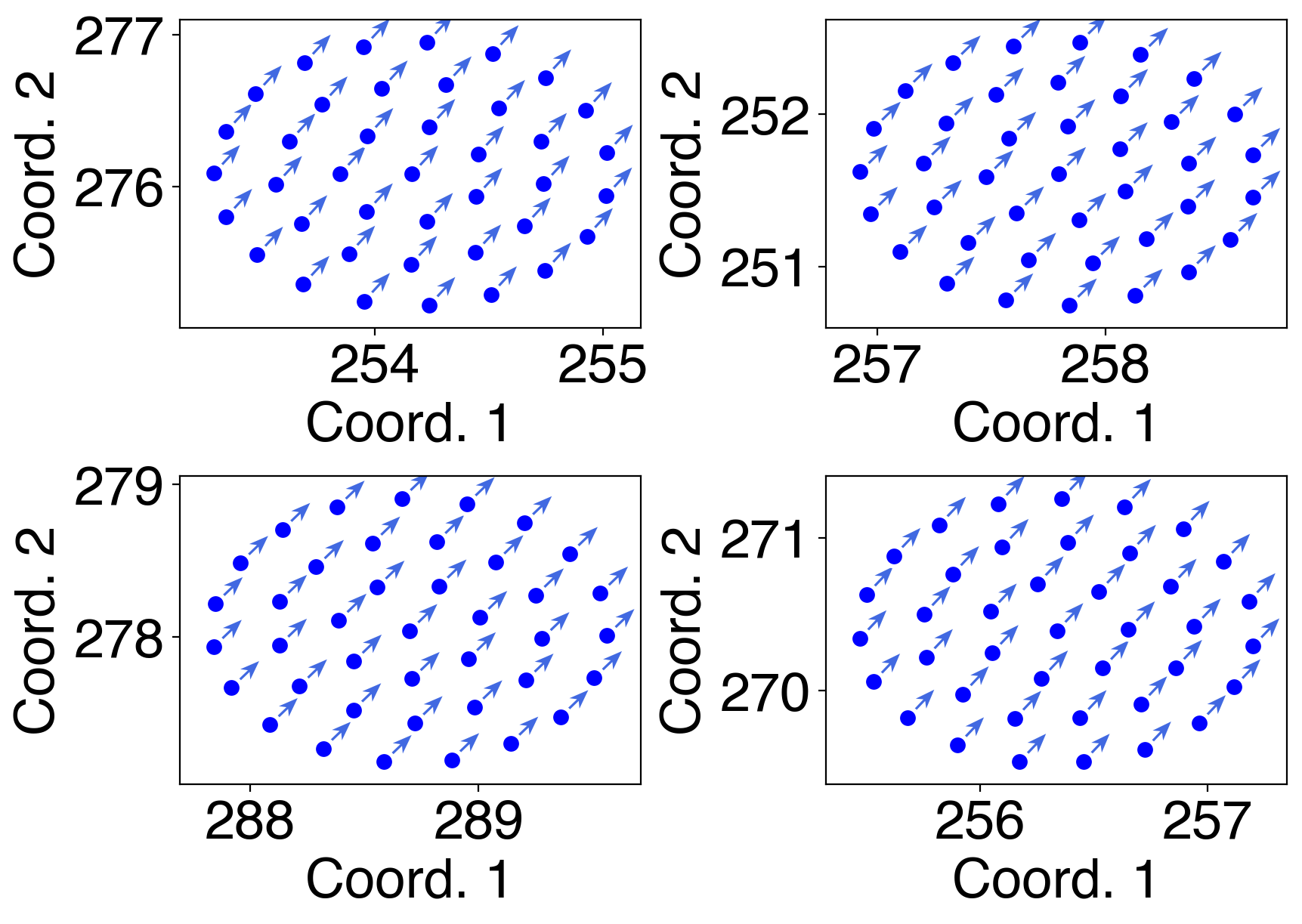}
        \caption{Flock evolved under the reconstructed kernels.}
        \label{fig:anticipate_traj_comp}
    \end{subfigure}
    \hfill
    \begin{subfigure}[t]{0.46\linewidth}
        \centering
        \includegraphics[width=\linewidth]{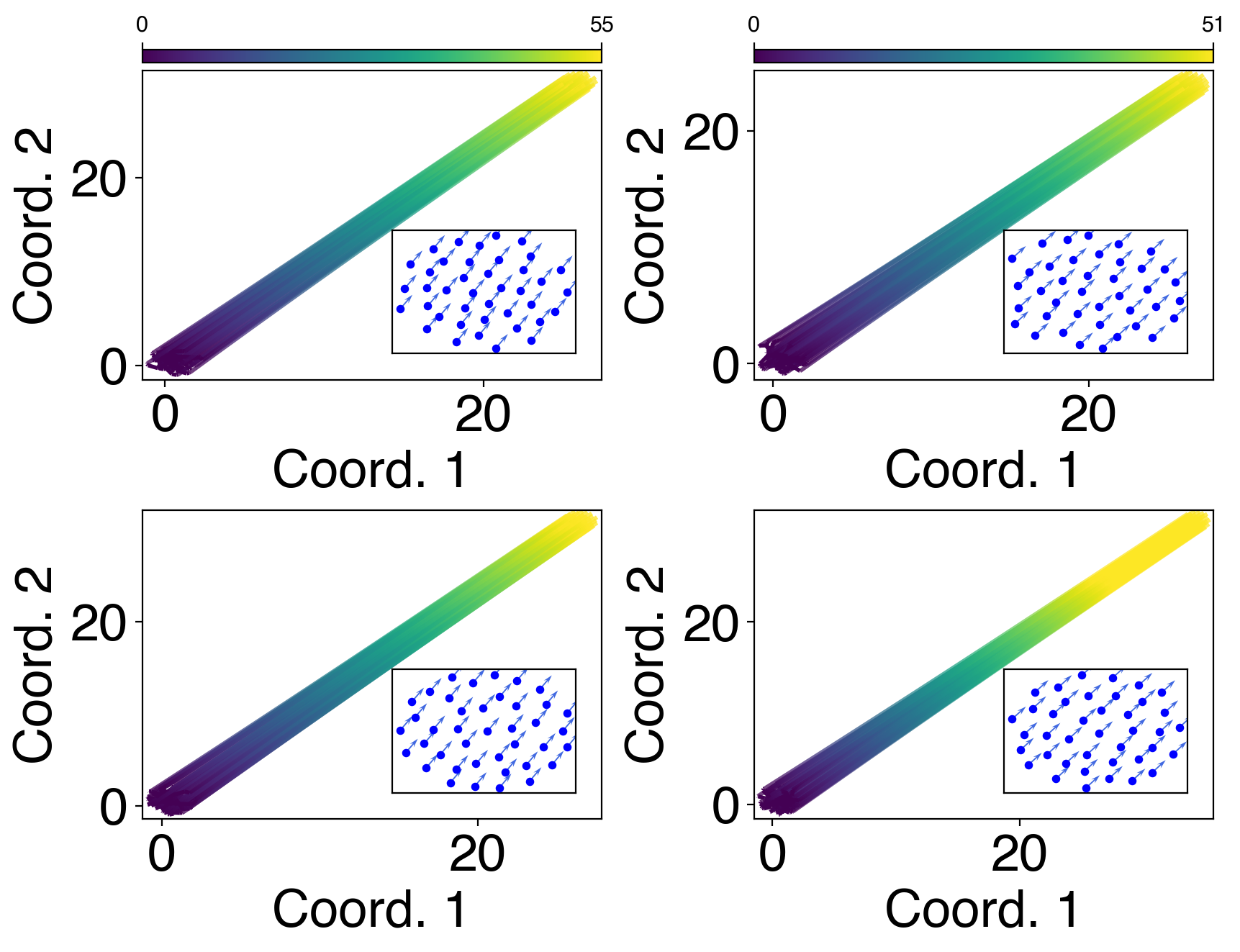}
        \caption{True (top) vs.\ rescaled reconstructed (bottom).}
        \label{fig:an_ic_traj_comp}
    \end{subfigure}
    \caption{Anticipated flocking: (left) observed flock evolved under the reconstructed kernels; (right) trajectories from two initial conditions, true (top) vs.\ rescaled reconstructed (bottom), final configurations inset.}
    \label{fig:an-flock_combined}
\end{figure}
Table~\ref{tab:anticipate} reports $\theta$, the recovered $\hat c_*$, and the flocking score $I_{\mathrm{flock}}$~\eqref{eq:Iflock} evaluated under the rescaled reconstructed kernels; the near-unit $I_{\mathrm{flock}}$ confirms a coherent flock. This closes the semi-static progression. Fish flocking showed that a second-order moving flock sends the learning right-hand side to zero, so the scaling ambiguity of the static regime returns. Anticipated flocking goes one step further: the flocking data do not identify the alignment kernel $\phi^A$ at all, and its recovery is possible only because the model ties $\phi^A$ and $\phi^E$ to a common potential $U$. Structure, not data, resolves the un-learnability of $\phi^A$.
\begin{table}[H]
\centering
\begin{tabular}{ c c c c c | c c}
\toprule
 & $\theta$ & $c_*$ & $\hat c_*$ (Mean) & $\hat c_*$ (Std)
 & $I_{\mathrm{flock}}$ (Mean) & $I_{\mathrm{flock}}$ (Std) \\
\midrule
 Mean & $2.45\times 10^{-4}$ & $47.46$ & $48.12$ & $0.43$
      & $1 - 2.00\times 10^{-7}$ & $1.47\times 10^{-6}$ \\
 Std  & $9.73\times 10^{-9}$ & $0.65$  & $0.81$  & $0.12$
      & $1.36\times 10^{-7}$ & $2.16\times 10^{-6}$ \\
\bottomrule
\end{tabular}
\caption{Anticipated flocking ($N=40$, $M=250$): angle $\theta$, true and
recovered scaling $c_*,\hat c_*$, and flocking score $I_{\mathrm{flock}}$,
following the reporting convention of Section~\ref{sec:example}.}
\label{tab:anticipate}
\end{table}
\subsection{Quasi-static Patterns from Second-Order Systems}
We close with quasi-static configurations, where positions and velocities never become constant; instead the agents organize into a persistent rotating pattern with frozen pairwise distances. This is the self-propelling particle model of Section~\ref{sec:fish-flocking}, in the same catastrophic regime as studied in~\cite{PhysRevLett.96.104302}, but with a different terminal state, and the contrast is exactly what makes the regime informative. In fish flocking the speed settled at the zero set of the self-propulsion
force, the acceleration vanished, and the learning problem fell back to the
scale-ambiguous zero-right-hand-side regime. In a milling state the agents
keep turning, so the centripetal acceleration $\dot\bv_i$ stays nonzero on
the data. The right-hand side of~\eqref{eq:loss_2nd} is therefore nonzero,
placing the problem in the well-posed nonzero-right-hand-side regime, and we
report errors in the $\rho_T$-weighted norm directly, with no scaling or sign
recovery.
\begin{figure}[H]
    \centering
    \begin{subfigure}[t]{0.46\linewidth}
        \centering
        \includegraphics[width=\linewidth]{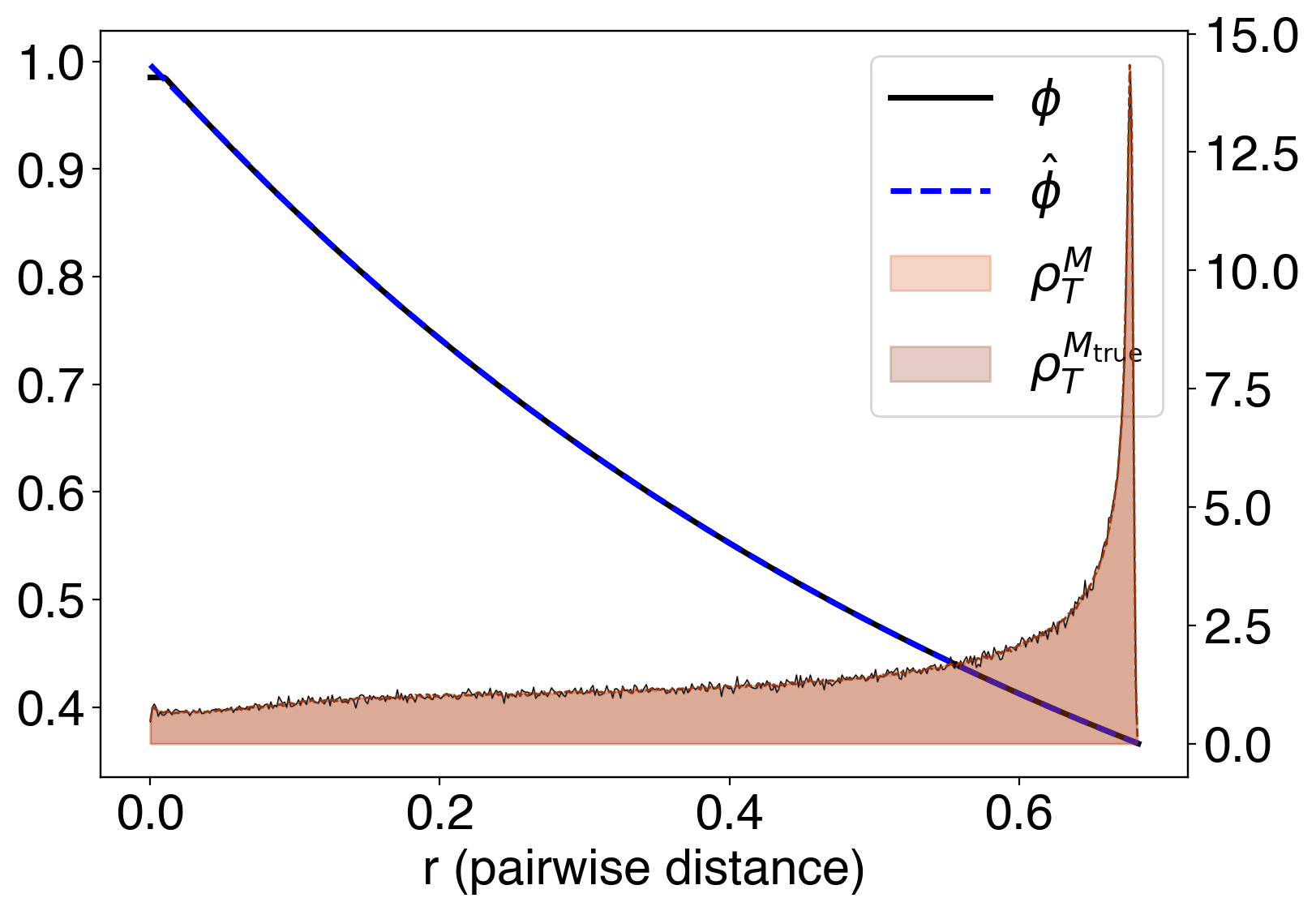}
        \caption{Ring milling.}
        \label{fig:fish-ring_phi_comp}
    \end{subfigure}
    \hfill
    \begin{subfigure}[t]{0.46\linewidth}
        \centering
        \includegraphics[width=\linewidth]{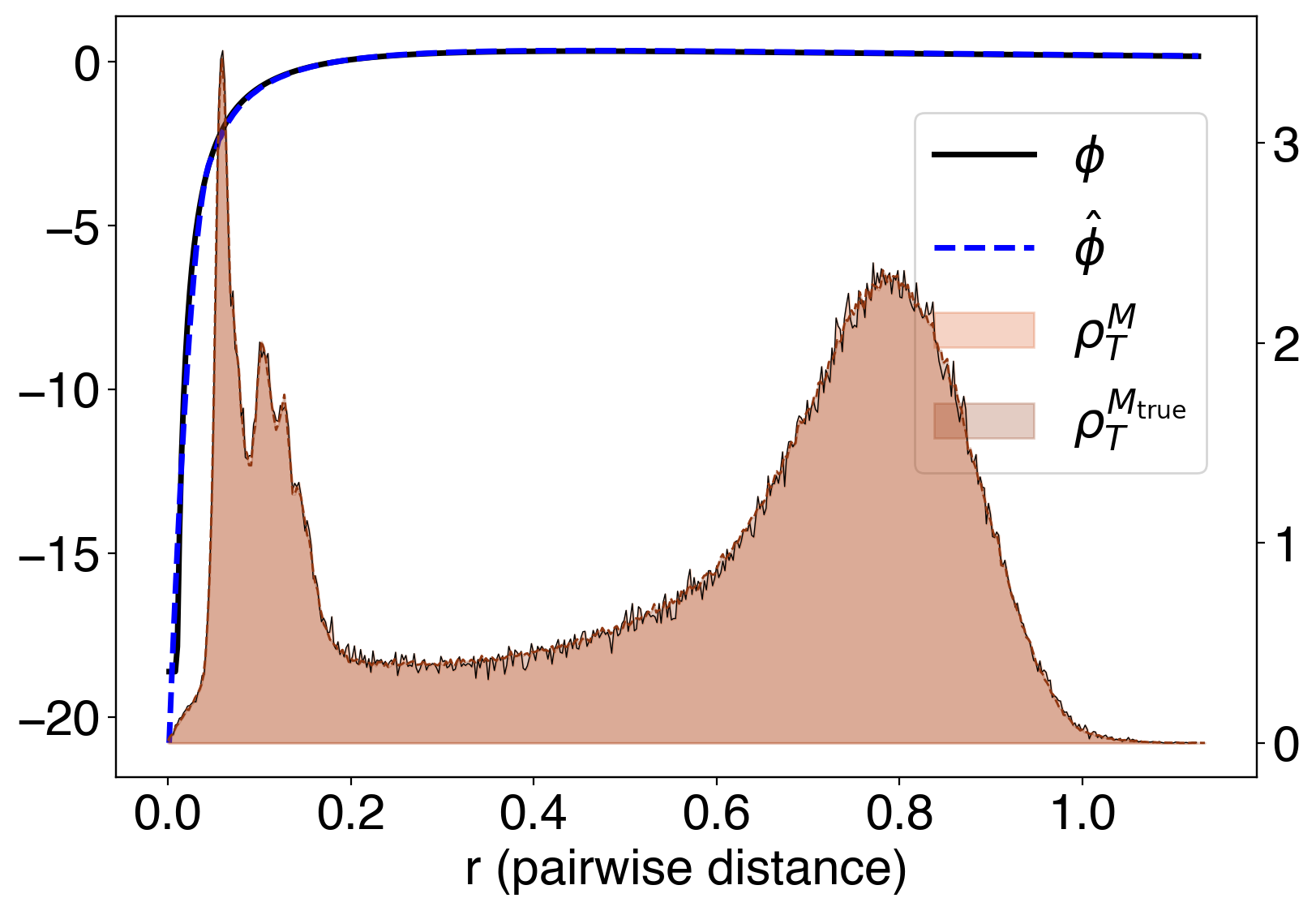}
        \caption{Double milling.}
        \label{fig:double-mill_phi_comp}
    \end{subfigure}
    \caption{Quasi-static milling: true and learned kernels $\phi^E$ and
    $\hat\phi^E$. Background $\rho_T^{M}$ vs.\ $\rho_T^{M_{\rm true}}$.}
    \label{fig:milling_phi_comp}
\end{figure}
We reuse the self-propulsion force~\eqref{eq:morse_force} and the Morse form~\eqref{eq:morse}, keeping $(\alpha,\beta,\ell_a,\ell_r,C_a)$ as in Section~\ref{sec:fish-flocking} and varying only $C_r$ to realize the two canonical milling states of~\cite{PhysRevLett.96.104302}. Where fish flocking used biased initial velocities to pick the flocking branch, here we set the initial velocities to zero and let the self-propulsion generate rotation, selecting a milling branch; initial positions follow the static and semi-static cases. Convergence is assessed by the milling criterion, which is based on the milling score~\eqref{eq:Imill}, with $\mathrm{Tol}=10^{-6}$.

The two subsections below give the single-ring and double-milling cases in turn; their results are summarized in Table~\ref{tab:fish-millings} and Figure~\ref{fig:milling_phi_comp}. Both kernels are recovered accurately, but the double mill has a substantially larger error than the single ring, reflecting its broader, multi-modal $\rho_T$ and the more involved geometry we discuss in Section~\ref{sec:double_milling}.
\begin{table}[H]
\centering
\small
\begin{tabular}{c c c c}
\toprule
Pattern & $\rho_T$ err. & $I_{\mathrm{mill}}$ (Mean) & $I_{\mathrm{mill}}$ (Std) \\
\midrule
Ring milling
 & $(8.88\pm 0.07)\times 10^{-5}$
 & $1 - (4.10\pm 0.19)\times 10^{-7}$
 & $(2.97\pm 0.04)\times 10^{-7}$ \\
Double milling
 & $(6.33\pm 0.66)\times 10^{-2}$
 & $1 - (5.00\pm 0.10)\times 10^{-3}$
 & $(2.72\pm 0.23)\times 10^{-3}$ \\
\bottomrule
\end{tabular}
\caption{Quasi-static milling ($N=40$, $M=250$): kernel error in the weighted norm $\|\cdot\|_{\rho_T}$ and milling score $I_{\mathrm{mill}}$, following the reporting convention of Section~\ref{sec:example}.}
\label{tab:fish-millings}
\end{table}
\subsubsection{Fish Milling: Rings}\label{sec:ring_milling}
The first quasi-static case is the single-ring milling state of~\cite{PhysRevLett.96.104302}, with Morse parameters $C_a=\ell_a=1$, $C_r=0.5$, $\ell_r=0.5$, self-propulsion $\alpha=1$, $\beta=0.5$, and cutoff $r_c=0.01$. With $\ell_r/\ell_a=C_r/C_a=0.5$, the agents settle onto a single ring rotating at the characteristic speed $\sqrt{\alpha/\beta}$ with fixed pairwise distances (Figure~\ref{fig:fish-ring_traj_comp}). They split into two subgroups circulating in opposite senses, but both orbit the common center $\bar\bx=\frac{1}{N}\sum_j\bx_j$.
\begin{figure}[H]
    \centering
    \begin{subfigure}[t]{0.46\linewidth}
        \centering
        \includegraphics[width=\linewidth]{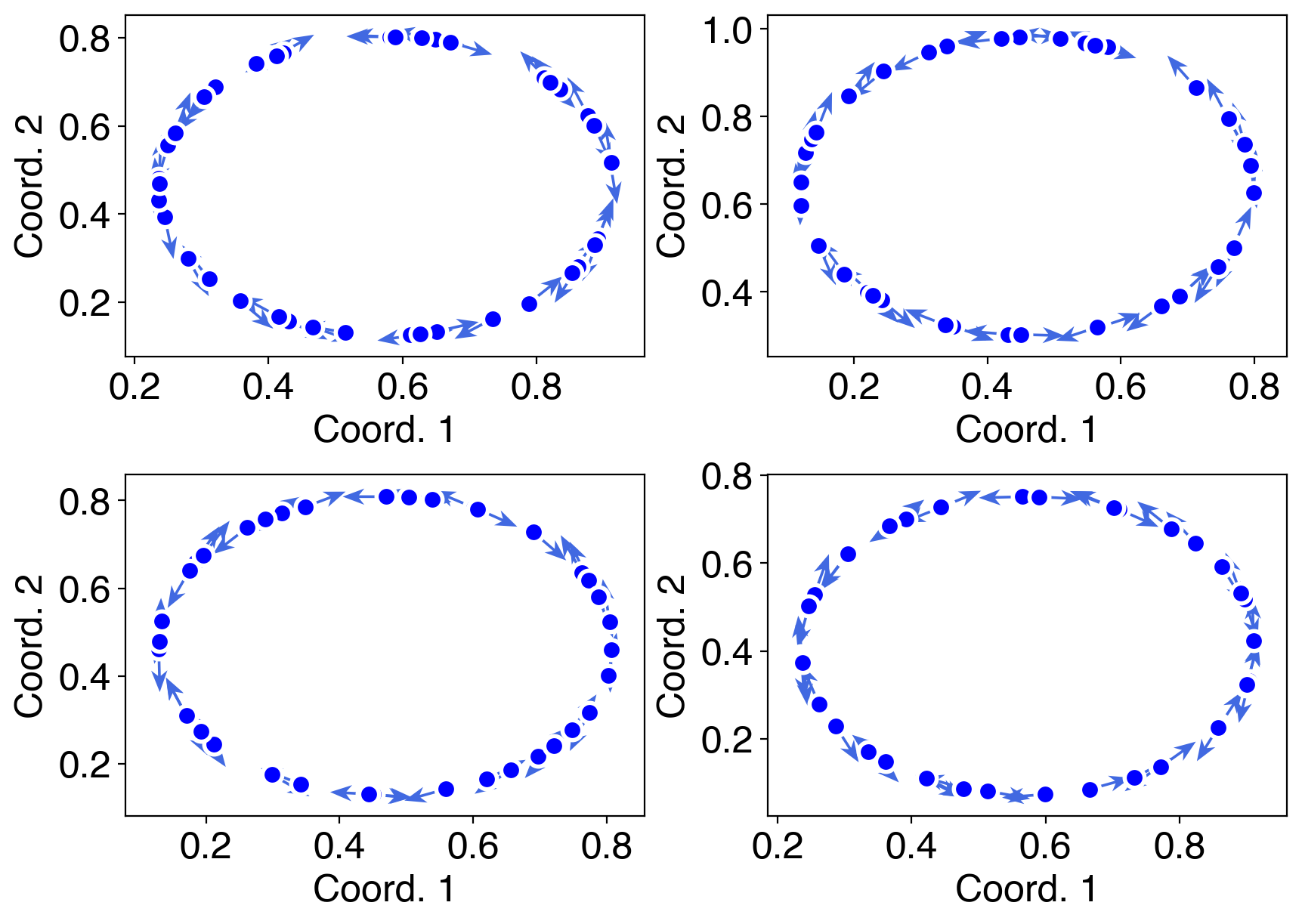}
        \caption{Ring milling evolved under $\hat\phi^E$.}
        \label{fig:fish-ring_traj_comp}
    \end{subfigure}
    \hfill
    \begin{subfigure}[t]{0.46\linewidth}
        \centering
        \includegraphics[width=\linewidth]{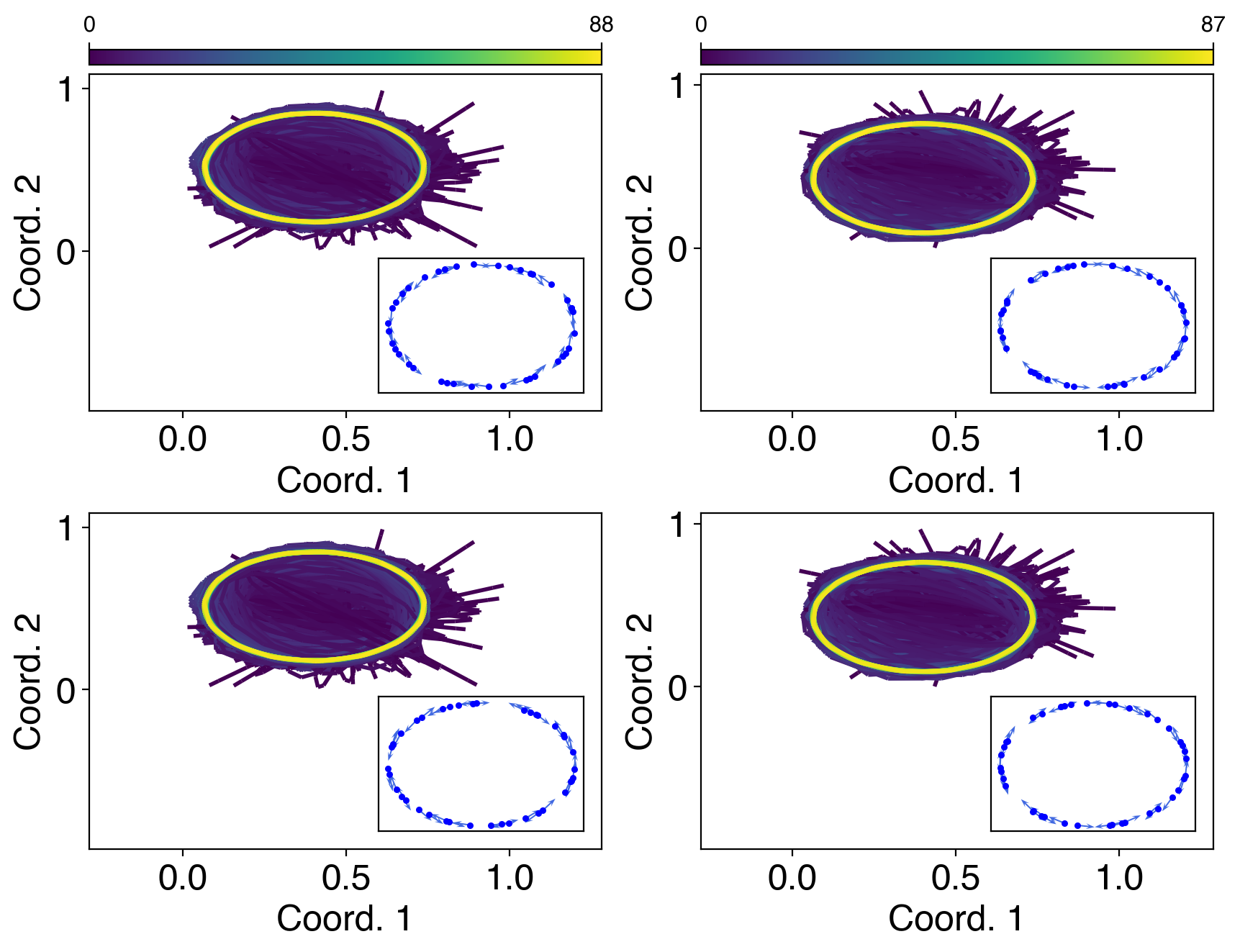}
        \caption{True (top) vs.\ learned $\hat\phi^E$ (bottom).}
        \label{fig:fish-mill-1_ic_traj_comp}
    \end{subfigure}
    \caption{Ring milling: (left) observed milling state evolved under
    $\hat\phi^E$; (right) trajectories from two initial conditions, true (top)
    vs.\ learned (bottom), final configurations inset.}
    \label{fig:fish-ring_combined}
\end{figure}
What distinguishes this regime from every earlier example is that the observed data are positions and velocities $\{\bx_i^m,\bv_i^m\}$ only, with no direct access to the accelerations $\dot\bv_i^m$ on the left-hand side of~\eqref{eq:loss_2nd}. For a rigid rotation, however, $\dot\bv_i$ is the
centripetal acceleration: it points from $\bx_i$ toward $\bar\bx$ with magnitude $\|\bv_i\|^2/\|\bx_i^c\|$, where $\bx_i^c=\bx_i-\bar\bx$. We therefore reconstruct the centripetal acceleration
\begin{equation}\label{eq:centripetal}
\dot\bv_i^m = -\frac{\|\bv_i^m\|^2}{\|\bx_i^{c,m}\|^2}\,\bx_i^{c,m}
\end{equation}
per snapshot and use it as the right-hand side of~\eqref{eq:loss_2nd}. The reconstruction uses only $\bx_i$, $\bar\bx$, and the speed $\|\bv_i\|$, not the sense of circulation, so it applies identically to the clockwise and counter-clockwise subgroups. Since the ring rotates at the characteristic speed, $\f\equiv\zero$ on the
data while $\dot\bv_i\neq\zero$, so, as anticipated in the opening of this
section, the problem lies in the well-posed nonzero-right-hand-side regime:
we report the $\|\cdot\|_{\rho_T}$ error of $\hat\phi^E$ directly, with no
scaling recovery.

As in the previous cases, the two panels of Figure~\ref{fig:fish-ring_combined} verify that $\hat\phi^E$ preserves the observed milling state and reproduces the trajectories from fresh initial conditions. The learned kernel tracks $\phi^E$ on $\supp{\rho_T}$, which here is the narrow band of pairwise distances realized on the ring (Figure~\ref{fig:fish-ring_phi_comp}); Table~\ref{tab:fish-millings} reports a small $\|\cdot\|_{\rho_T}$ error and a near-unit $I_{\mathrm{mill}}$, and noise robustness is in the Appendix (Figure~\ref{fig:fish-ring_noise}). The recovery rests on the rigid-rotation reconstruction, which presumes a single common center: all agents orbit $\bar\bx$. 
\subsubsection{Fish Milling: Double Milling}\label{sec:double_milling}
Keeping all parameters as in Section~\ref{sec:ring_milling} except $C_r=0.6$, the small increase from $0.5$ shifts the system into a double mill: the agents split into two interleaved, counter-rotating subgroups about a common center, each at squared speed $\|\bv_i\|^2=\alpha/\beta$ with rigid relative positions (Figure~\ref{fig:double-milling_combined}). Two opposed circulations are now superposed on the same annulus, so the velocity field is no longer a single coherent rotation: an agent's motion is not a uniform circular orbit about $\bar\bx$, the rigid-rotation reconstruction~\eqref{eq:centripetal} of the ring case no longer holds, and we take the accelerations $\dot\bv_i$ as observed alongside positions and velocities. Recovering the local rotational structure from $\{\bx_i,\bv_i\}$ is a separate geometric problem, outside the scope of the kernel learning studied here. With $\dot\bv_i\neq\zero$, the right-hand side of~\eqref{eq:loss_2nd}
is again nonzero, so the problem lies in the well-posed
nonzero-right-hand-side regime. We report the
$\|\cdot\|_{\rho_T}$ error of $\hat\phi^E$ directly, using a degree-$1$
B-spline basis with $P^{\mathrm{tol}}=0.01$.
\begin{figure}[H]
    \centering
    \begin{subfigure}[t]{0.46\linewidth}
        \centering
        \includegraphics[width=\linewidth]{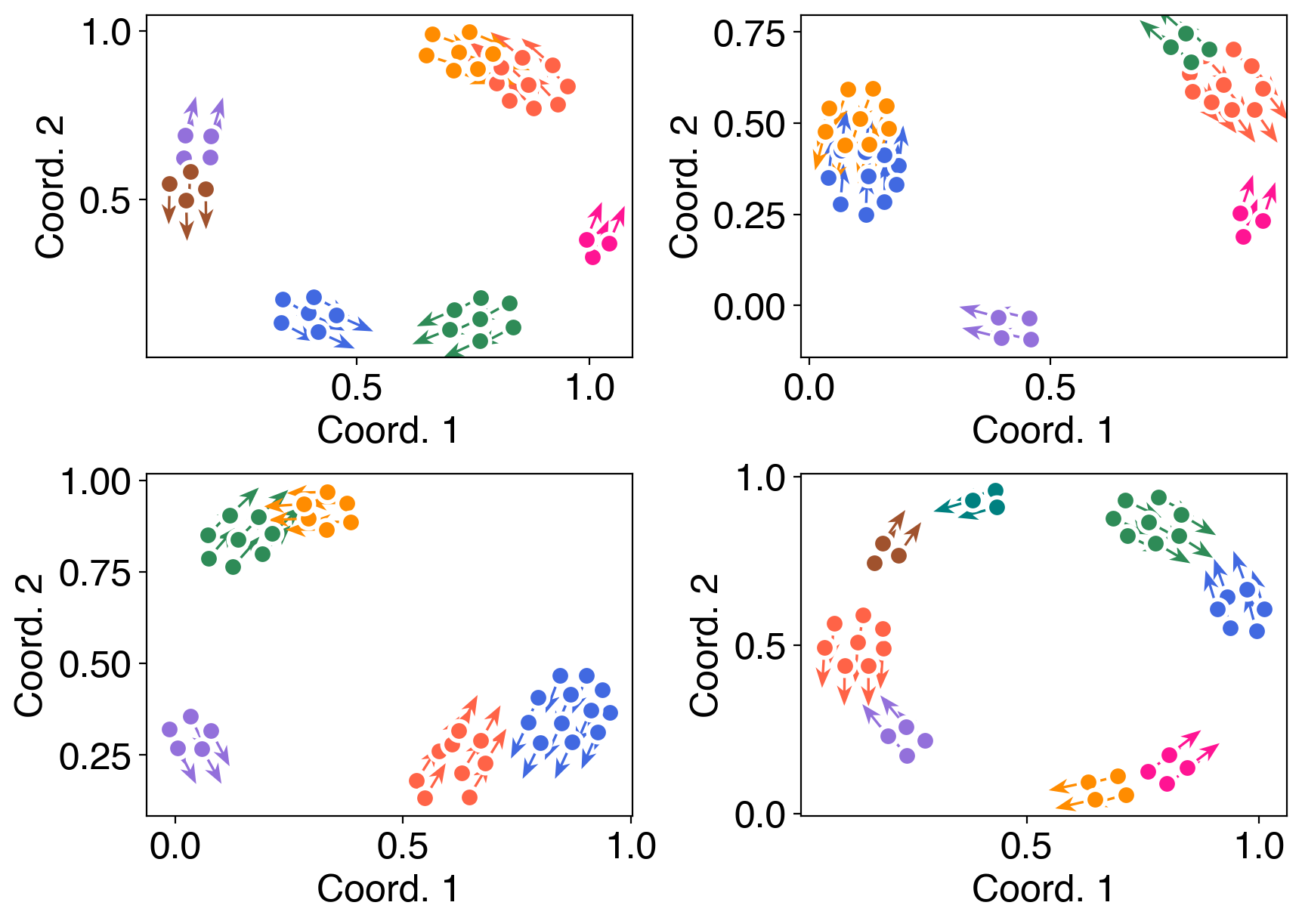}
        \caption{Double mill evolved under $\hat\phi^E$.}
        \label{fig:double-mill_traj_comp}
    \end{subfigure}
    \hfill
    \begin{subfigure}[t]{0.46\linewidth}
        \centering
        \includegraphics[width=\linewidth]{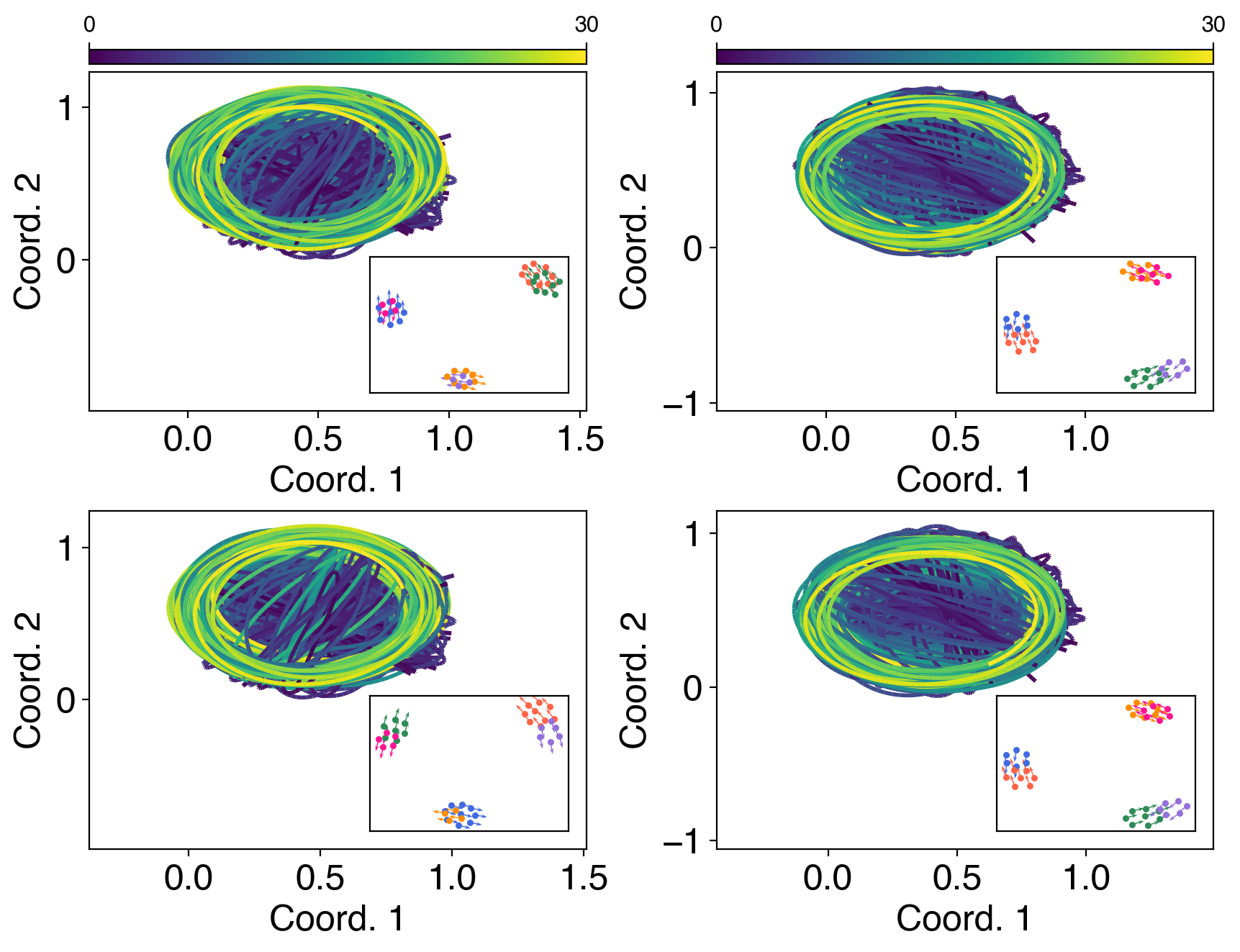}
        \caption{True (top) vs.\ learned $\hat\phi^E$ (bottom).}
        \label{fig:fish-mill-2_ic_traj_comp}
    \end{subfigure}
    \caption{Double milling: (left) observed double mill evolved under $\hat\phi^E$; (right) trajectories from two initial conditions, true (top) vs. learned (bottom), final configurations inset.}
    \label{fig:double-milling_combined}
\end{figure}
As in the previous cases, the two panels of Figure~\ref{fig:double-milling_combined} verify that $\hat\phi^E$ preserves the observed double mill and reproduces the trajectories from fresh initial conditions, with the kernel comparison in Figure~\ref{fig:double-mill_phi_comp}. Here, however, kernel accuracy and pattern fidelity come apart. Table~\ref{tab:fish-millings} shows that the $\|\cdot\|_{\rho_T}$ error is three orders of magnitude larger than for the single ring, driven by the broader, multi-modal $\rho_T$ of the two-subgroup geometry, yet $I_{\mathrm{mill}}$ remains near unity. The milling score measures only rotational coherence; it cannot tell whether the learned flow preserves one ring or two interleaved counter-rotating subgroups. A larger kernel error therefore does not by itself mean the emergent pattern is lost, and a scalar score cannot settle the question. We need diagnostics sensitive to the \emph{topology} of the pattern.

To this end we assign each realization $m$ a signature $\bs^m=(s_1^m,s_2^m,s_3^m)$ probing three successive levels of the collective steady state geometry, computed only for post-training assessment and never used in learning. The type indicator $s_1^m\in\{0,1\}$ distinguishes single from double milling, from the distribution of signed angular velocities and the total angular momentum at steady state. The count $s_2^m\in\mathbb{N}$ is the number of substructures detected at time $T$ by $K$-means on the normalized state $(\bx_i,\tilde\bv_i)$ with $\tilde\bv_i=\bv_i/\|\bv_i\|$. The persistence $s_3^m\in[-1,1]$ is the average silhouette score at $T$, $1.5T$, and $2T$ when the grouping inferred at $T$ is held fixed; a value near $1$ means that grouping remains a faithful description at later times, so the learned dynamics sustains the structure rather than matching it only at the observation time. Evolving the learned dynamics yields $\hat\bs^m$, and the agreement between $\bs^m$ and $\hat\bs^m$ measures how faithfully the topology is reproduced.

The two levels of recovery are reported separately. At the kernel level, Table~\ref{tab:fish-millings} gives the $\|\cdot\|_{\rho_T}$ error and the near-unit $I_{\mathrm{mill}}$. At the topology level, the type $s_1$ is recovered with over $99\%$ accuracy (Table~\ref{tab:s1_metrics}) and the persistence $s_3$ to within $10^{-3}$ on average, while the count $s_2$ is matched in roughly $81\%$ of realizations (Table~\ref{tab:fish_milling_s23}); $s_2$ is the strictest discrete diagnostic and the most sensitive to borderline assignments, which is why its agreement is lower while $s_3$ stays accurate. Despite a kernel error far larger than in the single-ring case, the learned $\hat\phi^E$ thus reproduces not only the rotational coherence of the flow but the topology of the double mill: in the quasi-static regime the framework recovers both the kernel and the collective geometry it induces.
\begin{table}[H]
\centering
\begin{tabular}{l c c c c}
\toprule
 & Accuracy & Precision & Recall & F1 score \\
\midrule
Mean $\pm$ Std
& $0.9924 \pm 0.0030$
& $0.9969 \pm 0.0025$
& $0.9955 \pm 0.0030$
& $0.9962 \pm 0.0015$ \\
\bottomrule
\end{tabular}
\caption{Topology recovery, type indicator: classification of $\hat s_1$ against $s_1$ (single vs.\ double milling), over $10$ trials.}
\label{tab:s1_metrics}
\end{table}
\begin{table}[H]
\centering
\begin{tabular}{l c c c c}
\toprule
 & Mis. rate of $s_2$ & $s_3$ & $\hat s_3$ & $|s_3 - \hat s_3|$ \\
\midrule
Mean $\pm$ Std
& $0.1942 \pm 0.0221$
& $0.8792 \pm 0.0004$
& $0.8785 \pm 0.0009$
& $0.0006 \pm 0.0007$ \\
\bottomrule
\end{tabular}
\caption{Topology recovery, count and persistence: misclassification rate $\mathbb{E}[\mathbf{1}(s_2\neq\hat s_2)]$ and silhouette score $s_3$ vs. $\hat s_3$, over $10$ trials.}
\label{tab:fish_milling_s23}
\end{table}
\section{Conclusion}\label{sec:conclude}
We have studied how to recover interaction kernels from single snapshots of collective steady state patterns rather than from full trajectories. Such data are degenerate: without temporal information, and with a right-hand side that may vanish at a steady state, the kernel need not be identifiable. Recasting the recovery as a $\rho_T$-weighted variational problem turns it into an eigenvalue problem whose solution, under the assumption of a suitable identifiability condition, estimates the kernel on the support of the observed pairwise-distance distribution.

What the examples show is that identifiability depends on the observed regime. The same model is well-posed when the learning equation carries a nonzero right-hand side, as in moving flocks and milling, but only identifiable up to a scaling when that side vanishes, as in static states and second-order flocking; the missing scale is then recovered from the first stopping time. In every case the kernel is constrained only on $\supp{\rho_T}$, so what can be learned is set by the geometry of that support, not by the form of the kernel. Where the data themselves fall short, structure can compensate: a shared potential recovers an alignment kernel that flocking data cannot identify, and a rigid-rotation constraint supplies unobserved accelerations. A larger kernel error need not destroy the emergent geometry, however, as the double mill shows, so kernel recovery should be checked against the pattern it generates.

Collective steady state patterns are therefore not just the end states of a dynamics; read through the right structure, they carry constraining information about the laws that produced them. Open questions remain on recovery under noise or partial observation, and on fixing the missing scale from other auxiliary information. The central mathematical question, however, is to characterize how the geometry of the terminal pairwise-distance distribution $\rho_T$ controls the null space and spectral gap of the observation operator, and hence which components of the interaction law are identifiable from collective patterns alone.
\section*{Acknowledgements}
BH developed, implemented, and analyzed the algorithms and data, MM and MZ designed the research, all authors wrote the manuscript.

Maggioni is partially supported by AFOSR award FA9550-23-1-0445. Zhong is partially supported by NSF-AoF grant $\#2225507$.
%
\appendix
\section{Numerical Implementation}\label{app:numerical-implementation}
We describe the finite-dimensional implementation of the scale-ambiguous learning problem arising from static first-order data and second-order flocking data. In these regimes, the observed interaction residual vanishes, and the kernel direction is recovered from a constrained homogeneous minimization problem.

Because the interaction equations decouple according to the receiving type, we solve one generalized eigenvalue problem for each $\idxcl\in\{1,\ldots,\numcl\}$.  For a fixed receiving type \(\idxcl\), define
\(
    \Hspace^E_{\idxcl}
    :=
    \bigoplus_{\idxcl'=1}^{\numcl}
    \Hspace^E_{\idxcl,\idxcl'},
    \;
    n_{\idxcl}
    :=
    \sum_{\idxcl'=1}^{\numcl}
    n_{\idxcl,\idxcl'},
\)
where
\(
    n_{\idxcl,\idxcl'}
    =
    \dim(\Hspace^E_{\idxcl,\idxcl'}).
\)
Let
\(
    \left\{
        \psi^{\idxcl,\idxcl'}_{\eta}
    \right\}_{\eta=1}^{n_{\idxcl,\idxcl'}}
    =
    \operatorname{basis}
    \bigl(
        \Hspace^E_{\idxcl,\idxcl'}
    \bigr)
\)
be the data-adaptive basis constructed by Algorithm~\ref{alg:adaptive-partition}. We write
\(
    \varphi^E_{\idxcl,\idxcl'}(r)
    =
    \sum_{\eta=1}^{n_{\idxcl,\idxcl'}}
    \alpha^{\idxcl,\idxcl'}_{\eta}
    \psi^{\idxcl,\idxcl'}_{\eta}(r)
\)
and collect the coefficients associated with the receiving type
\(\idxcl\) into
\[
    \balpha_{\idxcl}
    =
    \begin{bmatrix}
        \alpha^{\idxcl,1}_{1}
        &
        \cdots
        &
        \alpha^{\idxcl,1}_{n_{\idxcl,1}}
        &
        \cdots
        &
        \alpha^{\idxcl,\numcl}_{1}
        &
        \cdots
        &
        \alpha^{\idxcl,\numcl}_{n_{\idxcl,\numcl}}
    \end{bmatrix}^{\top}
    \in\R^{n_{\idxcl}}.
\]
For one observed configuration \(\bX^m\), the interaction residual for
receiving type \(\idxcl\) is linear in \(\balpha_{\idxcl}\):
\[
    \F^{\idxcl,E}_{\varphivec^E_{\idxcl}}(\bX^m)
    =
    \Psi^E_{\idxcl}(\bX^m)\balpha_{\idxcl},
\]
where
\(
    \Psi^E_{\idxcl}(\bX^m)
    \in
    \R^{N_{\idxcl}d\times n_{\idxcl}}
\)
is given by
\[
\begin{aligned}
    \Psi^E_{\idxcl}(\bX^m)
    =
    \Big[
    \F^{\idxcl,E}_{\psi^{\idxcl,1}_{1}}(\bX^m),
    \ldots,
    \F^{\idxcl,E}_{\psi^{\idxcl,1}_{n_{\idxcl,1}}}(\bX^m),
    \ldots,
    \F^{\idxcl,E}_{\psi^{\idxcl,\numcl}_{1}}(\bX^m),
    \ldots,
    \F^{\idxcl,E}_{\psi^{\idxcl,\numcl}_{n_{\idxcl,\numcl}}}(\bX^m)
    \Big].
\end{aligned}
\]
For a basis function
\(
    \psi^{\idxcl,\idxcl'}_{\eta}
\),
the corresponding column is
\[
    \F^{\idxcl,E}_{\psi^{\idxcl,\idxcl'}_{\eta}}(\bX^m)
    =
    \begin{bmatrix}
        \vdots
        \\[2pt]
        \displaystyle
        \sum_{\substack{
            j\in\cl_{\idxcl'}\\
            j\neq i
        }}
        \frac{
            \psi^{\idxcl,\idxcl'}_{\eta}(r^m_{ij})
        }{
            N_{\idxcl'}
        }
        \br^m_{ij}
        \\[4pt]
        \vdots
    \end{bmatrix}
    \in\R^{N_{\idxcl}d},
    \qquad
    i\in\cl_{\idxcl}.
\]
Stacking the \(M\) observed configurations gives
\[
    \Psi^E_{\idxcl,[1:M]}
    =
    \begin{bmatrix}
        \Psi^E_{\idxcl}(\bX^1)
        \\
        \vdots
        \\
        \Psi^E_{\idxcl}(\bX^M)
    \end{bmatrix}
    \in
    \R^{MN_{\idxcl}d\times n_{\idxcl}}.
\]
The type-weighted empirical residual norm is represented by
\(
    W_{\idxcl}
    :=
    \frac{1}{MN_{\idxcl}}
    \bm I_{MN_{\idxcl}d},
\)
and hence
\[
\begin{aligned}
    \frac{1}{MN_{\idxcl}}
    \sum_{m=1}^{M}
    \left\|
        \F^{\idxcl,E}_{\varphivec^E_{\idxcl}}(\bX^m)
    \right\|^2
    =
    \balpha_{\idxcl}^{\top}
    H_{\idxcl}
    \balpha_{\idxcl}, \quad
    H_{\idxcl}
    :=
    \bigl(
        \Psi^E_{\idxcl,[1:M]}
    \bigr)^{\top}
    W_{\idxcl}
    \Psi^E_{\idxcl,[1:M]}.
\end{aligned}
\]
The matrix \(H_{\idxcl}\) is symmetric and positive semidefinite.

We next represent the empirical
\(\rho^{M}_{T}\)-weighted kernel norm. For every ordered type pair
\((\idxcl,\idxcl')\), define
\[
    \bigl(
        G_{\idxcl,\idxcl'}
    \bigr)_{\eta q}
    :=
    \left\langle
        \psi^{\idxcl,\idxcl'}_{\eta},
        \psi^{\idxcl,\idxcl'}_{q}
    \right\rangle_{\rho^{M}_{T,\idxcl,\idxcl'}}.
\]
In the numerical implementation, the same matrix may equivalently be evaluated by quadrature using the histogram approximation of \(\rho^{M}_{T,\idxcl,\idxcl'}\).

The Gram matrix for receiving type \(\idxcl\) is the block-diagonal matrix
\[
    G_{\idxcl}
    :=
    \operatorname{diag}
    \left(
        G_{\idxcl,1},
        \ldots,
        G_{\idxcl,\numcl}
    \right)
    \in\R^{n_{\idxcl}\times n_{\idxcl}}.
\]
It satisfies
\[
    \balpha_{\idxcl}^{\top}
    G_{\idxcl}
    \balpha_{\idxcl}
    =
    \sum_{\idxcl'=1}^{\numcl}
    \left\|
        \varphi^E_{\idxcl,\idxcl'}
    \right\|_{\rho^{M}_{T,\idxcl,\idxcl'}}^2.
\]
The adaptive hypothesis space is chosen so that each retained basis direction is observable under \(\rho^{M}_{T}\). We therefore assume that \(G_{\idxcl}\) is positive definite on the resulting coefficient space. If \(G_{\idxcl}\) is only positive semidefinite, basis directions belonging to \(\ker G_{\idxcl}\) are not supported by the observed pairwise distances and are removed before solving the learning problem.

The constrained minimization problem for receiving type \(\idxcl\) is
\begin{equation}
\label{eq:finite-dimensional-CS}
    \min_{\balpha_{\idxcl}\in\R^{n_{\idxcl}}}
    \balpha_{\idxcl}^{\top}
    H_{\idxcl}
    \balpha_{\idxcl}
    \qquad
    \text{subject to}
    \qquad
    \balpha_{\idxcl}^{\top}
    G_{\idxcl}
    \balpha_{\idxcl}
    =
    1.
\end{equation}
Equivalently, we solve the symmetric generalized eigenvalue problem
\begin{equation}
\label{eq:generalized-eigenproblem}
    H_{\idxcl}\balpha_{\idxcl,j}
    =
    \lambda_{\idxcl,j}
    G_{\idxcl}\balpha_{\idxcl,j},
    \qquad
    j=1,\ldots,n_{\idxcl},
\end{equation}
with
\(
    0
    \leq
    \lambda_{\idxcl,1}
    \leq
    \lambda_{\idxcl,2}
    \leq
    \cdots
    \leq
    \lambda_{\idxcl,n_{\idxcl}},
\)
and the generalized eigenvectors chosen \(G_{\idxcl}\)-orthonormally:
\[
    \balpha_{\idxcl,i}^{\top}
    G_{\idxcl}
    \balpha_{\idxcl,j}
    =
    \delta_{ij}.
\]
The coefficient vector associated with the smallest generalized eigenvalue,
\(
    \widehat{\balpha}_{\idxcl}
    :=
    \balpha_{\idxcl,1},
\)
gives the normalized estimator
\[
    \widehat{\varphi}^E_{\idxcl,\idxcl'}(r)
    =
    \sum_{\eta=1}^{n_{\idxcl,\idxcl'}}
    \widehat{\alpha}^{\idxcl,\idxcl'}_{\eta}
    \psi^{\idxcl,\idxcl'}_{\eta}(r),
    \qquad
    \idxcl'=1,\ldots,\numcl.
\]

The smallest generalized eigenvalue
\(
    \lambda_{\idxcl,1}
    =
    \min_{\balpha_{\idxcl}\neq\zero}
    \frac{
        \balpha_{\idxcl}^{\top}
        H_{\idxcl}
        \balpha_{\idxcl}
    }{
        \balpha_{\idxcl}^{\top}
        G_{\idxcl}
        \balpha_{\idxcl}
    }
\)
measures the minimum normalized equilibrium residual. When the true kernel
block belongs to the hypothesis space and the observed configurations are
exact equilibria, one expects
\(
    \lambda_{\idxcl,1}=0
\)
up to numerical error. The generalized spectral gap
\(
    \gamma_{\idxcl}
    :=
    \lambda_{\idxcl,2}
    -
    \lambda_{\idxcl,1}
\)
measures the separation between the minimizing kernel direction and the
nearest competing direction. If
\(
    \gamma_{\idxcl}>0,
\)
the minimizing direction is unique up to sign, and its perturbation
sensitivity is controlled by
Proposition~\ref{prop:spectral-gap}.

The eigenvalues and the gap are associated with a fixed data set and its
corresponding data-adaptive hypothesis space. They are therefore
within-problem stability indicators: they need not vary monotonically with
the number of agents, and eigenvalue magnitudes obtained from different
hypothesis spaces should not be interpreted as intrinsic condition numbers
that are directly comparable across different examples.

Finally, the generalized eigenvector determines only the normalized kernel
direction. Its physical sign is selected by the minimal-energy criterion
described in Section~\ref{sec:min_energy}, and the remaining positive
receiving-type scaling is recovered by the stopping-time procedure in
Algorithm~\ref{alg:timescale}. For homogeneous systems,
\(\numcl=1\), the construction reduces to a single generalized eigenvalue
problem.
\section{Additional Examples}
\begin{figure}[H]
    \centering
    \begin{subfigure}[t]{0.46\linewidth}
        \centering
        \includegraphics[width=\linewidth]{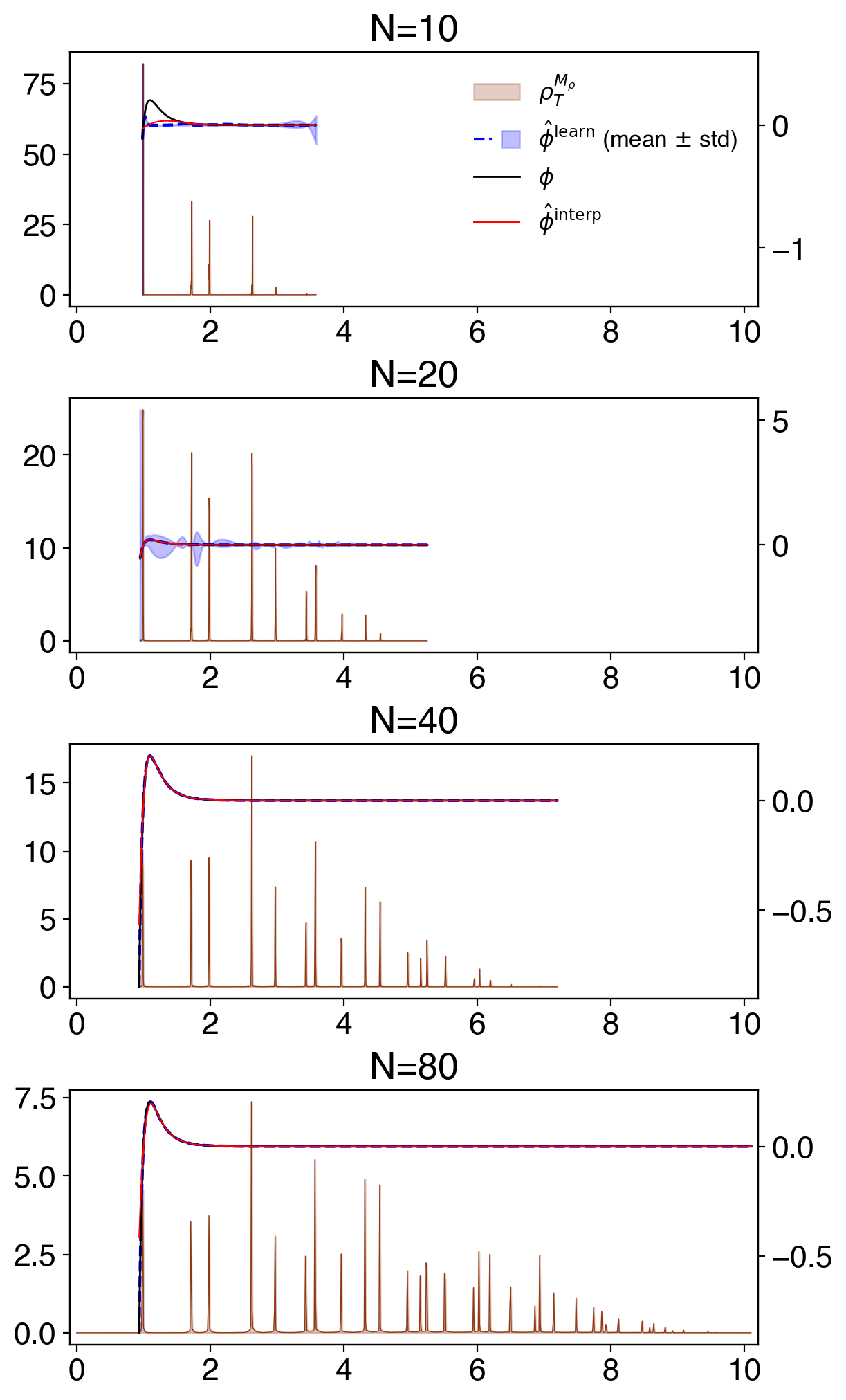}
           \caption{LJ}
    \label{fig:LJ_rhoT}
    \end{subfigure}
    \hfill 
    \begin{subfigure}[t]{0.46\linewidth}
        \centering
        \includegraphics[width=\linewidth]{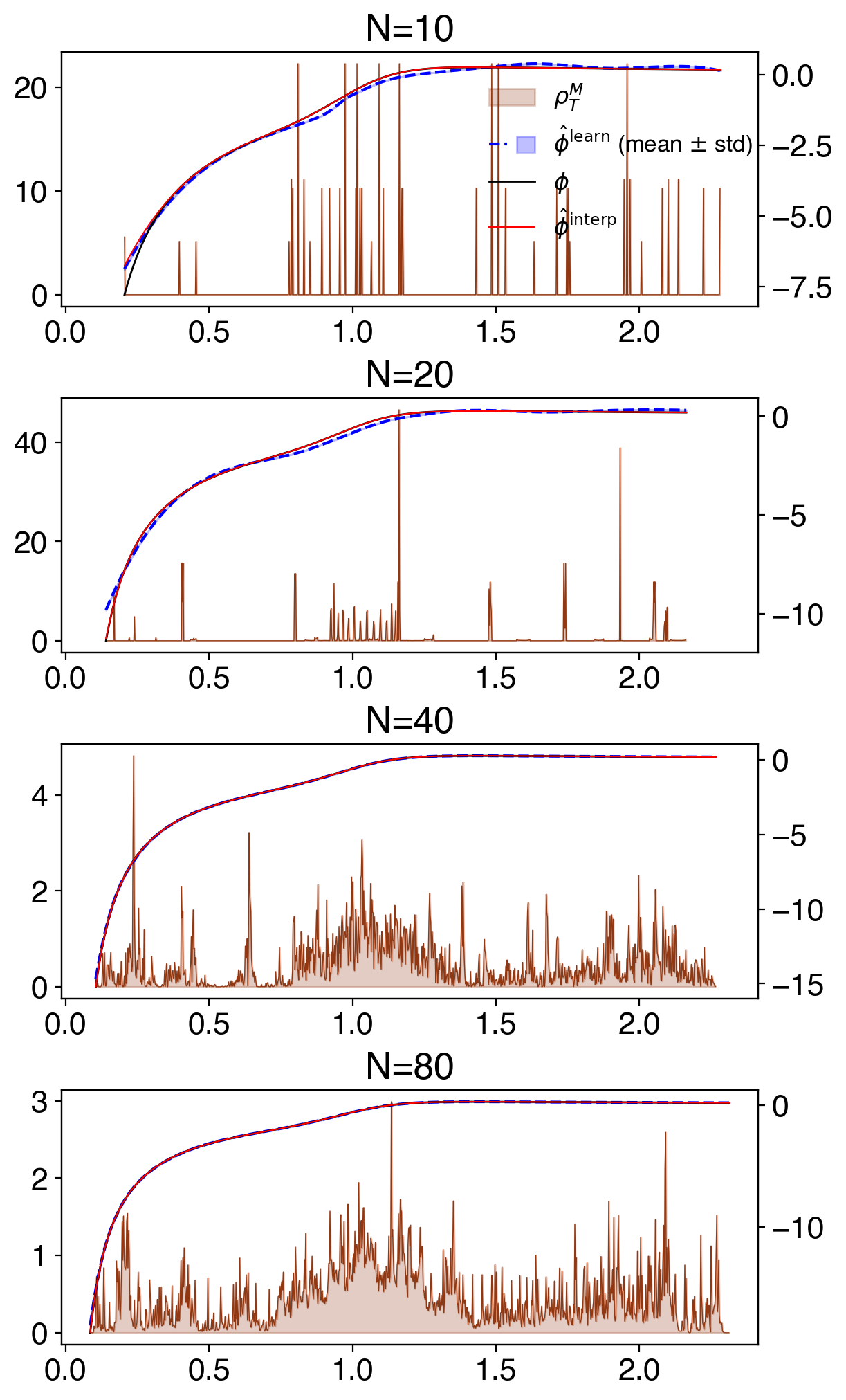}
        \caption{Tanh}
        \label{fig:tanh_rhoT}
    \end{subfigure}
    \caption{ Learning of $\phi$ with $N \in \{10, 20, 40, 80\}$, compared 
   to B-spline interpolation $\hat\phi^{\mathrm{interp}}$. Background shows 
   $\rho_T^{M}$.}
\end{figure}

\begin{figure}[H]
    \centering
    \begin{subfigure}[t]{0.46\linewidth}
        \centering
        \includegraphics[width=\linewidth]{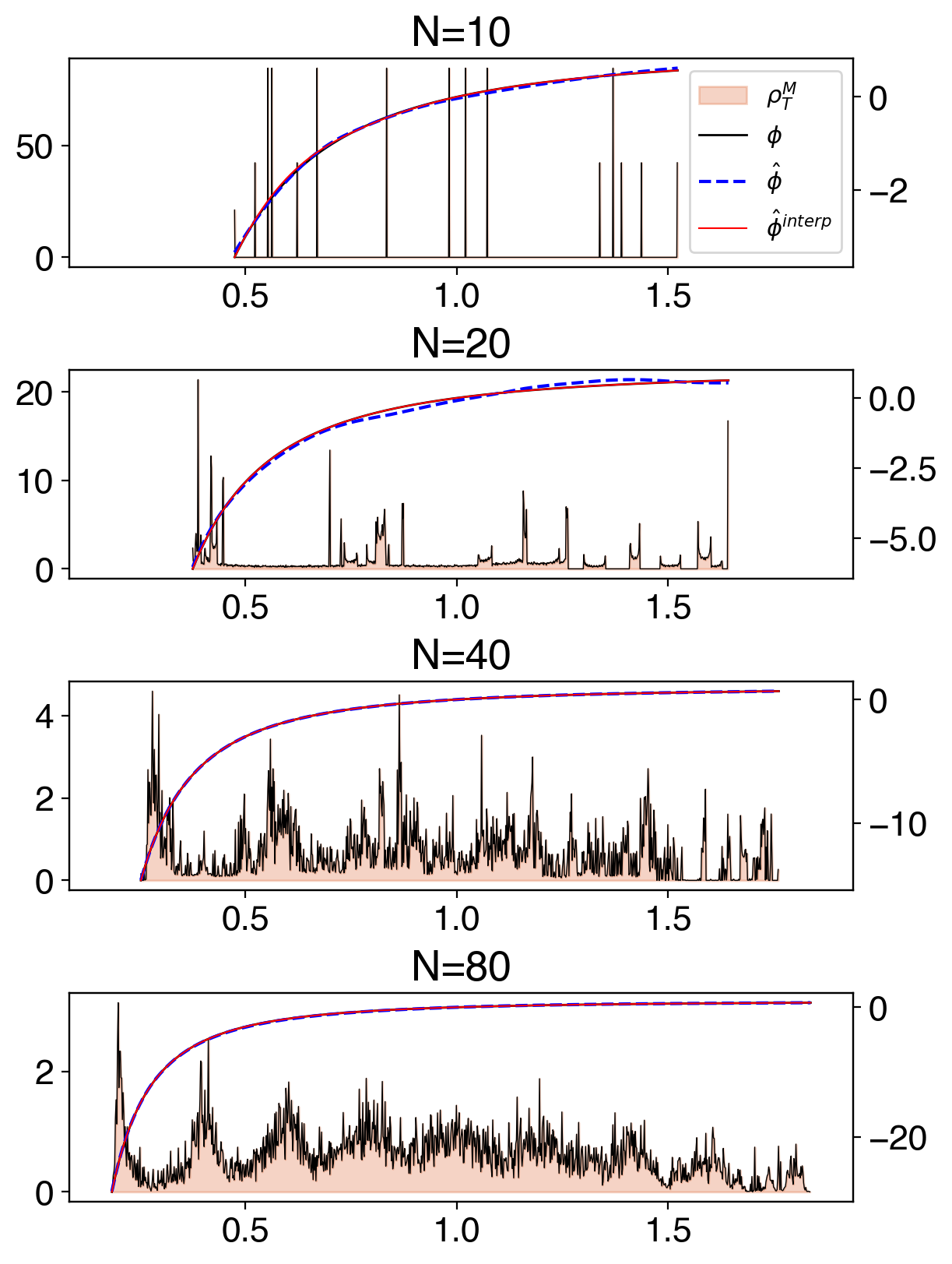}
           \caption{Anticipated flocking}
    \label{fig:anticipate_rhoT}
    \end{subfigure}
    \hfill 
    \begin{subfigure}[t]{0.46\linewidth}
        \centering
        \includegraphics[width=\linewidth]{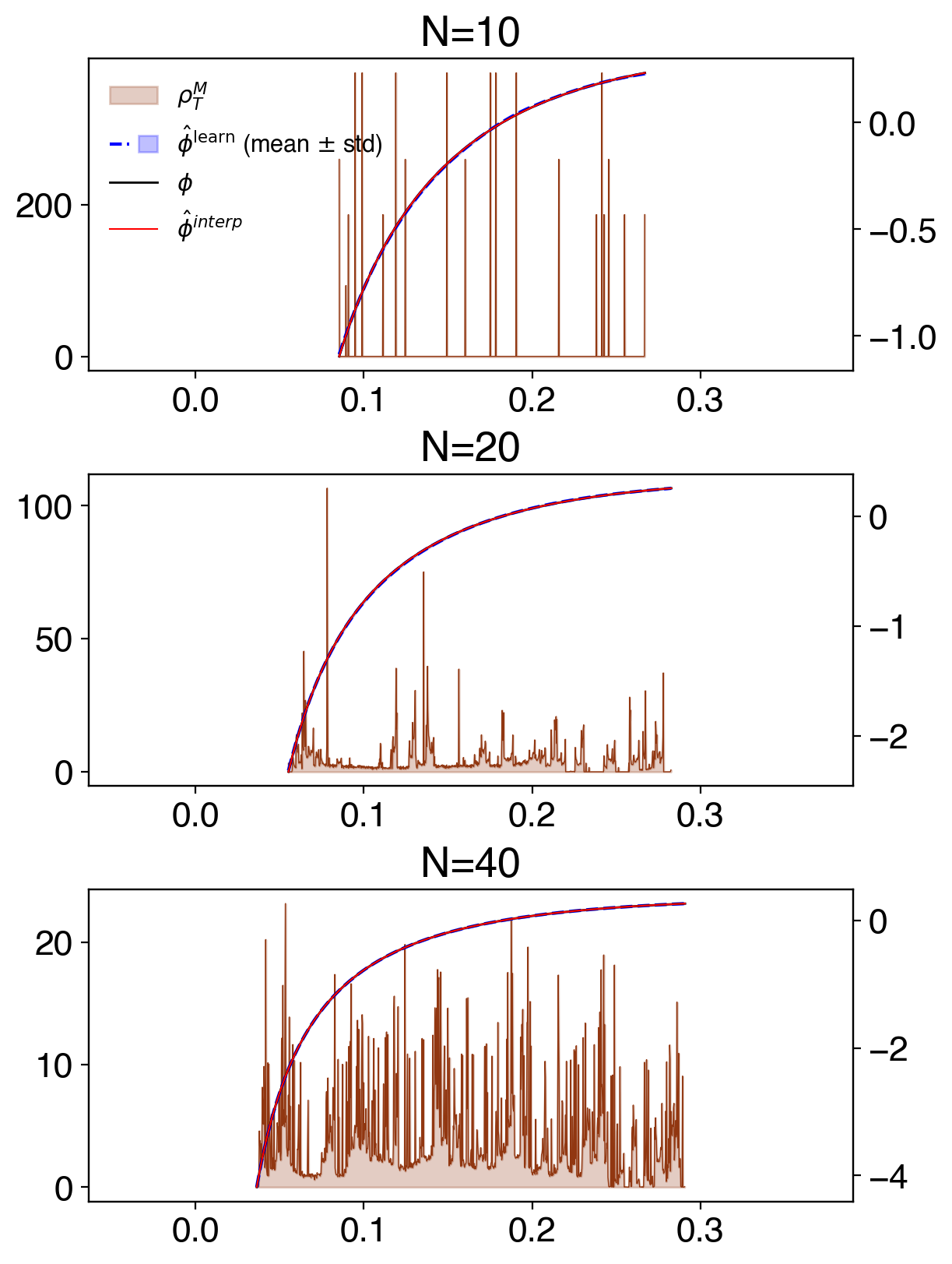}
        \caption{Fish flocking}
        \label{fig:fish-flock_rhoT}
    \end{subfigure}
    \caption{ Learning of $\phi$ with $N \in \{10, 20, 40, 80\}$, compared 
   to B-spline interpolation $\hat\phi^{\mathrm{interp}}$. Background shows 
   $\rho_T^{M}$.}
\end{figure}

\begin{figure}[H]
    \centering
    \includegraphics[width=\linewidth]{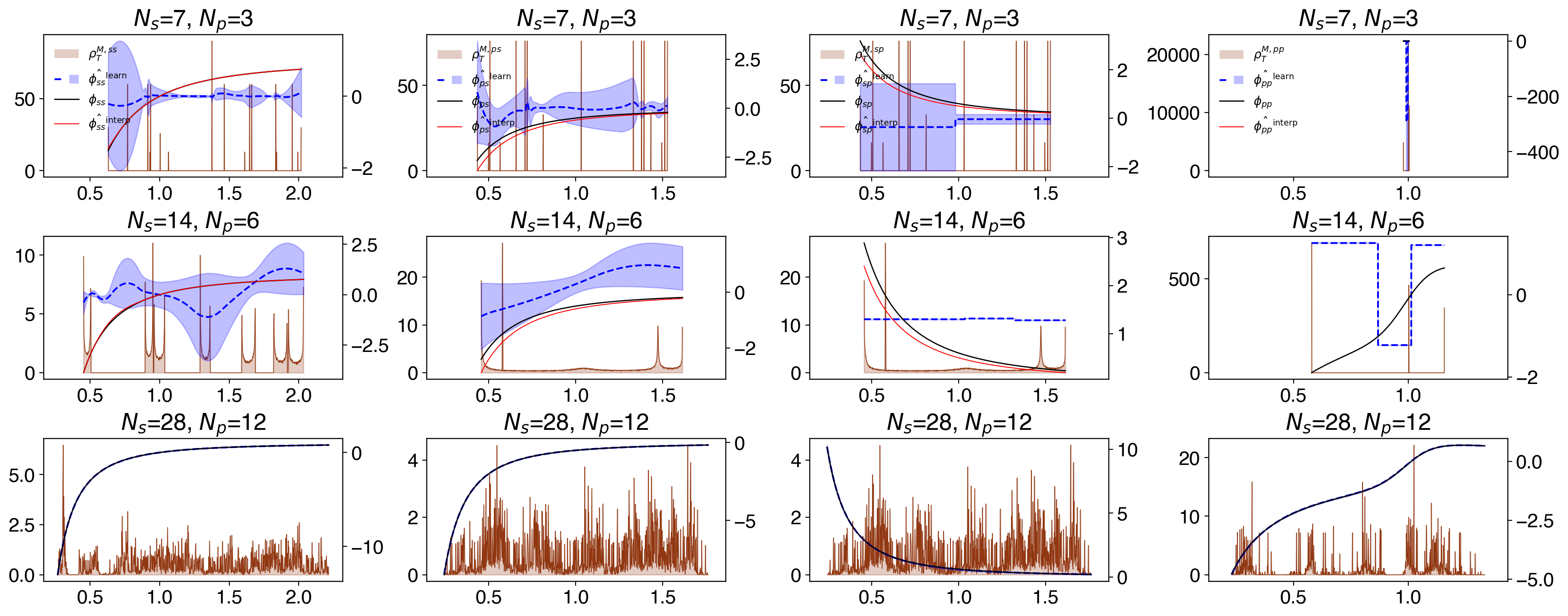}
    \caption{Predator--prey (static): learning of $\phi_{\idxcl_1, \idxcl_2}$ for 
    $(N, N_s, N_p) \in \{(10, 7, 3),\ (20, 14, 6),\ (40, 28, 12)\}$, compared 
    to B-spline interpolation $\hat\phi^{\mathrm{interp}}_{\idxcl_1, \idxcl_2}$. 
    Background shows $\rho_T^{M, \idxcl_1, \idxcl_2}$.}
    \label{fig:MS-static_rhoT}
\end{figure}

\begin{figure}[H]
    \centering
    \includegraphics[width=\linewidth]{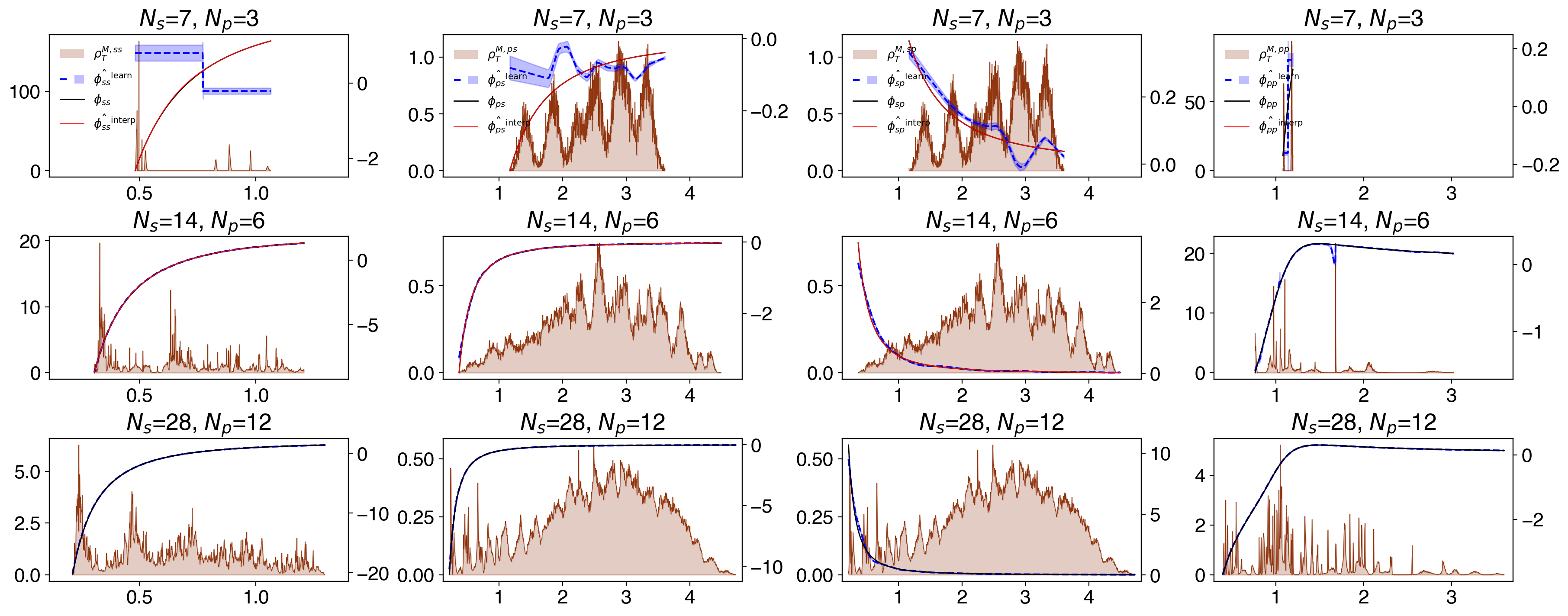}
    \caption{Predator--prey (flocking): learning of $\phi_{\idxcl_1, \idxcl_2}$ for 
    $(N, N_s, N_p) \in \{(10, 7, 3),\ (20, 14, 6),\ (40, 28, 12)\}$, compared 
    to B-spline interpolation $\hat\phi^{\mathrm{interp}}_{\idxcl_1, \idxcl_2}$. 
    Background shows $\rho_T^{M, \idxcl_1, \idxcl_2}$.}
    \label{fig:MS-flock_rhoT}
\end{figure}

\begin{figure}[H]
    \centering
    \begin{subfigure}[t]{0.46\textwidth}
        \centering
        \includegraphics[width=\linewidth]{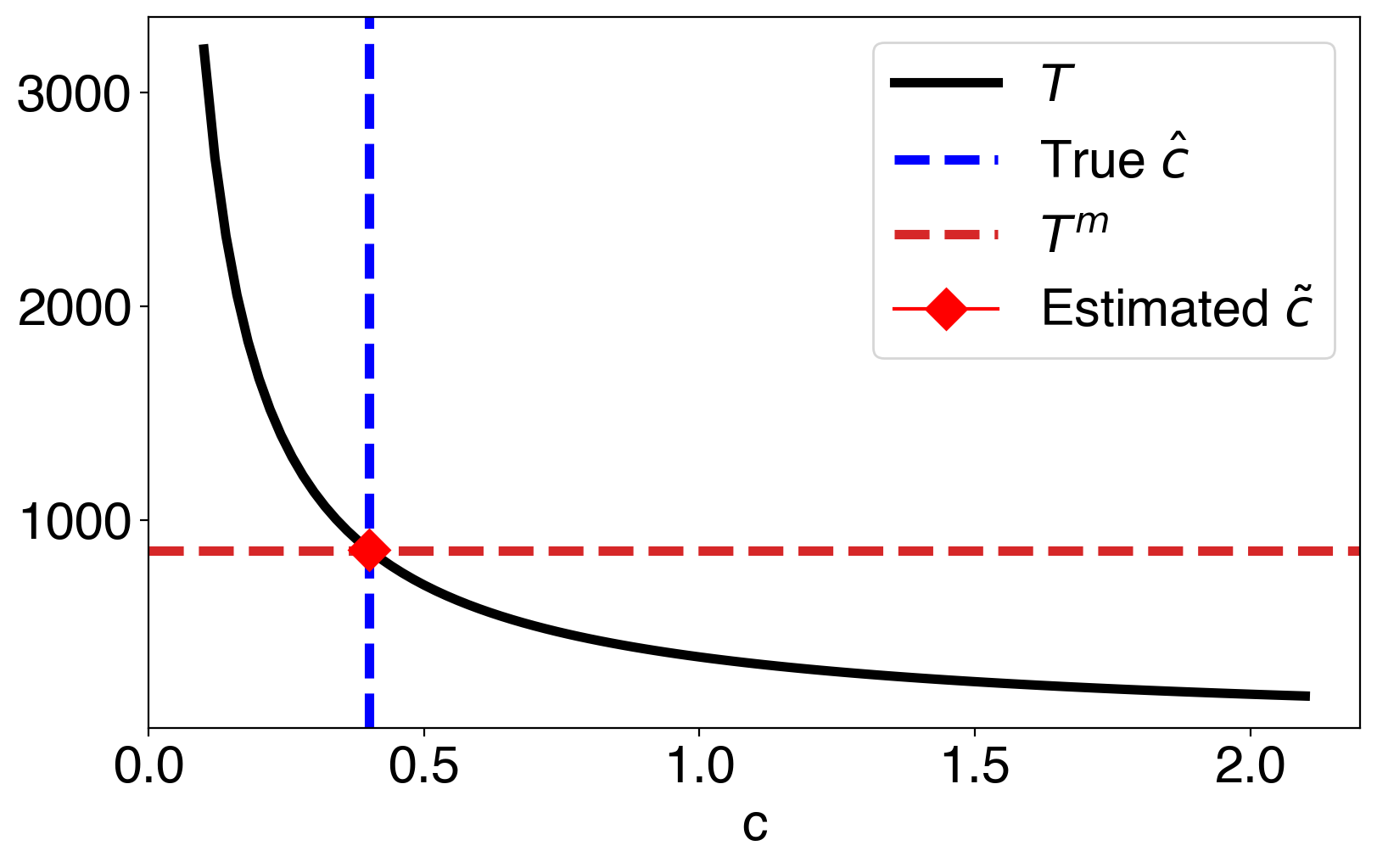}
        \caption{Estimated optimal $c_*$.}
        \label{fig:LJ_C_comp}
    \end{subfigure}%
    \hfill
    \begin{subfigure}[t]{0.46\textwidth}
        \centering
        \includegraphics[width=\linewidth]{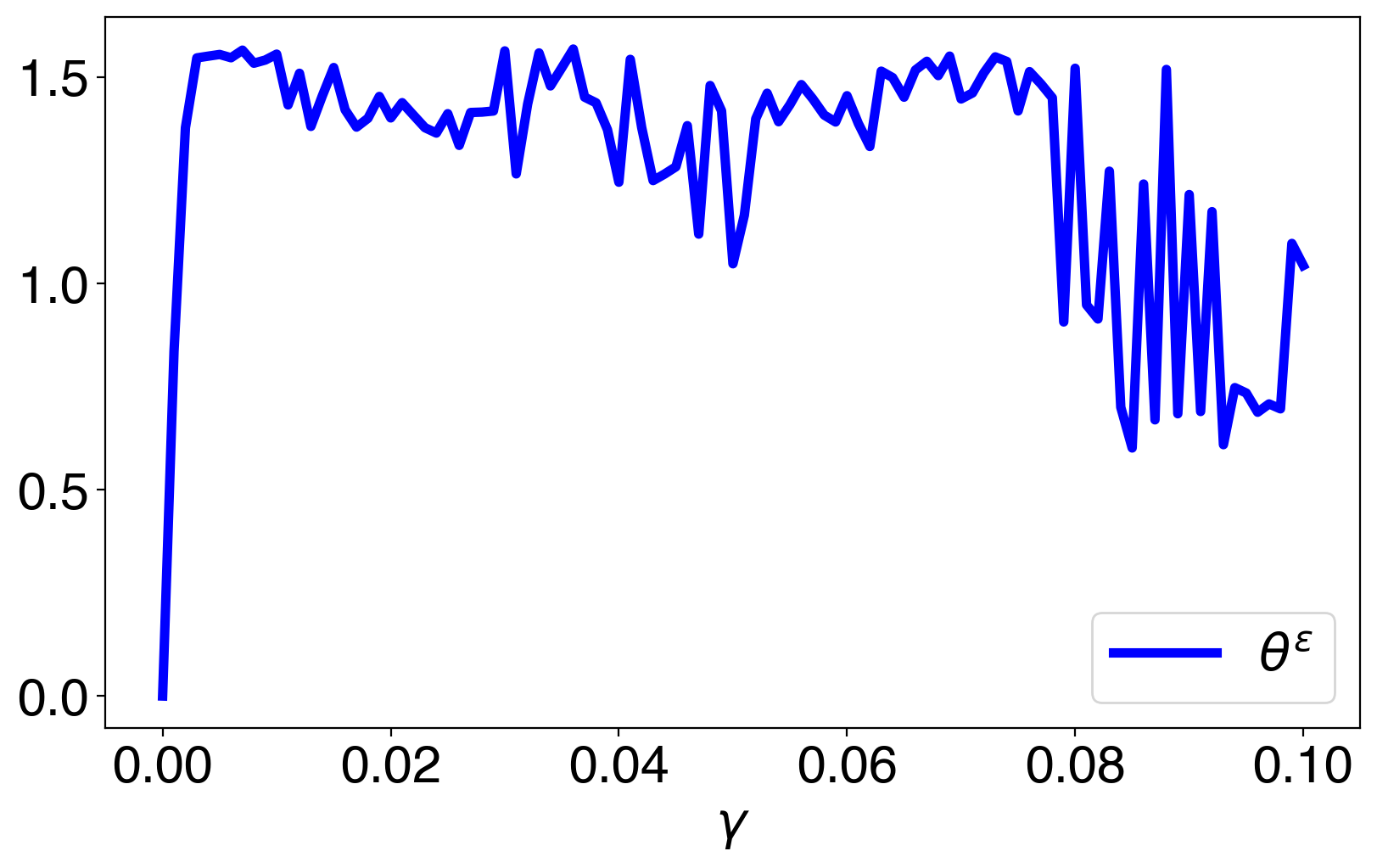}
        \caption{Noise effects in learning.}
        \label{fig:LJ_noise}
    \end{subfigure}
    \caption{ Additional learning from $LJ$-crystal steady patterns.}
\end{figure}

\begin{figure}[H]
    \centering
    \includegraphics[width = \linewidth]{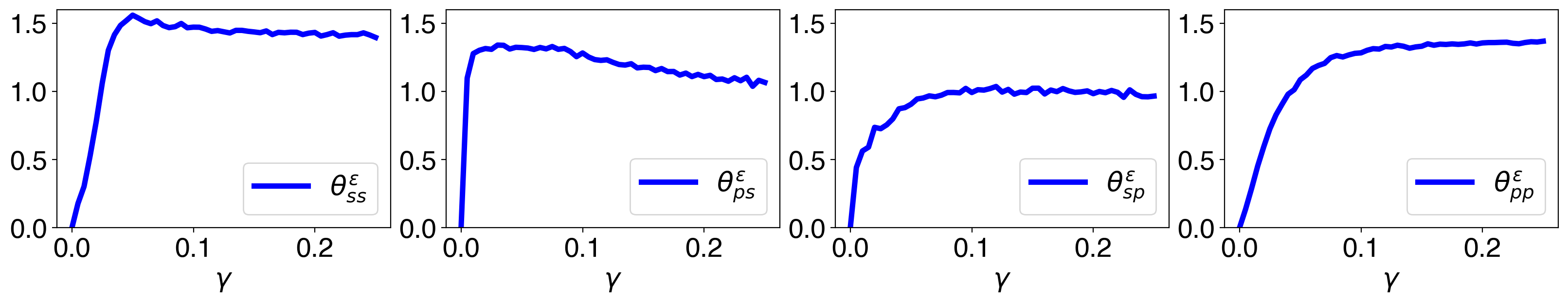}
    \caption{Predator--prey (flocking): Noise effects in learning.}
    \label{fig:MS-flock_noise}
\end{figure}

\begin{figure}[H]
    \centering
    \begin{subfigure}[t]{0.46\textwidth}
        \centering
        \includegraphics[width=\linewidth]{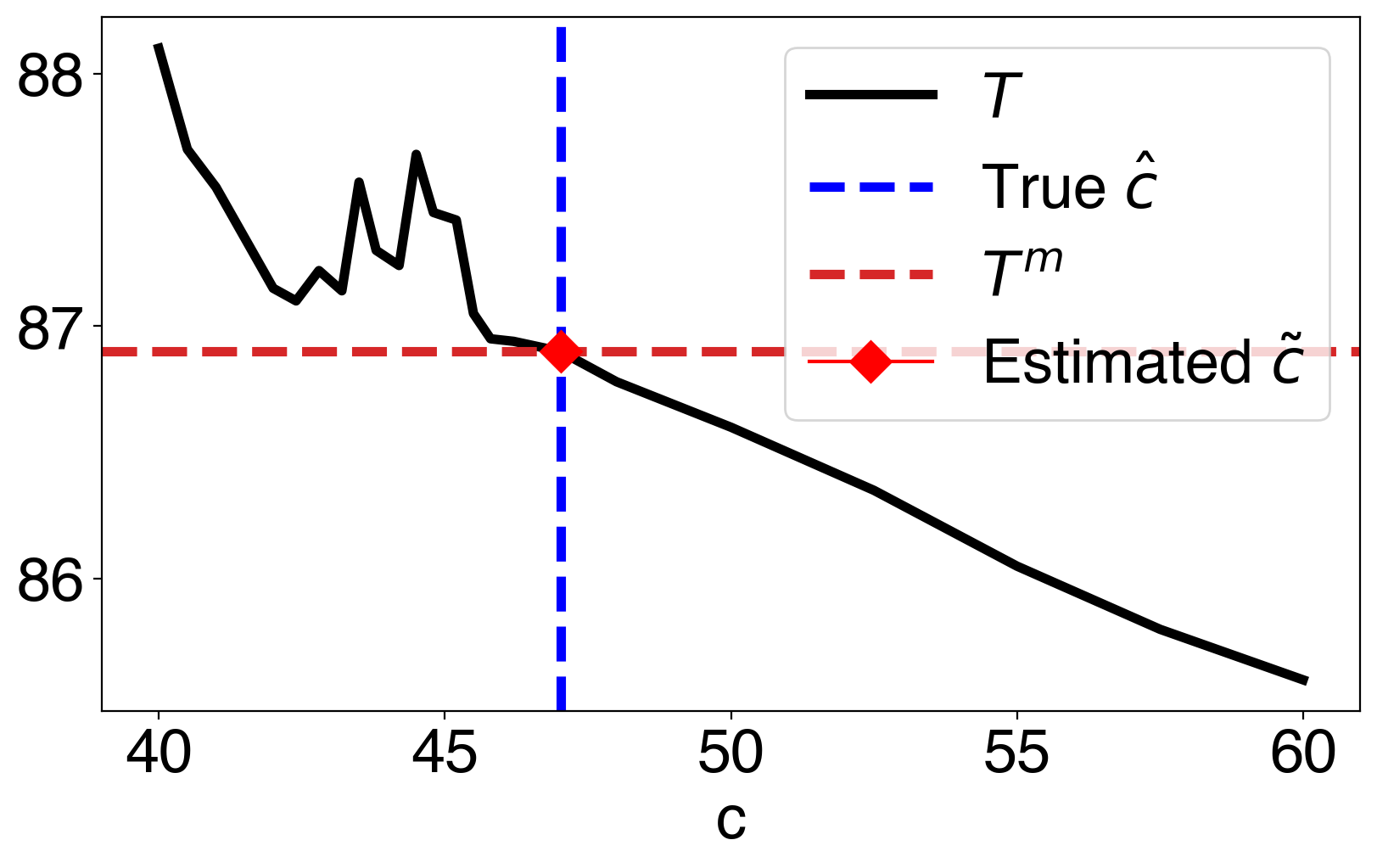}
        \caption{Estimated optimal $c_*$.}
        \label{fig:anticipate_C_comp}
    \end{subfigure}%
    \hfill
    \begin{subfigure}[t]{0.46\textwidth}
        \centering
        \includegraphics[width=\linewidth]{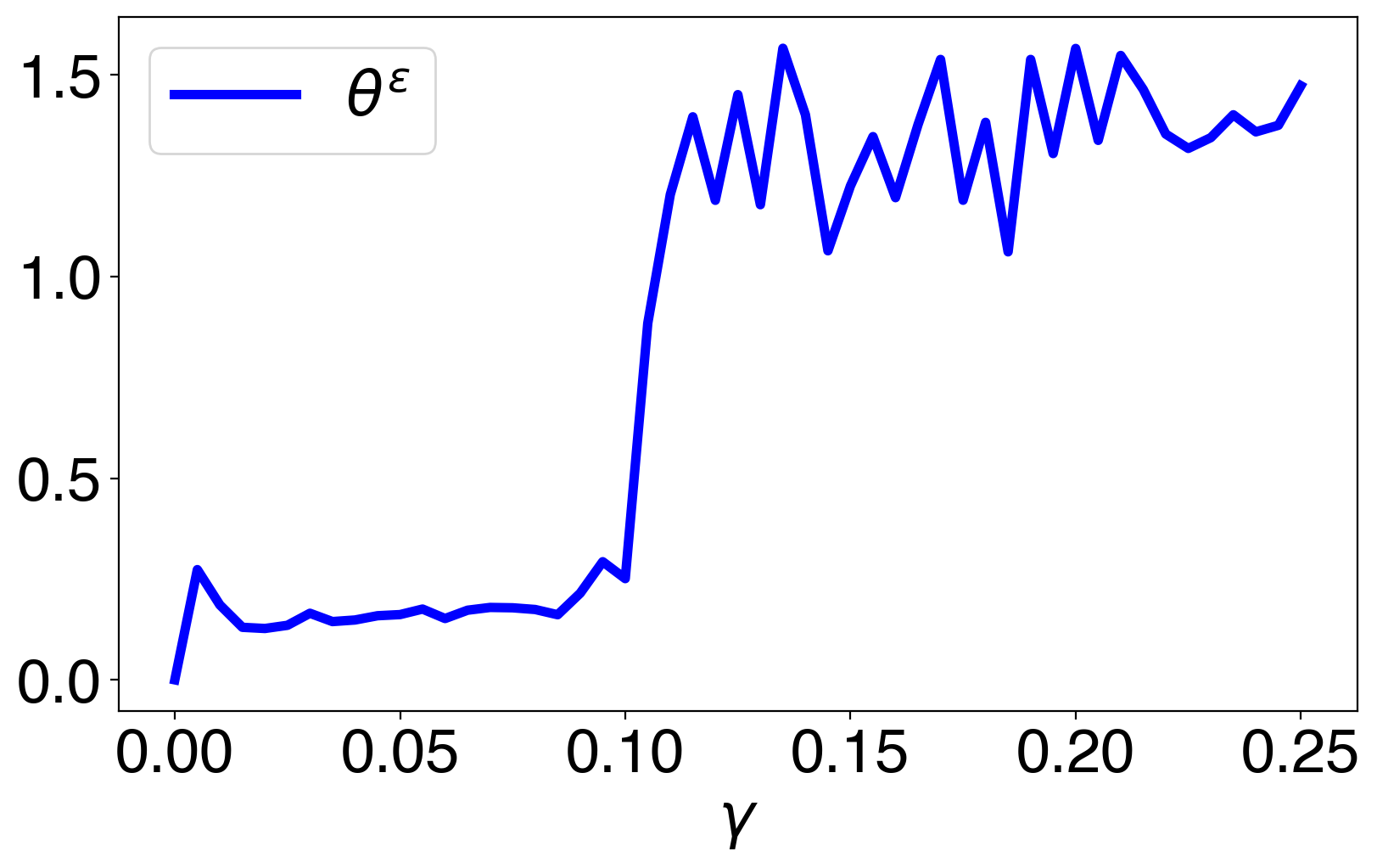}
        \caption{Noise effects in learning.}
        \label{fig:anticipate_noise}
    \end{subfigure}
    \caption{Additional learning from anticipated flocking steady patterns.}
\end{figure}

\begin{figure}[H]
    \centering
    \begin{subfigure}[t]{0.46\textwidth}
        \centering
        \includegraphics[width=\linewidth]{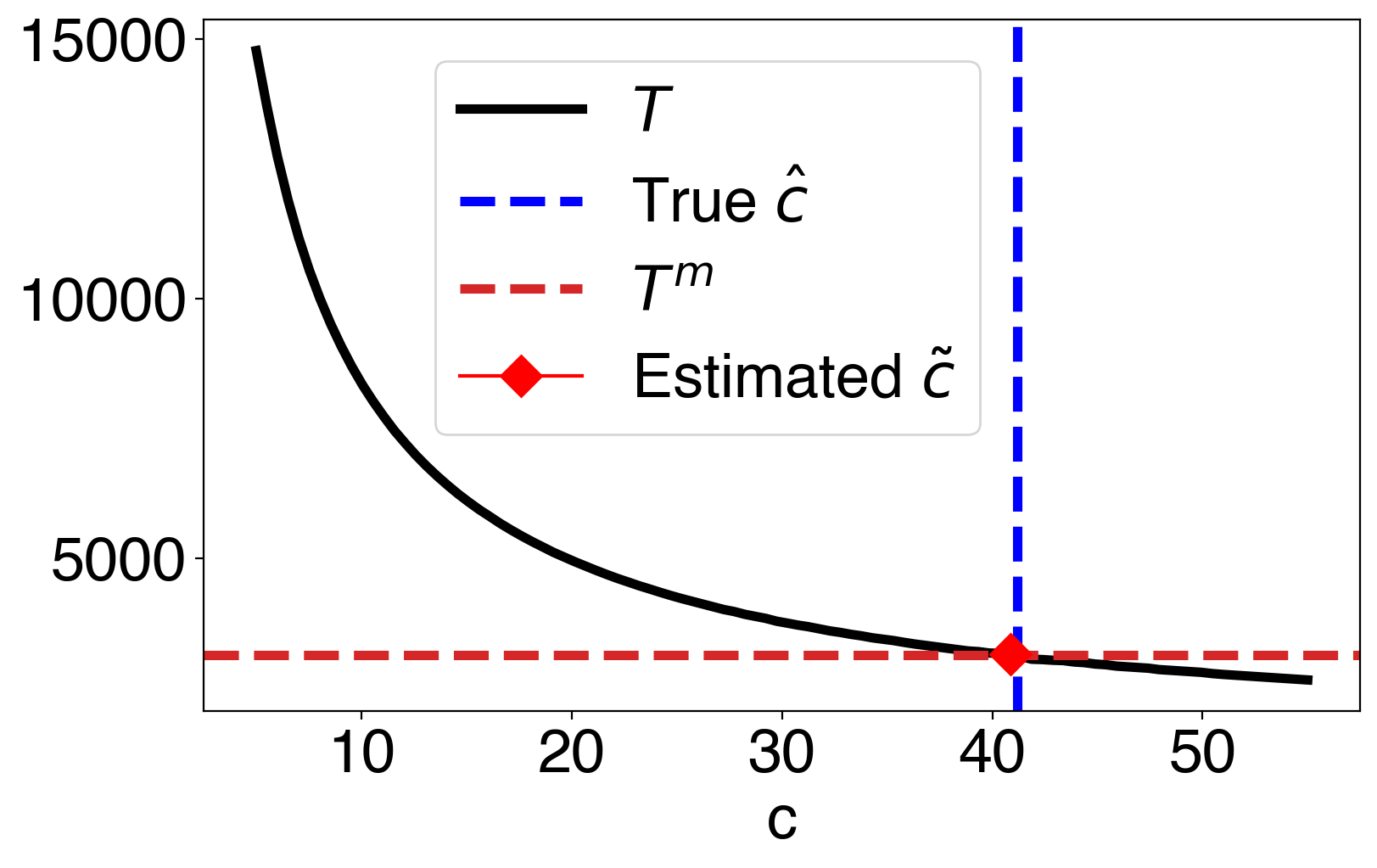}
        \caption{Estimated optimal $c_*$.}
        \label{fig:tanh_C_comp}
    \end{subfigure}%
    \hfill
    \begin{subfigure}[t]{0.46\textwidth}
        \centering
        \includegraphics[width=\linewidth]{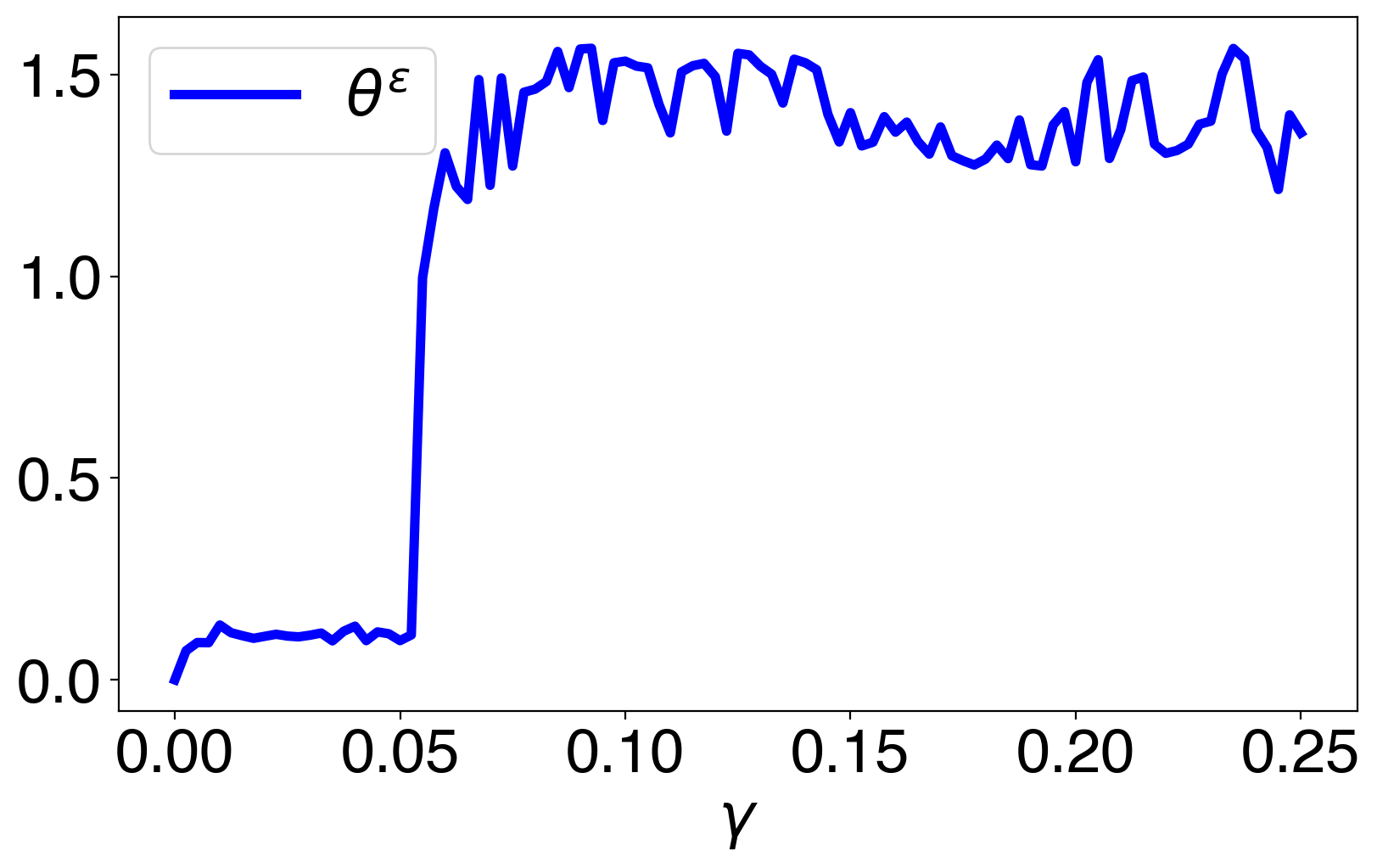}
        \caption{Noise effects in learning.}
        \label{fig:tanh_noise}
    \end{subfigure}
    \caption{Additional learning from  soccer ball steady patterns.}
\end{figure}


\begin{figure}[H]
    \centering
    \begin{subfigure}[t]{0.46\textwidth}
        \centering
        \includegraphics[width = \linewidth]{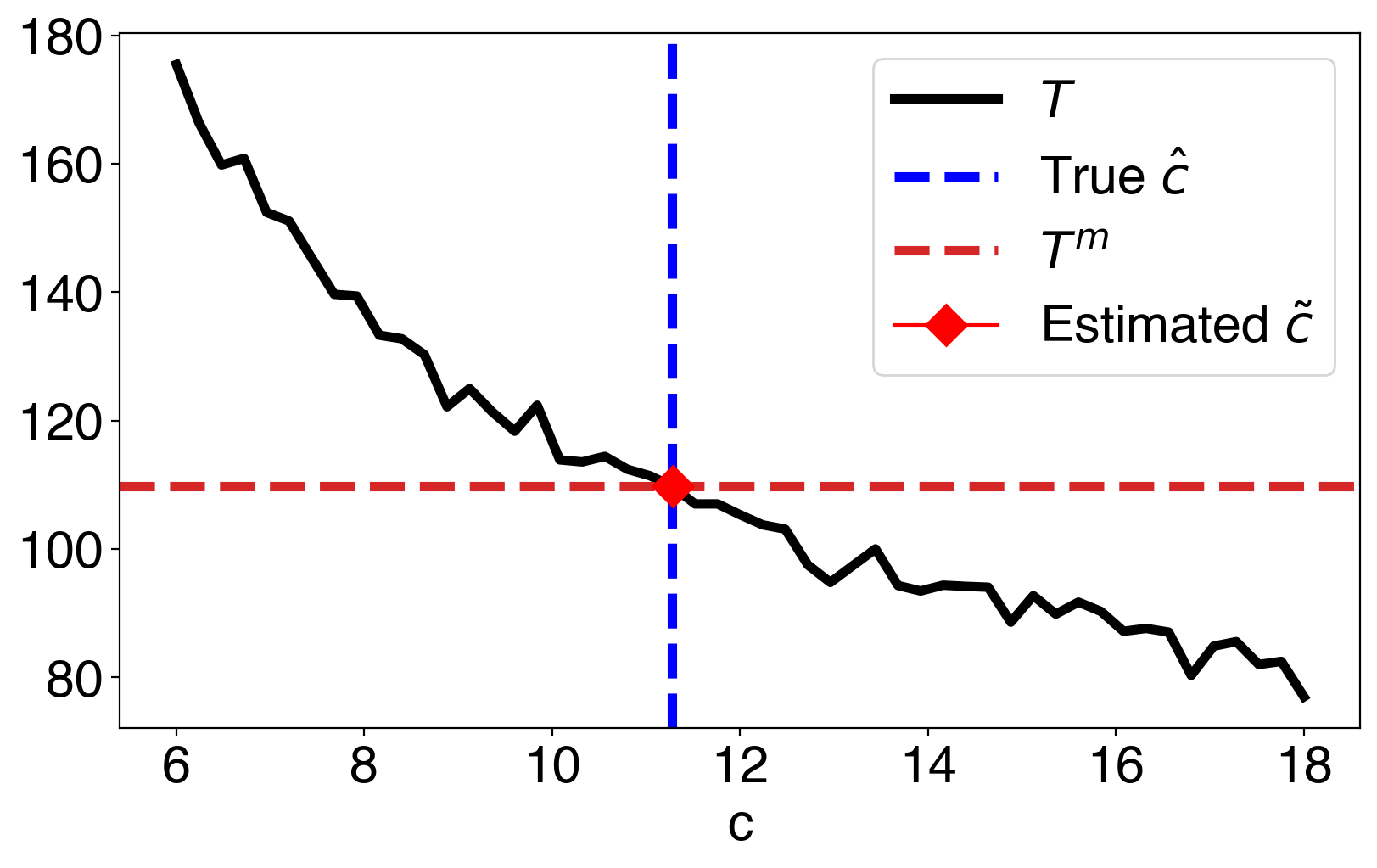}
        \caption{Estimated optimal $c_*$.}
        \label{fig:fish-flock_C_comp}
    \end{subfigure}%
    \hfill
    \begin{subfigure}[t]{0.46\textwidth}
        \centering
        \includegraphics[width = \linewidth]{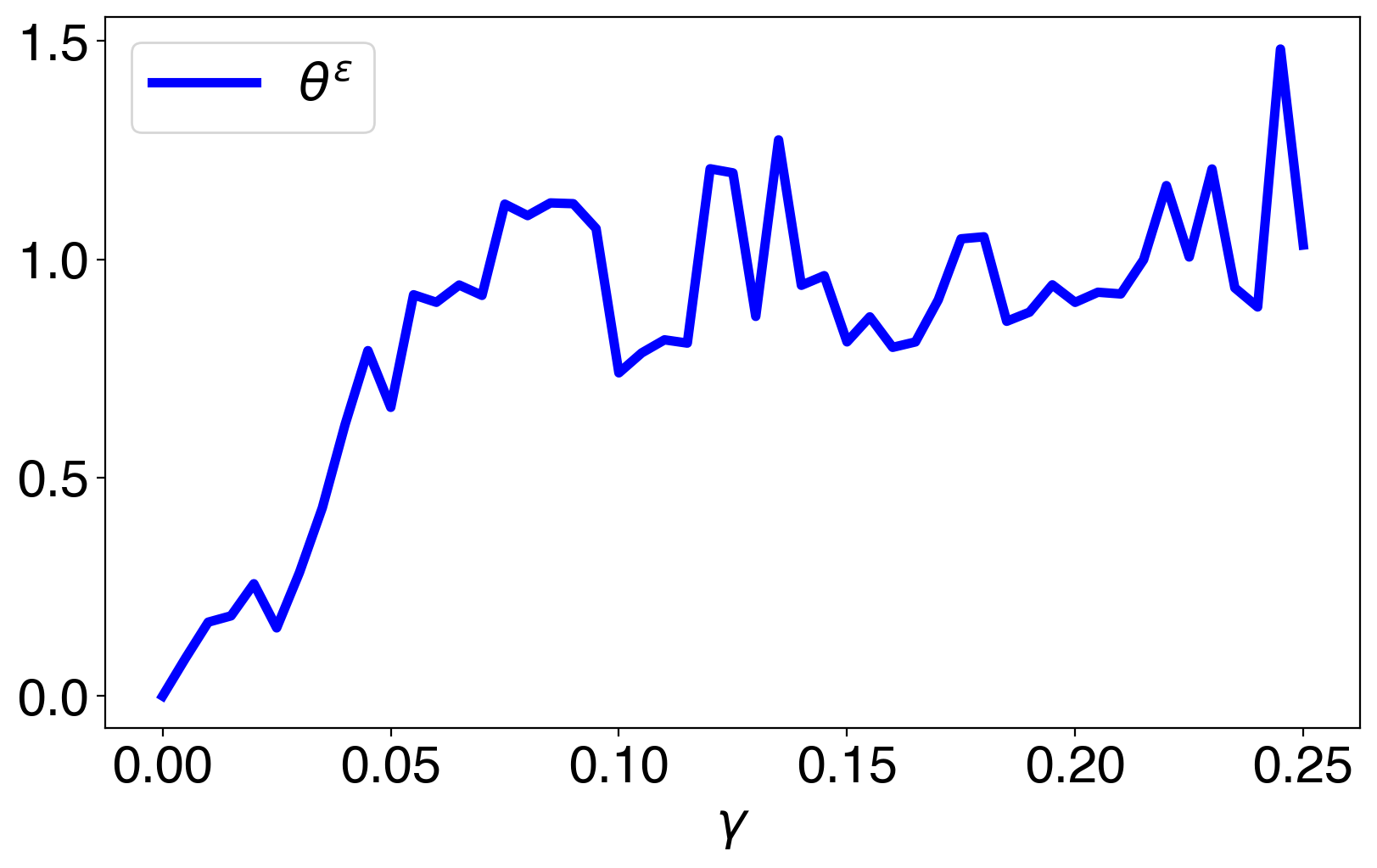}
        \caption{Noise effects in learning.}
        \label{fig:fish-flock_noise}
    \end{subfigure}
    \caption{Additional learning from fish flocking steady patterns.}
\end{figure}

\begin{figure}[H]
    \centering
    \begin{subfigure}[t]{0.46\textwidth}
        \centering
        \includegraphics[width=\linewidth]{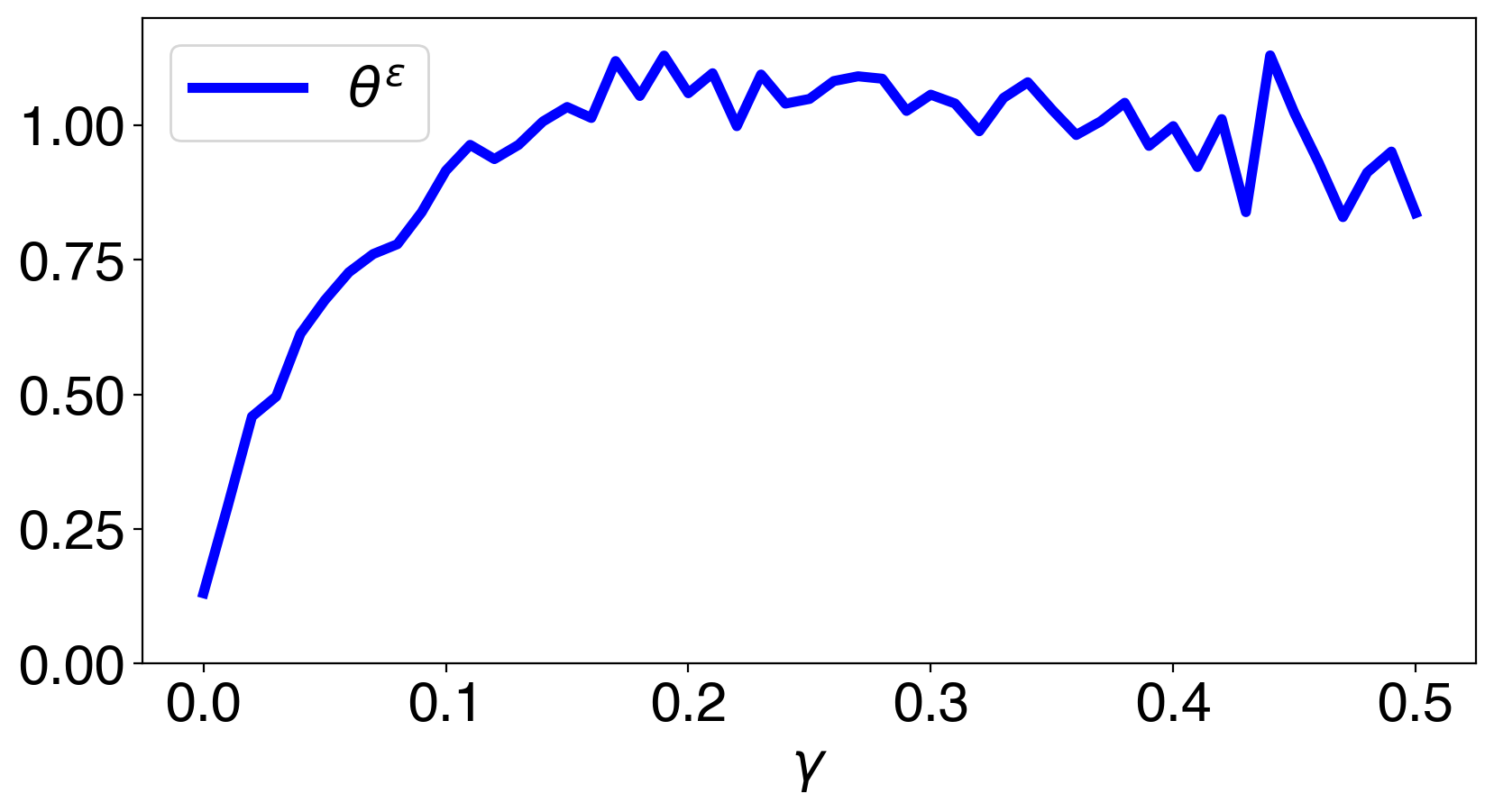}
    \caption{Fish milling (Rings).}
    \label{fig:fish-ring_noise}
    \end{subfigure}%
    \hfill
    \begin{subfigure}[t]{0.46\textwidth}
        \centering
        \includegraphics[width=\linewidth]{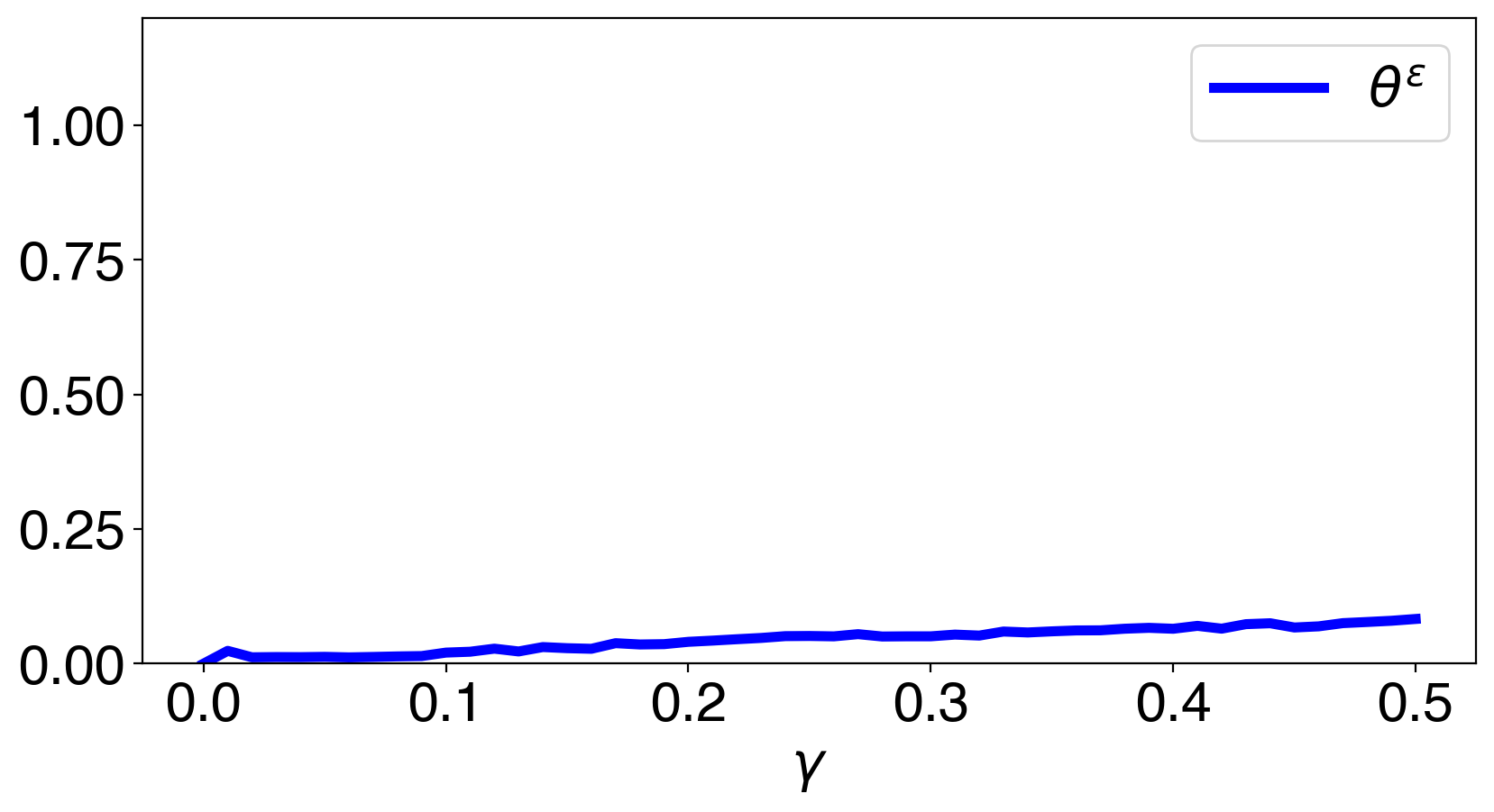}
    \caption{Double milling.}
    \label{fig:double-mill_noise}
    \end{subfigure}
    \caption{Noise effects in learning for quasi-static patterns from second-order systems}
\end{figure}
%
%

%
\end{document}